%% file: main.tex
\documentclass{article}

\usepackage{microtype}
\usepackage{graphicx}
\usepackage{subfigure}
\usepackage{booktabs} 

\usepackage{hyperref}

\usepackage[accepted]{icml2025}

\input{preamble}

\usepackage[textsize=tiny]{todonotes}

\icmltitlerunning{Causal Abstraction Inference under Lossy Representations}

\begin{document}

\twocolumn[
\icmltitle{Causal Abstraction Inference under Lossy Representations}



\icmlsetsymbol{equal}{*}

\begin{icmlauthorlist}
\icmlauthor{Kevin Xia}{columbia}
\icmlauthor{Elias Bareinboim}{columbia}
\end{icmlauthorlist}

\icmlaffiliation{columbia}{CausalAI Lab, Columbia University}

\icmlcorrespondingauthor{Kevin Xia}{kmx2000@columbia.edu}

\icmlkeywords{Machine Learning, Causality, ICML}

\vskip 0.3in
]



\printAffiliationsAndNotice{}  

\begin{abstract}
\input{section/00_abstract}
\end{abstract}

\input{section/01_introduction}

\input{section/02_soft_abs}

\input{section/03_inference}

\input{section/04_experiments}

\input{section/05_conclusion}

\bibliographystyle{icml2025}
\bibliography{references}

\clearpage
\appendix
\onecolumn

\input{section/A_proofs}

\input{section/B_additional_results}

\input{section/C_examples}
\input{section/D_experimental_details}
\input{section/E_discussion}


\end{document}

%% file: preamble.tex
\ifdefined\preamble\else\def\preamble{}

\usepackage[utf8]{inputenc} 
\usepackage[T1]{fontenc}    
\usepackage{url}            
\usepackage{booktabs}       
\usepackage{amsfonts}       
\usepackage{nicefrac}       
\usepackage{microtype}      
\usepackage{lipsum}
\usepackage{bbm}
\usepackage{enumitem}
\usepackage{optidef}
\usepackage{subcaption}
\usepackage{bm}
\usepackage{upgreek}
\usepackage{multicol}
\usepackage{wrapfig}
\usepackage{multirow}

\usepackage{tikz}
\usetikzlibrary{shapes,decorations,arrows,calc,arrows.meta,fit,positioning}
\tikzset{
    -Latex,auto,node distance =1 cm and 1 cm,semithick,
    state/.style ={ellipse, draw, minimum width = 0.7 cm},
    point/.style = {circle, draw, inner sep=0.04cm,fill,node contents={}},
    bidirected/.style={Latex-Latex,dashed},
    el/.style = {inner sep=2pt, align=left, sloped}
}

\usepackage{amsmath,amsthm,amsfonts,amssymb}

\usepackage{pifont}
\definecolor{myRed}{rgb}{0.75,0,0}
\definecolor{myGreen}{rgb}{0,0.75,0}
\usepackage{thmtools,thm-restate}
\declaretheorem{theorem}
\declaretheorem{lemma}

\declaretheorem{proposition}

\newtheorem{innercustomprop}{Proposition}
\newenvironment{customprop}[1]
  {\renewcommand\theinnercustomprop{#1}\innercustomprop}
  {\endinnercustomprop}

\newtheorem{innercustomthm}{Theorem}
\newenvironment{customthm}[1]
  {\renewcommand\theinnercustomthm{#1}\innercustomthm}
  {\endinnercustomthm}

\theoremstyle{definition}
\declaretheorem{definition}
\declaretheorem{example}

\newtheorem{innercustomdef}{Definition}
\newenvironment{customdef}[1]
  {\renewcommand\theinnercustomdef{#1}\innercustomdef}
  {\endinnercustomdef}

\newcommand{\Parents}{Pa}

\DeclareSymbolFont{symbolsC}{U}{txsyc}{m}{n}
\DeclareMathSymbol{\boxright}{\mathrel}{symbolsC}{128}

\def\*#1{\mathbf{#1}}

\newcommand{\pai}[1]{\*{pa}_{#1}}
\newcommand{\Pai}[1]{\*{Pa}_{#1}}

\newcommand{\ui}[1]{\*{u}_{#1}}
\newcommand{\Ui}[1]{\*{U}_{#1}}

\theoremstyle{definition}

\DeclareMathOperator{\unif}{Unif}

\newcommand{\xdasharrow}[2][->]{
\tikz[baseline=-\the\dimexpr\fontdimen22\textfont2\relax]{
\node[anchor=south,font=\scriptsize, inner ysep=1.5pt,outer xsep=8pt](x){#2};
\draw[shorten <=3.4pt,shorten >=3.4pt,dashed,#1](x.south west)--(x.south east);
}
}

\def\ddefloop#1{\ifx\ddefloop#1\else\ddef{#1}\expandafter\ddefloop\fi}
\def\ddef#1{\expandafter\def\csname bb#1\endcsname{\ensuremath{\mathbb{#1}}}}
\ddefloop ABCDEFGHIJKLMNOPQRSTUVWXYZ\ddefloop
\def\ddef#1{\expandafter\def\csname c#1\endcsname{\ensuremath{\mathcal{#1}}}}
\ddefloop ABCDEFGHIJKLMNOPQRSTUVWXYZ\ddefloop
\def\ddef#1{\expandafter\def\csname h#1\endcsname{\ensuremath{\widehat{#1}}}}
\ddefloop ABCDEFGHIJKLMNOPQRSTUVWXYZ\ddefloop
\def\ddef#1{\expandafter\def\csname v#1\endcsname{\ensuremath{\boldsymbol{#1}}}}
\ddefloop ABCDEFGHIJKLMNOPQRSTUVWXYZabcdefghijklmnopqrstuvwxyz\ddefloop
\def\ddef#1{\expandafter\def\csname v#1\endcsname{\ensuremath{\boldsymbol{\csname #1\endcsname}}}}
\ddefloop {alpha}{beta}{gamma}{delta}{epsilon}{varepsilon}{zeta}{eta}{theta}{vartheta}{iota}{kappa}{lambda}{mu}{nu}{xi}{pi}{varpi}{rho}{varrho}{sigma}{varsigma}{tau}{upsilon}{phi}{varphi}{chi}{psi}{omega}{Gamma}{Delta}{Theta}{Lambda}{Xi}{Pi}{Sigma}{varSigma}{Upsilon}{Phi}{Psi}{Omega}{ell}\ddefloop

\makeatletter
\newcommand*{\indep}{%
  \mathbin{%
    \mathpalette{\@indep}{}%
  }%
}
\newcommand*{\nindep}{%
  \mathbin{
    \mathpalette{\@indep}{\not}
  }%
}
\newcommand*{\@indep}[2]{%
  \sbox0{$#1\perp\m@th$}
  \sbox2{$#1=$}
  \sbox4{$#1\vcenter{}$}
  \rlap{\copy0}
  \dimen@=\dimexpr\ht2-\ht4-.2pt\relax
  \kern\dimen@
  {#2}%
  \kern\dimen@
  \copy0 
} 
\makeatother
\fi

%% file: section/00_abstract.tex
The study of causal abstractions bridges two integral components of human intelligence: the ability to determine cause and effect, and the ability to interpret complex patterns into abstract concepts. Formally, causal abstraction frameworks define connections between complicated low-level causal models and simple high-level ones. One major limitation of most existing definitions is that they are not well-defined when considering lossy abstraction functions in which multiple low-level interventions can have different effects while mapping to the same high-level intervention (an assumption called the abstract invariance condition). In this paper, we introduce a new type of abstractions called projected abstractions that generalize existing definitions to accommodate lossy representations.
We show how to construct a projected abstraction from the low-level model and how it translates equivalent observational, interventional, and counterfactual causal queries from low to high-level.
Given that the true model is rarely available in practice we prove a new graphical criteria for identifying and estimating high-level causal queries from limited low-level data.
Finally, we experimentally show the effectiveness of projected abstraction models in high-dimensional image settings.

%% file: section/01_introduction.tex
\section{Introduction}

The ability to determine cause and effect, and the ability to interpret complex patterns into abstract concepts, are two integral components of human intelligence. From the causality perspective, causal reasoning is vital in planning courses of actions, determining blame and responsibility, and generalizing across changing environments. From the abstraction perspective, humans generally grasp better intuition when understanding something at a high-level. For example, a human can easily parse the object in an image as a dog or a car instead of interpreting it as a collection of pixel values. Combining these two modes of reasoning is vital for building more advanced AI systems.

Causal inference is often studied under the semantics of structural causal models (SCMs) \citep{pearl:2k}. An SCM models reality with a collection of mechanisms and exogenous distributions. Each SCM induces a collection of distributions categorized into three successively more descriptive layers known as the Ladder of Causation or Pearl Causal Hierarchy (PCH) \citep{pearl:mackenzie2018, bareinboim:etal20}. These three layers refer to the observational ($\cL_1$), interventional ($\cL_2$), and counterfactual ($\cL_3$) distributions. In many causal inference tasks, the goal is to infer a quantity from a higher layer using data from lower layers, a problem known as \emph{cross-layer inference}. It is understood that it is generally impossible to infer higher layer information without additional assumptions (a result known as the Causal Hierarchy Theorem or CHT \citep{bareinboim:etal20}), so understanding the necessary assumptions for performing inferences is a key component of any causal inference task.


Existing works on causal abstractions have made significant progress in defining abstraction principles, proving insightful properties, and learning abstraction functions in practice \citep{rubenstein:etal17-causalsem, beckers2019abstracting, Beckers2019-BECACA-8, geiger2023causal, pmlr-v213-massidda23a, DBLP:conf/clear2/ZennaroDAWD23, felekis:etal24}. Causal abstractions are typically studied by comparing a high-level model $\cM_H$, defined over high-level variables $\*V_H$, with its low-level counterpart $\cM_L$, defined over $\*V_L$. An abstraction function $\tau$ maps from $\*V_L$ to $\*V_H$, and $\cM_H$ is formally defined as an abstraction of $\cM_L$ if it satisfies key properties with respect to $\tau$ such as commutativity with interventions. More recently, this notion has been relaxed to only enforcing properties between distributions of $\cM_H$ and $\cM_L$ from the PCH \citep{xia:bareinboim24}. For example, rather than saying $\cM_H$ is a full abstraction of $\cM_L$, one can say that $\cM_H$ is an abstraction of $\cM_L$ specifically for interventional quantities in $\cL_2$ or for a single causal effect $P(y \mid do(x)) \in \cL_2$. \citet{xia:bareinboim24} also shows the synergy between causal abstraction theory and representation learning \citep{10.1109/TPAMI.2013.50}, which has shown great success in many deep learning applications by mapping high-dimensional data like images or text to simpler representation spaces. These definitions of causal abstractions have accomplished formalizing a broad topic of human intelligence into mathematical language.

One particular limitation of existing definitions of abstractions is known as the Abstract Invariance Condition (AIC), which states, informally, that two values cannot be abstracted together if they have different downstream impacts. This is illustrated in Fig.~\ref{fig:aic-violation}. For example, a nutritionist may have collected data on two types of cholesterol, HDL and LDL, and are studying their impact on heart disease \citep{steinberg2007, truswell2010}. They would like to abstract the two together by summing them as total cholesterol (TC). However, this violates the AIC, as it is known that HDL decreases rate of heart disease while LDL increases it, so the sum is ambiguous (a lossy representation).\footnote{See App.~\ref{app:examples} Ex.~\ref{ex:noaic-heart-disease} for a more concrete explanation.} Nonetheless, it may still be desirable to have a consistent formalism in which these kinds of ambiguous abstractions are well-defined, since in many practical settings (where representation learning or dimensionality reduction is needed), the AIC is clearly violated or is impossible to verify.


In this paper, we study this extension of causal abstractions, which we later define as \emph{projected abstractions}, referring to the idea that an abstraction that violates the AIC results in a loss of information that is then characterized in the exogenous space. The proposed formalism generalizes abstractions both on the SCM and on the PCH level to allow for mathematically consistent abstractions even with AIC violations. Projected abstractions have many uses in practice, resulting in tractable causal inference and high-quality causal sampling even in the presence of extreme dimensionality reduction, a result which we show in the experiments.

To summarize, in Sec.~\ref{sec:soft-abs}, we generalize abstractions to settings which the AIC does not hold and provide an algorithm for constructing the high-level model. In Sec.~\ref{sec:inference}, we show how to perform causal inference from data within this class of abstractions when the true model is not observed. In Sec.~\ref{sec:experiments}, we empirically demonstrate the power of abstractions at performing causal inference in high-dimensional image settings. All proofs can be found in App.~\ref{app:proofs}. Appendices can be found in the full technical report, \citet{xialossy:tr}.

\begin{figure}
\centering
\includegraphics[width=\linewidth]{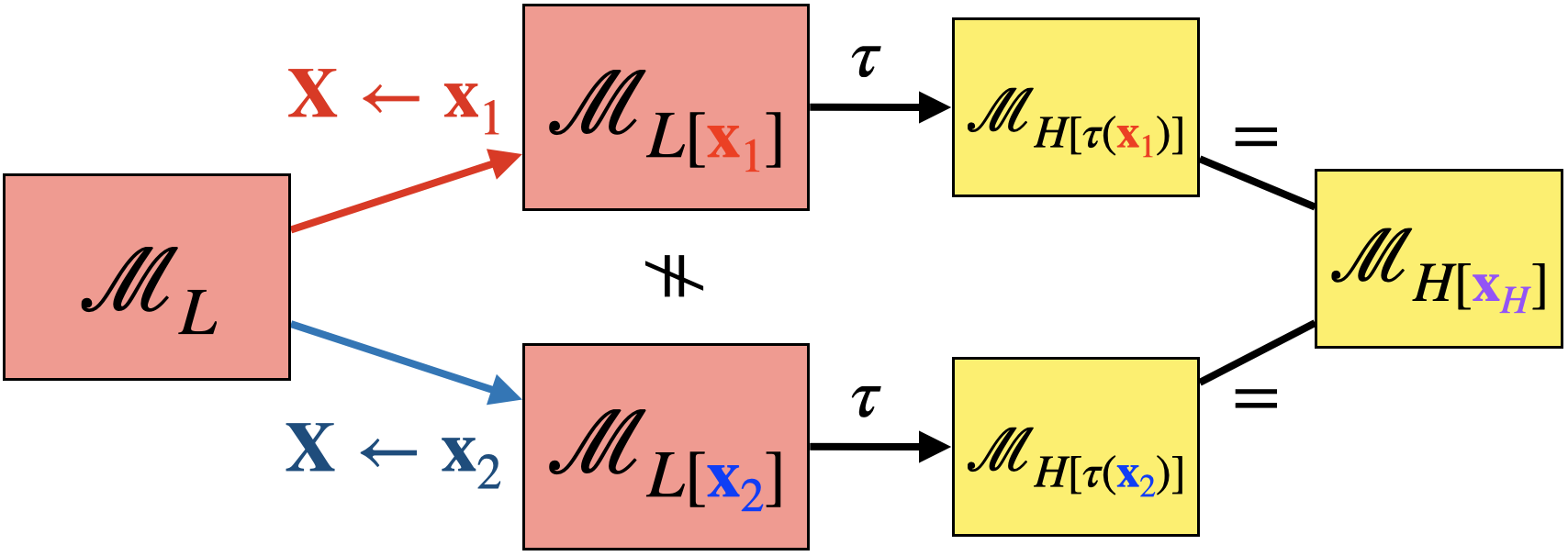}
\caption{An illustration of AIC violations. On the low level, two different interventions may be performed (e.g., {\color{red}$\*X \gets \*x_1$} and {\color{blue}$\*X \gets \*x_2$}). However, after applying the abstraction function $\tau$ to obtain the high-level model, both interventions are mapped to the same result ($\tau({\color{red}\*x_1}) = \tau({\color{blue}\*x_2}) = {\color{violet}\*x_H}$). If $\cM_{L}$ behaves differently under $\color{red}\*x_1$ compared to $\color{blue}\*x_2$, $\cM_H$ cannot stay consistent with both models.}
\label{fig:aic-violation}
\vspace{-0.1in}
\end{figure}

\subsection{Preliminaries}
\label{sec:prelims}

We now introduce the notation and definitions used throughout the paper. We use uppercase letters ($X$) to denote random variables and lowercase letters ($x$) to denote corresponding values. Similarly, bold uppercase ($\*X$) and lowercase ($\*x$) letters denote sets of random variables and values respectively. We use $\cD_{X}$ to denote the domain of $X$ and $\cD_{\mathbf{X}} = \cD_{X_1} \times \dots \times \cD_{X_k}$ for the domain of $\mathbf{X} = \{X_1, \dots, X_k\}$. We denote $P(\*X = \*x)$ (often shortened to $P(\*x)$) as the probability of  $\*X$ taking the values $\*x$ under the distribution $P(\*X)$. 

We utilize the basic semantic framework of structural causal models (SCMs) \citep{pearl:2k}, following the presentation in \citet{bareinboim:etal20}. 

\begin{definition}[Structural Causal Model (SCM)]
    \label{def:scm} 
    An SCM $\cM$ is a 4-tuple $\langle \*U, \*V, \cF, P(\*U)\rangle$, where $\*U$ is a set of exogenous variables (or ``latents'') that are determined by factors outside the model; $\*V$ is a set $\{V_1, V_2, \ldots, V_n\}$ of (endogenous) variables of interest that are determined by other variables in the model -- that is, in $\*U \cup \*V$; $\cF$ is a set of functions $\{f_{V_1}, f_{V_2}, \ldots, f_{V_n}\}$ such that each $f_{V_i}$ is a mapping from (the respective domains of) $\Ui{V_i} \cup \Pai{V_i}$ to $V_{i}$, where $\Ui{V_i} \subseteq \*U$, $\Pai{V_i} \subseteq \*V \setminus V_{i}$, and the entire set $\cF$ forms a mapping from $\*U$ to $\*V$. That is, for $i=1,\ldots,n$, each $f_{V_i} \in \cF$ is such that $v_i \leftarrow f_{V_i}(\pai{V_i}, \ui{V_i})$; and $P(\*U)$ is a probability function defined over the domain of $\*U$.
    \hfill $\blacksquare$
\end{definition}

Each $\cM$  induces a causal diagram $\cG$, where every $V_i \in \*V$ is a vertex, there is a directed arrow $(V_j \rightarrow V_i)$ for every $V_i \in \*V$ and $V_j \in \Pai{V_i}$, and there is a dashed-bidirected arrow $(V_j  \dashleftarrow \dashrightarrow V_i)$ for every pair $V_i, V_j \in \*V$ such that $\Ui{V_i}$ and $\Ui{V_j}$ are not independent (Markovianity is not assumed). Our treatment is constrained to \emph{recursive} SCMs, which implies acyclic causal diagrams, with finite discrete domains over endogenous variables $\mathbf{V}$. 

Counterfactual (and also interventional and observational) quantities can be computed from SCM $\cM$ as follows: 
\begin{definition}[Layer 3 Valuation {\citep[Def.~7]{bareinboim:etal20}}] 
\label{def:l3-semantics}
An SCM $\cM$ induces layer $\cL_3(\cM)$, a set of distributions over $\*V$, each with the form $P(\*Y_*) = P(\*Y_{1[\*x_1]}, \*Y_{2[\*x_2], \dots})$ such that 
\begin{align}
    \label{eq:def:l3-semantics}
    & P^{\cM}(\*y_{1[\*x_1]}, \*y_{2[\*x_2]}, \dots) = \\
    & \int_{\cD_{\mathbf{U}}} \mathbf{1}\left[\*Y_{1[\*x_1]}(\*u)=\*y_1, \*Y_{2[\*x_2]}(\*u) = \*y_2, \dots \right] dP(\*u) \nonumber
\end{align}
where ${\*Y}_{i[\*x_i]}(\*u)$ is evaluated under 
 $\mathcal{F}_{\*x_i}\! :=\! \{f_{V_j}\! :\! V_j \in \*V \setminus \*X_i\} \cup \{f_X \leftarrow x\! :\! X \in \*X_i\}$. $\cL_2$ is the subset of $\cL_3$ for which all $\*x_i$ are equal, and $\cL_1$ is the subset for which all $\*X_i = \emptyset$.
 \hfill $\blacksquare$
\end{definition}
Each $\*Y_i$ corresponds to a set of variables in a world where the original mechanisms $f_X$ are replaced with constants $\*x_i$ for each $X \in \*X_i$; this is also known as the mutilation procedure. This procedure corresponds to interventions, and we use subscripts to denote the intervening variables (e.g. $\*Y_{\*x}$) or subscripts with brackets when the variables are indexed (e.g. $\*Y_{1[\*x_1]}$). For instance, $P(y_x, y'_{x'})$ is the probability of the joint counterfactual event $Y=y$ had $X$ been $x$ and $Y=y'$ had $X$ been $x'$. 

We use the notation $\cL_i(\cM)$ to denote the set of $\cL_i$ distributions from $\cM$. We use $\bbZ$ to denote a set of quantities from Layer 2 (i.e. $\bbZ = \{P(\*V_{\*z_k})\}_{k=1}^{\ell}$), and $\bbZ(\cM)$ denotes those same quantities induced by SCM $\cM$ (i.e. $\bbZ(\cM) = \{P^{\cM}(\*V_{\*z_k})\}_{k=1}^{\ell}$).

The theory of causal abstractions developed in this paper build on the foundations of constructive abstraction functions, under which individual distributions of the PCH are well-defined between low and high-level models.

\begin{restatable}[Inter/Intravariable Clusterings {\citep[Def.~5]{xia:bareinboim24}}]{definition}{interintraclusters}
    \label{def:var-clusterings}
    Let $\cM$ be an SCM over $\*V$.
    \begin{enumerate}
        \item A set $\bbC$ is said to be an intervariable clustering of $\*V$ if $\bbC = \{\*C_1, \*C_2, \dots \*C_n\}$ is a partition of a subset of $\*V$. $\bbC$ is further considered admissible w.r.t.~$\cM$ if for any $\*C_i \in \bbC$ and any $V \in \*C_i$, no descendent of $V$ outside of $\*C_i$ is an ancestor of any variable in $\*C_i$. That is, there exists a topological ordering of the clusters of $\bbC$ relative to the functions of $\cM$.
        \item A set $\bbD$ is said to be an intravariable clustering of variables $\*V$ w.r.t.~$\bbC$ if $\bbD = \{\bbD_{\*C_i} : \*C_i \in \bbC\}$, where $\bbD_{\*C_i} = \{\cD^{1}_{\*C_i}, \cD^{2}_{\*C_i}, \dots, \cD^{m_i}_{\*C_i}\}$ is a partition (of size $m_i$) of the domains of the variables in $\*C_i$, $\cD_{\*C_i}$ (recall that $\cD_{\*C_i}$ is the Cartesian product $\cD_{V_1} \times \cD_{V_2} \times \dots \times \cD_{V_k}$ for $\*C_i = \{V_1, V_2, \dots, V_k\}$, so elements of $\cD_{\*C_i}^j$ take the form of tuples of the value settings of $\*C_i$).
        \hfill $\blacksquare$
    \end{enumerate}
\end{restatable}

\begin{restatable}[Constructive Abstraction Function {\citep[Def.~6]{xia:bareinboim24}}]{definition}{consabsfunc}
    \label{def:tau}
    A function $\tau: \cD_{\*V_L} \rightarrow \cD_{\*V_H}$ is said to be a constructive abstraction function w.r.t.~inter/intravariable clusters $\bbC$ and $\bbD$ iff
    \begin{enumerate}
        \item There exists a bijective mapping between $\*V_H$ and $\bbC$ such that each $V_{H, i} \in \*V_H$ corresponds to $\*C_i \in \bbC$;
        \item For each $V_{H, i} \in \*V_H$, there exists a bijective mapping between $\cD_{V_{H, i}}$ and $\bbD_{\*C_i}$ such that each $v_{H, i}^j \in \cD_{V_{H, i}}$ corresponds to $\cD^j_{\*C_i} \in \bbD_{\*C_i}$; and 
        \item $\tau$ is composed of subfunctions $\tau_{\*C_i}$ for each $\*C_i \in \bbC$ such that $\*v_H = \tau(\*v_L) = (\tau_{\*C_i}(\*c_i) : \*C_i \in \bbC)$, where $\tau_{\*C_i}(\*c_i) = v^j_{H,i}$ if and only if $\*c_i \in \cD^{j}_{\*C_i}$. We also apply the same notation for any $\*W_L \subseteq \*V_L$ such that $\*W_L$ is a union of clusters in $\bbC$ (i.e. $\tau(\*w_L) = (\tau_{\*C_i}(\*c_i) : \*C_i \in \bbC, \*C_i \subseteq \*W_L)$).
    \hfill $\blacksquare$
    \end{enumerate}
\end{restatable}

Finally, we state the AIC formally below.

\begin{restatable}[Abstract Invariance Condition (AIC)]{definition}{aicdef}
    \label{def:invariance-condition}
    Let $\cM_L = \langle \*U_L, \*V_L, \cF_L, P(\*U_L) \rangle$ be an SCM and $\tau: \cD_{\*V_L} \rightarrow \cD_{\*V_H}$ be a constructive abstraction function relative to $\bbC$ and $\bbD$. The SCM $\cM_L$ is said to satisfy the abstract invariance condition (AIC, for short) with respect to $\tau$ if, for all $\*v_1, \*v_2 \in \cD_{\*V_L}$ such that $\tau(\*v_1) = \tau(\*v_2)$, $\forall \*u \in \cD_{\*U_L}, \*C_i \in \bbC$, the following holds:
    \begin{equation}
        \label{eq:aic}
        \begin{split}
            & \tau_{\*C_i} \left( \left ( f^L_V(\pai{V}^{(1)}, \*u_V): V \in \*C_i \right ) \right) \\
            &= \tau_{\*C_i} \left( \left ( f^L_V(\pai{V}^{(2)}, \*u_V): V \in \*C_i \right ) \right),
        \end{split}
    \end{equation}
    where $\pai{V}^{(1)}$ and $\pai{V}^{(2)}$ are the values corresponding to $\*v_1$ and $\*v_2$.
    \hfill $\blacksquare$
\end{restatable}



A table summarizing the notation can be found in App.~\ref{app:notation-table}, detailed explanations of these definitions can be found in App.~\ref{app:extended-prelim}, and additional useful definitions from prior work can be found in App.~\ref{app:additional-defs}.

%% file: section/02_soft_abs.tex
\section{Abstractions under AIC Violations}
\label{sec:soft-abs}

The abstract invariance condition (AIC) states, in words, that two low-level values cannot map to the same high-level value if they have different downstream effects. This is a critical property that must hold for existing definitions of abstractions to be well-defined.
In this paper, we will use the following running example to illustrate the key points.

\begin{example}
    \label{ex:noaic-issue}
    For concreteness, consider a setting in which different insurance companies ($Z$) offer various insurance plans ($X$), which affect whether an insurance claim is approved ($Y$). For simplicity, suppose there are two insurance companies ($z_1$ and $z_2$) that offer three insurance plans ($x_1$, $x_2$, and $x_3$), and the claim is either approved ($Y = 1$) or not approved ($Y = 0$). Suppose the true model $\cM^* = \cM_L = \langle \*U_L, \*V_L, \cF_L, P(\*U_L)\rangle$ is described as
    {
    \allowdisplaybreaks
    \begin{align}
        \*U_L &= \{U_Z, U_X^{z_1}, U_X^{z_2}, U_Y^{x_1}, U_Y^{x_2}, U_Y^{x_3}\} \nonumber \\
        \*V_L &= \{Z, X, Y\} \nonumber \\
        \cF_L \! &= \!
        \begin{cases}
            f^L_Z(u_Z) &= u_Z \\
            f^L_X(z, u_X^{z_1}, u_X^{z_2}) &= u_X^{z} \\
            f^L_Y(x, u_Y^{x_1}, u_Y^{x_2}, u_Y^{x_3}) &= u_Y^x
        \end{cases} \label{eq:ex-insurance-scm} \\
        P(\*U_L)\! &= \! \begin{cases}
            P(U_Z = z_1) = 0.5 \\
            P(U_X^{z_1}) \! = \! \{x_1 \! \rightarrow \! 0.4; x_2 \! \rightarrow \! 0.1; x_3 \! \rightarrow \! 0.5\} \\
            P(U_X^{z_2}) \! = \! \{x_1 \! \rightarrow \! 0.1; x_2 \! \rightarrow \! 0.4; x_3 \! \rightarrow \! 0.5\}\\
            P(U_Y^{x_1} = 1) = 0.9, P(U_Y^{x_2} = 1) = 0.1, \\
            P(U_Y^{x_3} = 1) = 0.9
        \end{cases} \nonumber
    \end{align}
    }
    The interpretation of the model is as follows: Insurance plans $x_1$ and $x_3$ are very effective, with $0.9$ probability of claim acceptance, while $x_2$ is very ineffective at only $0.1$ probability. Insurance company $z_1$ is more reputable than $z_2$ and is more likely to offer plan $x_1$ over $x_2$, while company $z_2$ prefers to offer plan $x_2$ over $x_1$.

    Suppose an important factor of consideration not shown in the model is that $x_1$ and $x_2$ are cheaper insurance plans, while $x_3$ is more expensive. A data scientist who is studying this model may choose to abstract the different plans away, categorizing them simply as ``cheap'' and ``expensive'' plans. Formally, they would study a set of higher-level variables $\*V_H = \{Z_H, X_H, Y_H\}$, where $Z_H = Z$, $Y_H = Y$, and $X_H$ has a domain $\cD_{X_H} = \{x_C, x_E\}$ corresponding to cheap and expensive plans respectively. There exists an abstraction function $\tau: \cD_{\*V_L} \rightarrow \cD_{\*V_H}$ such that $\tau$ maps $x_1$ and $x_2$ to $x_C$ (cheap) and maps $x_3$ to $x_E$ (expensive). We will sometimes use the notation $X_L$ to describe $X$ to disambiguate from $X_H$, and we will use the notation $Z$ and $Y$ instead of $Z_H$ and $Y_H$ since the variables are the same on both levels.

    This immediately brings the AIC into question. If the data scientist is interested in the causal effect of cheap plans on claim acceptance (i.e., $P(Y_{X_H = x_C} = 1)$), whether $x_C$ refers to $x_1$ or $x_2$ is ambiguous. To witness, note that
    \begin{align}
        P(Y_{X_L = x_1} = 1) &= 0.9 \\
        P(Y_{X_L = x_2} = 1) &= 0.1.
    \end{align}
    Since $\tau(x_1) = \tau(x_2) = x_C$, but $P(Y_{x_1}) \neq P(Y_{x_2})$, the AIC is clearly violated, leaving the intervention on $x_C$ ambiguous.
    \hfill $\blacksquare$
\end{example}

Fundamentally, the issue with AIC violations is clear: formal definitions of abstractions expect an equality between low-level and corresponding high-level quantities, but it is not well-defined when one high-level quantity corresponds to multiple differing low-level quantities. In practice, the AIC can be a difficult restriction. Generally, it is assumed to be true whenever abstractions are applied, but it is difficult to verify given that the true SCM and functions are rarely available in real-world settings. The assumption is also likely to be incorrect when applying abstractions na\"ively, for example, by performing representation learning or dimensionality reduction without taking the AIC into account. By definition, dimensionality reduction is a lossy transformation of the original data, and the AIC is violated if any of the lost information is relevant for downstream functions.

Even when the AIC does not hold, it does not necessarily mean that these lossy transformations should not be used. Representation learning and dimensionality reduction are often performed to improve tractability or interpretability at the cost of some lost information. Hence, it would still be desirable to perform causal inferences in the high-level space even under AIC violations. To address the issue of different low-level quantities matching the same high-level quantity, one can reinterpret the high-level quantity as a distribution over its corresponding low-level quantities, where the randomness in the distribution results from the lost information from the abstraction (i.e., a hard intervention on the high-level translates to a soft intervention on the low-level).


\subsection{Projected Abstractions}

\begin{figure}
\centering
\includegraphics[width=\linewidth]{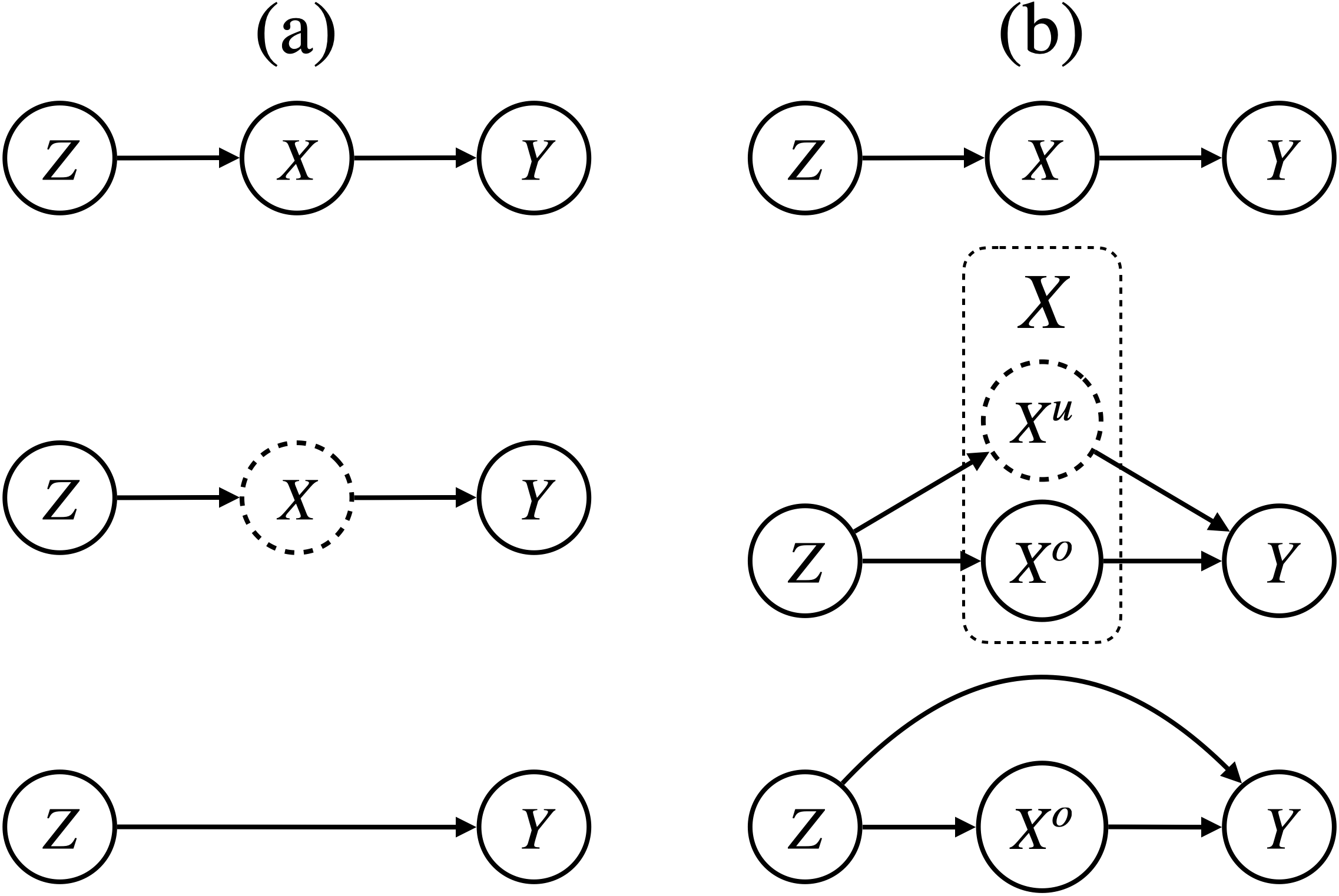}
\caption{Comparison between (a) full SCM projections and (b) partial SCM projections. When $X$ is fully projected away, its function is subsumed by its child's function $f_Y$. When $X$ is partially projected, it is split into observed portion $X^o$ and unobserved portion $X^u$. The role of $X^o$ is preserved, while $X^u$ is subsumed into the function $f_Y$.}
\label{fig:full-vs-partial-proj}
\vspace{-0.1in}
\end{figure}

The discussion on relaxing the AIC begins with the concept of SCM projections \citep{lee:bar19a}, which can be viewed as a primitive form of abstraction. An SCM $\cM$ projected to a subset of variables $\*W \subseteq \*V$ is a functionally identical SCM defined over $\*W$, where the functions of $\*V \setminus \*W$ are subsumed by other downstream functions (see App.~\ref{app:proofs} Def.~\ref{def:scm-proj} for the full definition and App.~\ref{app:examples} Ex.~\ref{ex:scm-projection} for an example). In the context of constructive abstraction functions, the act of projecting away a variable can be viewed as excluding the variable from all intervariable clusters. This brings the first major insight in addressing AIC violations. In general, when reducing the granularity of a variable, some parts of the variable deemed less important are abstracted away while others are retained. While by definition, SCM projections only allow for entire variables to be included or excluded, one could conceive of SCM projections in which variables are only partially projected away (see App.~\ref{app:examples} Ex.~\ref{ex:scm-partial-projection} for an example). Formally, partial SCM projections can be defined as follows.

\begin{restatable}[Partial SCM Projection]{proposition}{partialscmproj}
    \label{prop:partial-scm-projection}
    Let $\*V$ be a set of variables and $\*W \subseteq \*V$ be a subset. For each $W_i \in \*W$, let $\delta_i: \cD_{W_i^o} \times \cD_{W_i^u} \rightarrow \cD_{W_i}$ be a surjective function mapping new variables $W_i^o$ and $W_i^u$ to $W_i$. $W_i^o$ and $W_i^u$ are called the observed and unobserved projections of $W_i$ respectively. Denote $\delta(\*W^o, \*W^u) = \*W$, where $\*W^o = \{W_i^o : W_i \in \*W\}$ and $\*W^u = \{W_i^u : W_i \in \*W\}$. For any SCM $\cM = \langle \*U, \*V, \cF, P(\*U) \rangle$, there exists an SCM $\cM' = \langle \*U' = \*U \cup \*W^u, \*V' = \*W^o, \cF', P(\*U') \rangle$ such that, for all $\*u \in \cD_{\*U}$, $\*X \subseteq \*W$, and $\*x \in \cD_{\*X}$,
    \begin{equation}
        \label{eq:partial-scm-projection}
        \*w^o_{\*x} = \cM'_{[\*x^o]}(\*u, \*x^u, \*z^u),
    \end{equation}
    where $\delta(\*w^o_{\*x}, \*w^u_{\*x}) = \*W_{\*x}(\*u)$, $\delta(\*x^o, \*x^u) = \*x$, $\*Z^u = \*W^u \setminus \*X^u$, and $\*z^u$ are the corresponding values from $\*w^u_{\*x}$. $\cM'$ is called a partial SCM projection of $\cM$ over $\*W^o$.    
    \hfill $\blacksquare$
\end{restatable}

In words, a partial SCM projection of $\cM$ over $\*W^o$ is essentially a smaller version of $\cM$ defined only on the variables of $\*W \subseteq \*V$, where each $W_i \in \*W$ is only partially represented in the projection. A function $\delta$ splits $W_i$'s domain into its observed ($W_i^o$) and unobserved ($W_i^u$) portions. Eq.~\ref{eq:partial-scm-projection} ensures that any value of $\*W^o$ obtained from an intervention on the original SCM $\cM_{\*x}$ will match the corresponding output from $\cM'$, when the observed portion of the intervention $\*x^o$ is applied to $\cM'$, while the unobserved portions of $\*x^u$ and $\*w^u$ are passed as unobserved arguments to the functions. A comparison between regular SCM projections and partial SCM projections is shown in Fig.~\ref{fig:full-vs-partial-proj}. The definition of projected abstractions follow.

\begin{definition}[Projected Abstraction]
    \label{def:projected-abs}
    An SCM $\cM_H$ is a projected abstraction of $\cM_L$ if and only if it is a partial SCM projection of a $\tau$-abstraction \citep[Def.~3.13]{beckers2019abstracting} (also Def.~\ref{def:tau-abs} in App.~\ref{app:proofs}) of $\cM_L$.
    \hfill $\blacksquare$
\end{definition}

To provide intuition for projected abstractions, consider the following example.
\begin{example}
    \label{ex:proj-abs}
    Continuing Example \ref{ex:noaic-issue}, given the setup of Eq.~\ref{eq:ex-insurance-scm}, suppose $X_H \in \{x_C, x_E\}$ is given the function
    \begin{equation}
        f_{X}^H(z, u_X^{z_1}, u_X^{z_2}) = \begin{cases}
        x_C & u_X^z \in \{x_1, x_2\} \\
        x_E & u_X^z = x_3
        \end{cases},
    \end{equation}
    and define $X_H^u \in \{x_1, x_2\}$ as a random variable with distribution
    \begin{equation}
        \label{eq:ex-projected-exog}
        P(X_H^u = x_i) = P(X_L = x_i \mid X_L \in \{x_1, x_2\}, z).
    \end{equation}
    Suppose now $Y$ is now given a high-level function
    \begin{equation}
        \label{eq:ex-high-y}
        f_Y^H(x_H, x_H^u, u_Y^{x_i}) = \begin{cases}
            u_Y^{x_1} & x_H = x_C, X_H^u = x_1 \\
            u_Y^{x_2} & x_H = x_C, X_H^u = x_2 \\
            u_Y^{x_3} & x_H = x_E
        \end{cases}.
    \end{equation}
    Observe the intuition from constructing these functions from the perspective of projected abstractions. $f_X^H$ behaves identically to $f_X^L$, except the output remaps the value of $X_L$ to the corresponding $X_H$ (i.e.\ $f_X^H = \tau(f_X^L)$). However, due to the AIC violation, $f_Y^H$ is unable to disambiguate between $x_1$ and $x_2$ if $X_H = x_C$. The solution is to introduce a new exogenous variable $X_H^u$ which represents information in $X_L$ that is not captured in $X_H$ and disambiguates between $x_1$ and $x_2$. $f_Y^H$ then uses both $X_H$ and $X_H^u$ to mimic the behavior of $X_L$. It is clear that $X_L$ can be constructed as $\delta(X_H, X_H^u)$, defined as
    \begin{equation}
        \delta(x_H, x_H^u) = 
        \begin{cases}
            x_1 & x_H = x_C, X_H^u = x_1 \\
            x_2 & x_H = x_C, X_H^u = x_2 \\
            x_3 & x_H = x_E
        \end{cases},
    \end{equation}
    which matches Eq.~\ref{eq:ex-high-y}. Indeed, $\cM_H = \langle \*U_H = \*U_L \cup \{X_H^u\}, \*V_H, \cF_H = \{f_Z^L, f_X^H, f_Y^H\}, P(\*U_H) \rangle$ is a partial SCM projection (and also projected abstraction) of $\cM_L$ over $\*V_H$. The graph corresponding to $\cM_L$ is clearly the top graph of Fig.~\ref{fig:full-vs-partial-proj}(b), but note that through Eq.~\ref{eq:ex-projected-exog}, there is now a dependence from $Z$ to $Y$, so the graph for $\cM_H$ is instead the bottom graph of Fig.~\ref{fig:full-vs-partial-proj}(b).

    It is easy to see that Eq.~\ref{eq:partial-scm-projection} holds in this example. For instance, fix $U_Z = z_1$, $U_X^{z_1} = x_2$, $U_Y^{x_2} = 1$. Clearly, evaluating $\cM_L$ with these values results in $Z = z_1, X = x_2, Y = 1$. Note that $x_2 = \delta(x_C, x_2)$, and this is the only set of values of $X_H, X_H^u$ that map to $x_2$. Indeed, on the high level, with $U_Z = z_1$, $U_X^{z_1} = x_2$, $U_Y^{x_2} = 1, X_H^u = x_2$, it must also be the case that $Z = z_1, X_H = x_C, Y = 1$.
    \hfill $\blacksquare$
\end{example}

Projected abstractions make an important step to working around the AIC as Eq.~\ref{eq:partial-scm-projection} allows for quantities to be well-defined between low and high-level variables by simply obtaining a partial projection of the original SCM $\cM_L$ over the high-level variables $\*V_H$. However, unlike full SCM projections, partial SCM projections are not unique in terms of the induced PCH distributions. Prop.~\ref{prop:partial-scm-projection} guarantees its existence but is underspecified in a couple of ways. First, $P(\mathbf{U}')$ is not fully defined, and it is not clear how $\*W^u$ should be sampled (e.g., it is not clear how Eq.~\ref{eq:ex-projected-exog} is chosen in Ex.~\ref{ex:proj-abs}). Second, Eq.~\ref{eq:partial-scm-projection} does not specify what behavior $\cM'$ should follow when $\*z^u$ does not match $\*w_{\*x}^{\*u}$ (e.g., How should $Y$ depend on $X_H^u$ in Ex.~\ref{ex:proj-abs} if $X_H = x_E$?).

The specific choice of partial SCM projection that best serves as an abstraction can be determined by understanding how low-level interventions relate to high-level interventions. In other words, given a high-level intervention $\*X_H \gets \*x_H$, it is important to define the corresponding low-level soft-intervention $\sigma_{\*X_L}$, which is a distribution over all possible interventions $\*x_L$ that map to $\*x_H$. The consequence of the underspecification of partial SCM projections is that there are many possible choices of defining $\sigma_{\*X_L}$. For a full discussion on how $\sigma_{\*X_L}$ should be decided, see App.~\ref{app:add-results}. A useful general form of $\sigma_{\*X_L}$ is defined as follows. Split $\sigma_{\*X_L}$ into individual soft interventions $\sigma_{\*C_i}$ for each intervariable cluster $\*C_i \subseteq \*X_L$. Then define each $\sigma_{\*C_i}$ as
\begin{equation}
    \label{eq:low-intervention-short}
    P(\sigma_{\*C_i} = \*c_i) = P(\*c_i \mid \tau(\*c_i) = v_{H, i}, \pai{V_{H, i}}, \ui{V_{H, i}}^c).
\end{equation}
In words, a high-level intervention should be equivalent to a distribution over the corresponding low-level interventions that assigns probability to each possible intervention based on their prior probabilities given their parents.\footnote{Here, $\*u_{V_{H, i}}^c$ can informally be thought of as the confounded exogenous parents of $V_{H, i}$. The full definition is somewhat involved, and the subtleties are discussed in App.~\ref{app:discussion-nonmarkov}. Due to space constraints, the main body provides intuition in Markovian settings, where unobserved confounding is not present.}

\begin{example}
    \label{ex:query-definition}
    Continuing Example \ref{ex:noaic-issue}, suppose the data scientist is interested in the causal effect of choosing a cheap insurance plan on claim approval. In other words, she would like to study the intervention $X_H \gets x_C$, which is ambiguous on the low-level as it could refer to either $X_L \gets x_1$ or $X_L \gets x_2$. More specifically, according to Eq.~\ref{eq:low-intervention-short}, $X_H \gets x_C$ corresponds to a soft intervention $\sigma_{X_C}$ on the low level, defined as
    \begin{equation}
        \sigma_{X_L} = \begin{cases}
            x_1 & \text{w.p. } P(x_1 \mid X_L \in \{x_1, x_2\}, z) \\
            x_2 & \text{w.p. } P(x_2 \mid X_L \in \{x_1, x_2\}, z)
        \end{cases}
    \end{equation}
    While there are many ways to disambiguate whether $x_C$ is referring to $x_1$ or $x_2$, this choice of $\sigma_{X_L}$ will assign probabilities based on the prior probabilities of $X_L$ being one of $x_1$ or $x_2$. Moreover, the probabilities change depending on the value of $z$. This makes intuitive sense, since under the intervention $X_H \gets x_C$, we expect that if $Z = z_1$, then $X_L$ is more likely to be $x_1$ than $x_2$, or vice-versa when $Z = z_2$. From a query perspective, this implies that
    \begin{align}
        & P(Y_{X_H = x_C} = 1 \mid Z = z_1) \label{eq:q-cond-z1} \\
        &= P(Y_{\sigma_{X_L}(x_C, Z)} = 1 \mid Z = z_1) \nonumber \\ 
        &= \hspace{-4mm} \sum_{x_i \in \{x_1, x_2\}} \hspace{-4mm}P(x_i \mid X_L \in \{x_1, x_2\}, z_1)P(Y_{x_i} = 1) = 0.74 \nonumber \\
        & \text{Likewise, } P(Y_{X_H = x_C} = 1 \mid Z = z_2) = 0.26 \label{eq:q-cond-z2}
    \end{align}
    \hfill $\blacksquare$
\end{example}

While projected abstractions are defined over the entire SCM, the mapping between low and high-level interventions are more clear at the query-level (i.e., individual interventional and counterfactual distributions of interest). Such quantities can be defined as follows.

\begin{definition}[Generalized Query]
    \label{def:generalized-query}
    Denote $\*Y_{L,*}$ as a set of counterfactual variables over $\*V_L$. That is,
    \begin{equation}
        \label{eq:valid-low-ctf}
        \*Y_{L, *} = \left( \*Y_{L, 1[\sigma_{\*X_{L, 1}}]}, \*Y_{L, 2[\sigma_{\*X_{L, 2}}]}, \dots\right),
    \end{equation}
    where each $\*Y_{L, i[\sigma_{\*X_{L, i}}]}$ corresponds to the potential outcomes of the variables $\*Y_{L, i}$ under the (possibly soft) intervention $\sigma_{\*X_{L, i}}$ over $\*X_{L, i}$. Each $\*Y_{L, i}$ and $\*X_{L, i}$ must be unions of clusters from $\bbC$ (i.e.~$\*Y_{L, i} = \bigcup_{\*C \in \bbC'} \*C$ for some $\bbC' \subseteq \bbC$) such that $\tau(\*Y_{L, i})$ and $\tau(\*X_{L, i})$ are well-defined (i.e.~$\tau(\*Y_{L, i}) = \left(\bigwedge_{\*C \in \bbC'} \tau_{\*C}(\*C) \right)$). For the high-level counterpart, denote
    \begin{equation}
        \label{eq:valid-high-ctf}
        \*Y_{H, *} = \tau(\*Y_{L, *}) = \left(\*Y_{H, 1[\*x_{H, 1}]}, \*Y_{H, 2[\*x_{H, 2}]}, \dots\right), 
    \end{equation}
    such that $\*Y_{H, i} = \tau(\*Y_{L, i})$, and $\*X_{H, i} = \tau(\*X_{L, i})$ for all $i$. For any value $\*y_{H, *} \in \cD_{\*Y_{H, *}}$, denote
    \begin{equation}
        \cD_{\*Y_{L, *}}(\*y_{H, *}) = \{\*y_{L, *} : \*y_{L, *} \in \cD_{\*Y_{L, *}}, \tau(\*y_{L, *}) = \*y_{H, *}\},
    \end{equation}
    that is, the set of all values $\*y_{L, *}$ such that $\tau(\*y_{L, *}) = \*y_{H, *}$.

    For any high-level query
    \begin{equation}
        \label{eq:high-query}
        \tau(Q) = P(\*Y_{H, *} = \*y_{H, *}),
    \end{equation}
    of the form of Eq.~\ref{eq:valid-high-ctf}, its low-level counterpart is 
    \begin{equation}
        \label{eq:low-query}
        Q = \sum_{\*y_{L, *} \in \cD_{\*Y_{L, *}}(\*y_{H, *})} P(\*Y_{L, *} = \*y_{L, *}),
    \end{equation}
    of the form of Eq.~\ref{eq:valid-low-ctf}.
    \hfill $\blacksquare$
\end{definition}

\begin{algorithm}[t]
    \footnotesize
    \caption{Constructing $\cM_H$ from $\cM_L$.}
    \label{alg:map-abstraction}
    \begin{algorithmic}[1]
        \INPUT $\cM_L = \langle \*U_L, \*V_L, \cF_L, P(\*U_L)\rangle$, constructive abstraction function $\tau$ from clusters $\bbC$ and $\bbD$
        \STATE $\*U_H \gets \*U_L, P(\*U_H) \gets P(\*U_L)$
        \STATE $\*V_H \gets \bbC, \cD_{\*V_H} \gets \bbD$
        \FOR{$W \in \*V_L$}
            \STATE $W^o, W^u \gets \texttt{project}(W)$ \COMMENT{construct $\delta$ from Prop.~\ref{prop:partial-scm-projection}}
            \STATE $\*U_H \gets \*U_H \cup \{W^u\}$
        \ENDFOR
        \FOR{$\*C_i \in \bbC$ (and corresponding $V_i \in \*V_H$)}
            \STATE $P(\delta(\*c_i^o, \*C_i^u) = \*c_i \mid \*U_L) \gets P(\*C_i = \*c_i \mid \tau(\*c_i) = v_i, \pai{V_i}, \ui{V_{H, i}}^c$) \COMMENT{from Eq.~\ref{eq:low-intervention-short}}
            \STATE $f_i^H \gets \tau(f_V^L(\delta(\pai{V}^o, \pai{V}^u), \ui{V}): V \in \*C_i)$
        \ENDFOR
        \STATE $\cF_H \gets \{f_i^H: \*C_i \in \bbC\}$
        \STATE \textbf{return} $\cM_H = \langle \*U_H, \*V_H, \cF_H, P(\*U_H)\rangle$
    \end{algorithmic}
\end{algorithm}

This query definition connects the distributions of $\cL_3(\cM_H)$ to corresponding distributions of $\cL_3(\cM_L)$. Compared to earlier definitions, 
Eq.~\ref{eq:valid-low-ctf} has been generalized to account for soft interventions in addition to hard interventions. Under constructive abstractions functions $\tau$, a notion of $Q$-$\tau$ consistency was established for certain queries $Q \in \cL_3(\cM_L)$ (App.~\ref{app:proofs} Def.~\ref{def:q-tau-consistency}), which still apply under this generalized definition. In short, given a low level query $Q$ (Eq.~\ref{eq:low-query}) and its high-level counterpart $\tau(Q)$ (Eq.~\ref{eq:high-query}), $\cM_H$ is said to be $Q$-$\tau$ consistent with $\cM_L$ if $Q^{\cM_L} = \tau(Q)^{\cM_H}$. One can then say that $\cM_H$ is an abstraction of $\cM_L$ specifically for the query $Q$, even if $\cM_H$ may not be $Q'$-$\tau$ consistent with $\cM_L$ for other query choices $Q'$. If $\cM_H$ is $Q$-$\tau$ consistent with $\cM_L$ for all $\tau(Q) \in \cL_i(\cM_H)$, then $\cM_H$ is said to be $\cL_i$-$\tau$ consistent with $\cM_L$.

With $\sigma_{\*X_{L, i}}$ defined in Eq.~\ref{eq:low-intervention-short}, one can then algorithmically construct a projected abstraction consistent in all queries. Given $\cM_L$ and a constructive abstraction function $\tau$ (which may not satisfy the AIC), Alg.~\ref{alg:map-abstraction} can be used to construct the high-level abstraction $\cM_H$. In line 4, each $W \in \*V_L$ is split into its observed and unobserved counterparts $W^o$ and $W^u$. Line 8 assigns each $W^u$ a distribution based on Eq.~\ref{eq:low-intervention-short}. Line 9 builds the high-level function using the low-level function with inputs reconstructed using $\delta$. Finally, the full high-level model $\cM_H$ is assembled and returned in line 10. Under these inputs, Alg.~\ref{alg:map-abstraction} constructs a projected abstraction $\cM_H$ that is $Q$-$\tau$ consistent with $\cM_L$ for all possible high-level $\cL_3$ queries, as shown by the following result.

\begin{restatable}{theorem}{algconstructmh}
    \label{thm:alg-construct-mh}
    The SCM $\cM_H$ constructed by Alg.~\ref{alg:map-abstraction} is a projected abstraction of $\cM_L$ that is $Q$-$\tau$ consistent with $\cM_L$ for all $\tau(Q) \in \cL_3(\cM_H)$.
    \hfill $\blacksquare$
\end{restatable}

As an example, it can be verified that running Alg.~\ref{alg:map-abstraction} on $\cM_L$ in Ex.~\ref{ex:noaic-issue} results in the SCM $\cM_H$ from Ex.~\ref{ex:proj-abs}.

%% file: section/03_inference.tex
\section{Projected Abstraction Inference}
\label{sec:inference}

Alg.~\ref{alg:map-abstraction} finds an abstraction model $\cM_H$ that is consistent with its low-level counterpart $\cM_L$ for all queries, but it requires the full specification of $\cM_L$. In practice, $\cM_L$ typically represents the true model of reality and will not be observed. Inferences of $\cL_2$ and $\cL_3$ queries must be made through limited available data, usually observational ($\cL_1$).

The Causal Hierarchy Theorem \citep[Thm.~1]{bareinboim:etal20} states that cross-layer inference, or inferring higher layer quantities (e.g., $\cL_2$, $\cL_3$) from lower layer data (e.g., $\cL_1$), is generally impossible without additional assumptions. Many such assumptions take the form of a graphical model, such as a causal diagram \citep{pearl:95a}, which imply constraints between causal distributions from causal \citep{bareinboim:etal20} and counterfactual Bayesian networks \citep{correa2024ctfcalc}. In the context of abstractions, when $\tau$ is a constructive abstraction function that satisfies the AIC, it has been shown that one can avoid assuming the entire causal diagram of the low-level model in favor of a cluster causal diagram (C-DAG) \citep{anand:etal23} w.r.t.\ the intervariable clusters $\bbC$. 
Unfortunately, this graphical model is insufficient for the case when the AIC is violated.

\begin{proposition}[C-DAG Insufficiency (Informal)]
    \label{prop:cdag-insufficiency}
    For a constructive abstraction function $\tau$ over intervariable clusters $\bbC$ in which the AIC does not hold, the C-DAG $\cG_{\bbC}$ implies constraints that may be unsound.
    \hfill $\blacksquare$
\end{proposition}

To witness why this is the case, Fig.~\ref{fig:full-vs-partial-proj}(b) shows the issue clearly. Attempting an abstraction in violation of the AIC is akin to performing a partial SCM projection, which may introduce new dependencies between SCM functions, therefore implying new edges in the graph. Ex.~\ref{ex:query-definition} explains this dependence numerically. Since no variables are clustered together in the example, both the original causal diagram $\cG$ and the C-DAG $\cG_{\bbC}$ are represented by the top graph in Fig.~\ref{fig:proj-cdag-examples}. However, this graph implies that $P(Y_{x_H} \mid z) = P(Y_{x_H})$. Evidently, this is not true since Eq.~\ref{eq:q-cond-z1} is not equal to Eq.~\ref{eq:q-cond-z2}. As hinted by the construction in Alg.~\ref{alg:map-abstraction}, the high-level function $f^H_Y$ requires some additional information from $Z$ to decide between interpreting $x_C$ as $x_1$ or $x_2$. This information adds a dependence from $Z$ to the function of $f_Y^H$, which requires adding a directed edge from $Z$ to $Y$.

While the original C-DAG construction is not valid for projected abstraction inferences, one can use a modified version that adds the new required dependencies into the C-DAG.

\begin{definition}[Partially Projected C-DAG]
    \label{def:proj-cdag}
    Let $\tau: \cD_{V_L} \rightarrow \cD_{V_H}$ be a constructive abstraction function w.r.t.\ intervariable clusters $\bbC$ and intravariable clusters $\bbD$. Let $\cG_{\bbC} = \langle \*V_H, \*E_{\bbC} \rangle$ be a C-DAG (with nodes $\*V_H$ and edges $\*E_{\bbC}$), of graph $\cG$ w.r.t.\ $\bbC$. Let $\*V_H^\dagger \subseteq \*V_H$ be the set of AIC violation variables (App.~\ref{app:proofs} Def.~\ref{def:aic-violators}). Then, construct $\cG_{\bbC}^\dagger = \langle \*V_H, \*E_{\bbC}^{\dagger} \rangle$ as follows. Start by setting $\*E_{\bbC}^\dagger \gets \*E_{\bbC}$. Then apply the following rules for all $X \in \*V_H^{\dagger}$.
    \\
    (1) If $Z \rightarrow X \rightarrow Y$ in $\*E_{\bbC}$, then add $Z \rightarrow Y$ into $\*E_{\bbC}^\dagger$.
    \\
    (2) If $Z \xdasharrow[<->]{}  X \rightarrow Y$ in $\*E_{\bbC}$, then add $Z \xdasharrow[<->]{}  Y$ and $X \xdasharrow[<->]{}  Y$ into $\*E_{\bbC}^\dagger$.
    \\
    (3) If $Z \leftarrow X \rightarrow Y$ in $\*E_{\bbC}$, then add $Z \xdasharrow[<->]{} Y$ into $\*E_{\bbC}^\dagger$.
    \\
    Repeat iteratively to accommodate new edges.\footnote{Procedure can be applied algorithmically in one pass by applying all rules for each node in $\*V_H^\dagger$ in topological order.} $\cG^{\dagger}_{\bbC}$ is called the partially projected C-DAG of $\cG$ w.r.t.\ $\bbC$ and $\*V_H^\dagger$.
    \hfill $\blacksquare$
\end{definition}

\begin{figure}
\centering
\includegraphics[width=1\linewidth]{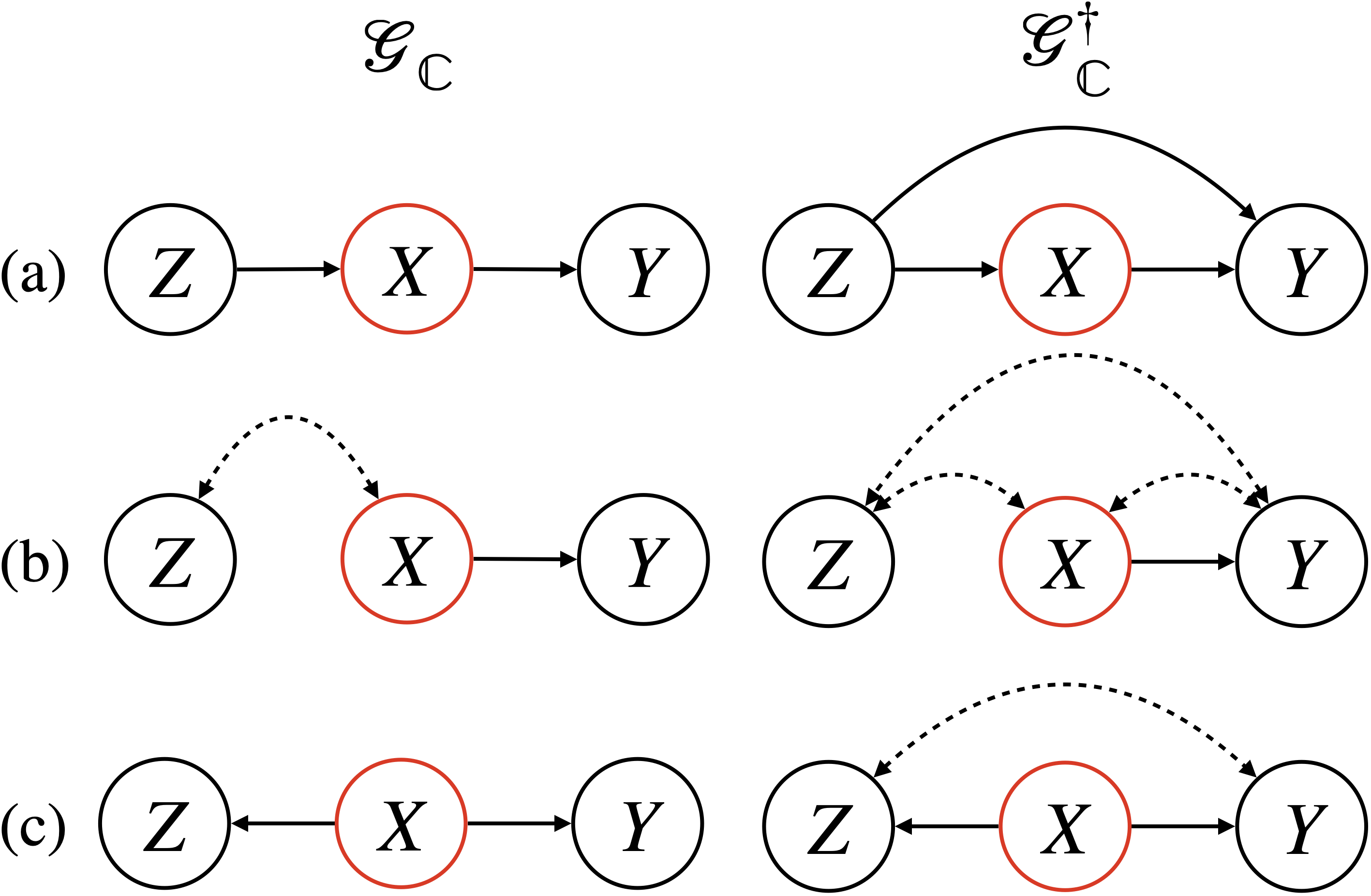}
\caption{Examples of C-DAGs (left) and their corresponding projected C-DAGs (right), with AIC violation variables $\*V_H^{\dagger}$ outlined in red.}
\label{fig:proj-cdag-examples}
\vspace{-0.1in}
\end{figure}

The steps correspond to the intuition discussed earlier--when performing a partial projection, parts of the variables in $\*V_H^{\dagger}$ are projected into the exogenous space, resulting in additional dependences that require additional edge connections. Examples of C-DAGs and their corresponding projected C-DAGs are shown in Fig.~\ref{fig:proj-cdag-examples}. In the figure, rows (a), (b), and (c) correspond to examples of steps 1, 2, 3 respectively. It turns out that this new definition is precisely what is needed for abstraction inference in the absence of the AIC.

\begin{theorem}[Projected C-DAG Sufficiency and Necessity (Informal)]
    \label{thm:proj-cdag-completeness}
    Let $\cM_L$ be an SCM over variables $\*V_L$, $\tau: \cD_{\*V_L} \rightarrow \cD_{\*V_H}$ be a constructive abstraction function w.r.t.\ clusters $\bbC$ and $\bbD$, and $\*V_H^{\dagger}$ be the AIC violation set. The partially projected C-DAG $\cG_{\bbC}^{\dagger}$ w.r.t.\ $\bbC$ and $\*V_H^\dagger$ completely describes all constraints over $\*V_H$.
    \hfill $\blacksquare$
\end{theorem}

In other words, the projected C-DAG provides exactly the constraints necessary to solve the task of performing causal inferences across abstractions, even when the AIC is violated. In particular, certain interventional and counterfactual distributions may be inferrable from a combination of the projected C-DAG $\cG_{\bbC}^{\dagger}$ and the available datasets from $\cM_L$. 
Determining precisely which queries can be inferred is known as the identification problem, which is defined below in the context of abstract identification. 

\begin{definition}[Abstract Identification (General)]
    \label{def:abs-id}
    Let $\tau: \cD_{\*V_H} \rightarrow \cD_{\*V_L}$ be a constructive abstraction function. Consider projected C-DAG $\cG_{\bbC}^{\dagger}$, and let $\bbZ = \{P(\*V_{L[\*z_k]})\}_{k=1}^{\ell}$ be a collection of available interventional (or observational if $\*Z_k = \emptyset$) distributions over $\*V_L$. Let $\Omega_L$ and $\Omega_H$ be the space of SCMs defined over $\*V_L$ and $\*V_H$, respectively, and let $\Omega_L(\cG_{\bbC}^\dagger)$ and $\Omega_H(\cG_{\bbC}^\dagger)$ be their corresponding subsets that induce $\cG_{\bbC}^\dagger$. A query $Q$ is said to be $\tau$-ID from $\cG_{\bbC}^\dagger$ and $\bbZ$ iff for every $\cM_L \in \Omega_L(\cG_{\bbC}^\dagger), \cM_H \in \Omega_H(\cG_{\bbC}^\dagger)$ such that $\cM_H$ is $\bbZ$-$\tau$ consistent with $\cM_L$, $\cM_H$ is also $Q$-$\tau$ consistent with $\cM_L$.
    \hfill $\blacksquare$
\end{definition}

In words, a query $Q$ is considered $\tau$-ID if, for any pair of models $\cM_L$ and $\cM_H$ such that both are compatible with $\cG_{\bbC}^\dagger$ and $\bbZ$, they also match in $Q$. In contrast, $Q$ is not $\tau$-ID if there exist $\cM_L$ and $\cM_H$ that are compatible with both $\cG_{\bbC}^\dagger$ and $\bbZ$ but disagree on $Q$ (i.e., $Q^{\cM_L} \neq \tau(Q)^{\cM_H}$). Abstract identification may seem like a difficult property to check, but it turns out that there is a natural connection with the classical identification problem, as shown below.

\begin{restatable}[Dual Abstract ID (General)]{theorem}{dualabsid}
    \label{thm:dual-abs-id}
    Consider a counterfactual query $Q$ over $\*V_L$, a constructive abstraction function $\tau$ w.r.t.~clusters $\bbC$ and $\bbD$, a projected C-DAG $\cG_{\bbC}^\dagger$, and data $\bbZ$ from $\*V_L$. $Q$ is $\tau$-ID from $\cG_{\bbC}^\dagger$ and $\bbZ$ if and only if $\tau(Q)$ is ID from $\cG_{\bbC}^\dagger$ and $\tau(\bbZ)$.
    \hfill $\blacksquare$
\end{restatable}

In words, $\tau$-identification across abstractions is equivalent to classic identification on the high-level space.

\begin{example}
    Continuing Ex.~\ref{ex:noaic-issue}, note that $X_H$ is the only AIC violator in $\*V_H$, since $x_1$ and $x_2$ both map to $x_C$ but have different effects on $Y$. Hence, $\*V_H^\dagger = \{X_H\}$, and the C-DAG $\cG_{\bbC}$ and projected C-DAG $\cG_{\bbC}^\dagger$ are the two graphs in Fig.~\ref{fig:proj-cdag-examples}(a). To answer the query of interest $P(Y_{X_H = x_C} = 1)$, one can apply Thm.~\ref{thm:dual-abs-id} to simply identify the quantity w.r.t.~$P(\*V_H)$ and $\cG_{\bbC}^\dagger$. In this case, note that the causal effect of $X_H$ on $Y$ can be computed via backdoor adjustment on $Z$, so $P(Y_{X_H = x_C} = 1)$ is equal to
    \begin{align}
        & \sum_z P(Y = 1 \mid X_H = x_C, Z = z)P(Z = z) \\
        &= \sum_z P(Y = 1 \mid X_L \in \{x_1, x_2\}, z)P(z) \\
        &= (0.7)(0.74) + (0.3)(0.26) = 0.596.
    \end{align}
    \hfill $\blacksquare$
\end{example}

Thm.~\ref{thm:dual-abs-id} implies that, in practice, $\tau$-ID can be checked by performing any classical ID procedure on the high-level space. This may include algorithmic approaches or other optimization-based approaches. 

%% file: section/04_experiments.tex
\section{Experiments}
\label{sec:experiments}

\begin{figure*}
    \begin{center}
    \includegraphics[width=\textwidth,keepaspectratio]{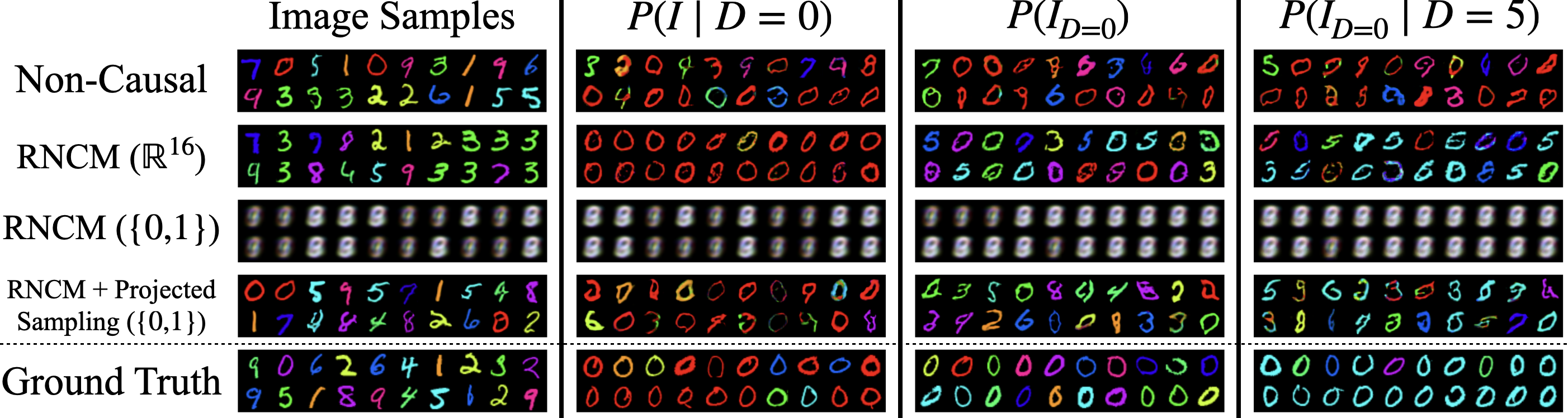}
    \caption{Colored MNIST results. Samples from different causal queries (top) are collected from competing approaches (left). The expressions in parentheses are the representation sizes. The left column shows direct image samples from each of the models, while the second, third, and fourth columns show samples generated from an $\cL_1$, $\cL_2$, and $\cL_3$ query, respectively.
    }
    \label{fig:mnist-bd-samp-results}
    \end{center}
\end{figure*}

\begin{figure}
\centering
\includegraphics[width=1\linewidth]{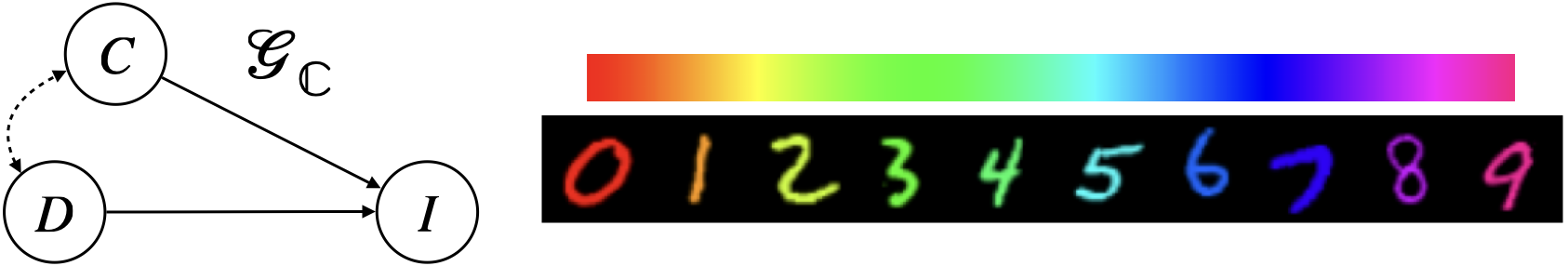}
\caption{(Left) Graph of Colored MNIST experiment. (Right) Correlation shown between color and digit.}
\label{fig:colored-mnist-legend}
\end{figure}

We perform two experiments to demonstrate the benefits of projected abstractions. The models in the experiments leverage Neural Causal Models (NCMs) \citep{xia:etal21, xia:etal23}, specifically the generative adversarial implementation called GAN-NCMs. Details of the experiment setup can be found in App.~\ref{app:experimental-details}, and code can be found at \url{https://github.com/CausalAILab/ProjectedCausalAbstractions}.

In the first experiment, we test the necessity of the projected C-DAGs when the AIC does not hold. The high-level query $\tau(Q) = P(y_{x} \mid  z)$ is estimated in the graph setting shown in Fig.~\ref{fig:proj-cdag-examples}(a), where $Z$ is a digit from 0 to 9, $X$ is a corresponding colored MNIST image, and $Y$ is a label denoting the color prediction of $X$. $\tau(X)$ maps the image to a binary variable representing the shade (light or dark) of $X$.

The results are shown in Fig.~\ref{fig:mnist-est-results}. Three different GAN-NCMs are trained: one directly on the low-level data that does not use abstractions (red), an abstracted one constrained by the C-DAG (yellow), and an abstracted one constrained by the projected C-DAG (blue). 95\% confidence intervals of the errors are plotted in the figure. Note that the abstractionless model and the projected C-DAG model have decreasing error with more samples, but the regular C-DAG model is unable to learn the correct query. The abstractionless model has higher error than the projected C-DAG model since it operates in a higher-dimensional space.

In the second experiment, we test an interesting consequence of the projected abstraction theory: the soft intervention definition in Eq.~\ref{eq:low-intervention-short} can be directly modeled and sampled if attempting to reconstruct the low-level data. We call this approach \emph{projected sampling} and explain it in more detail in App.~\ref{app:projected-sampling}. We show this in the causal colored MNIST experiment \citep{xia:bareinboim24}. In the model, digit $D$ and color $C$ both cause the image $I$, but they are confounded (e.g., 0's are red, 5's are cyan, see Fig.~\ref{fig:colored-mnist-legend}). Three different queries are tested (the right three columns of Fig.~\ref{fig:mnist-bd-samp-results}). $P(I \mid D = 0)$ is an $\cL_1$ query representing images conditioned on digit $=0$, resulting in red 0's. $P(I_{D = 0})$ is an $\cL_2$ query representing images with the digit intervened as 0, cutting the confounding and resulting in 0's of all colors. $P(I_{D = 0} \! \mid \! D \! = \! 5)$ is an $\cL_3$ query representing images with digit intervened as 0, conditioned on the digit originally being 5. This results in 0's with colors of images that were originally 5's, resulting in cyan 0's.

Four methods are compared on these queries in Fig.~\ref{fig:mnist-bd-samp-results}, with the ground truth shown on row 5. The non-causal approach (row 1) simply directly models the conditional distribution between digit and image and therefore fails to model anything higher than $\cL_1$. The representational NCM or RNCM \citep{xia:bareinboim24} (row 2) is able to decently reproduce all queries, but it uses a 16-dimensional representation space, which cannot shrink much further due to AIC limitations. When forced to take a binary representation (row 3), the RNCM clearly lacks the representation power to properly generate images. In contrast, using a projected sampling approach (row 4) can reproduce the images even with a representation size as small as a binary digit.

%% file: section/05_conclusion.tex
\section{Conclusion}

This paper introduced projected abstractions (Def.~\ref{def:projected-abs}), which can be constructed algorithmically (Alg.~\ref{alg:map-abstraction}, Thm.~\ref{thm:alg-construct-mh}), to overcome the AIC limitation. When the full model was not available, we leveraged a new graphical model (Def.~\ref{def:proj-cdag}, Thm.~\ref{thm:proj-cdag-completeness}) that allowed for causal inferences through the abstract-ID problem (Def.~\ref{def:abs-id}, Thm.~\ref{thm:dual-abs-id}). Finally, we demonstrated the ability of projected abstractions to leverage representation learning within difficult causal inference settings through high-dimensional image experiments.

\begin{figure}
\centering
\includegraphics[width=\linewidth]{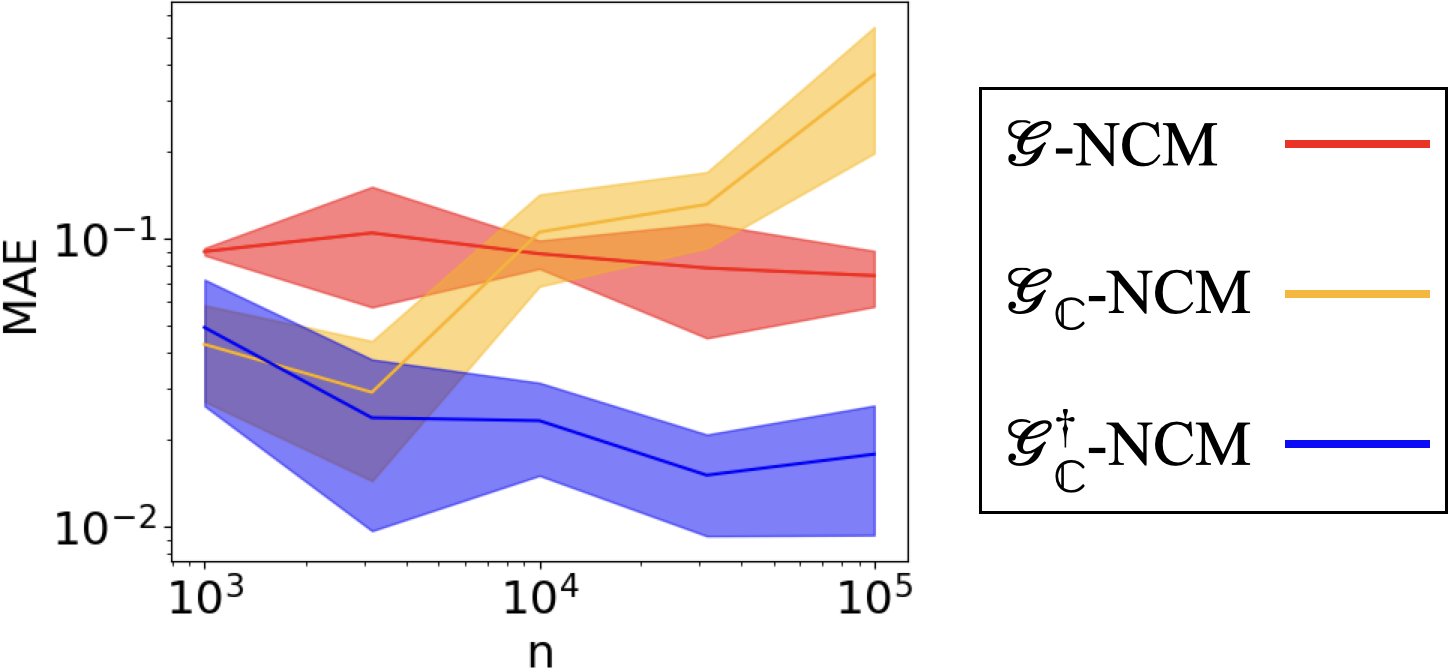}
\vspace{-0.05in}
\caption{Mean absolute error (MAE) v.\ number of samples for the MNIST estimation task. Comparisons between an abstractionless approach (red), a C-DAG approach (yellow), and a projected C-DAG approach (blue).}
\label{fig:mnist-est-results}
\end{figure}

\section*{Impact Statement}

This paper presents work whose goal is to advance the field of causal inference, a subfield of machine learning. The results in this paper may have implications bringing together strong practical results in representation learning and computer vision research with the explainability and generalizability of causal inference results. The trend is that this will lead to smarter AI, which itself has many consequences out of the scope of this work, but the benefit of understanding causal inference is that it can lead to less bias and more accountability of AI models.

\section*{Acknowledgements}

This research is supported in part by the NSF, ONR, AFOSR, DoE, Amazon, JP Morgan, and The Alfred P. Sloan Foundation.

%% file: section/A_proofs.tex
\section{Proofs}
\label{app:proofs}

This section contains all proofs of the paper as well as the necessary theoretical background and notation used.

\subsection{Notation}
\label{app:notation-table}

This paper uses a heavy amount of notation, combining notation from the counterfactual inference, causal generative models, and causal abstractions literature. See Table \ref{tab:notation} for a summary of the notation used in this paper.

\begin{table}[h]
    \center
    \begin{tabular}{l|l}
    \hline \hline
    Notation & Description \\ \hline \hline
    Uppercase letters ($X$)       & Random variables            \\ \hline
    Lowercase letters ($x$)       & Corresponding values of random variables            \\ \hline
    Bold uppercase letters ($\*X$) & Set of random variables \\ \hline
    Bold lowercase letters ($\*x$) & Corresponding set of values \\ \hline
    $\cD_{X}$ & Domain of $X$ \\ \hline
    $\cD_{\*X}$ & $\cD_{X_1} \times \dots \times \cD_{X_k}$ for $\*X = \{X_1, \dots, X_k\}$ \\ \hline
    $\cM$ & SCM (Def.~\ref{def:scm}) \\ \hline
    $\*U$ & Exogenous variables of an SCM \\ \hline
    $\*V$ & Endogenous variables of an SCM \\ \hline
    $\cF = \{f_{V_i} : V_i \in \*V\}$ & Set of functions of an SCM \\ \hline
    $\*U_{V_i}$ & Subset of $\*U$ that are inputs to $f_{V_i}$ \\ \hline
    $\Pai{V_i}$ & Subset of $\*V$ that are inputs to $f_{V_i}$ \\ \hline
    $P(\*X = \*x)$ or $P(\*x)$ & Probability of $\*X$ taking values $\*x$ under distribution $P(\*X)$ \\ \hline
    $P^{\cM}(\*X = \*x)$ or $P^{\cM}(\*x)$ & Probability of $\*X$ taking values $\*x$ in SCM $\cM$ \\ \hline
    $\cG$ & Causal diagram (Def.~\ref{def:cg}) \\ \hline
    Variable subscript ($\*Y_{\*x}$, $\*Y_{i[\*x_i]}$) & $\*Y$ ($\*Y_i$) under intervention on $\*X = \*x$ ($\*X_i = \*x_i$)  \\ \hline
    Star in subscript ($\*Y_*$) & Counterfactual term set (Def.~\ref{def:l3-semantics}) \\ \hline
    $\cL_i$ & Distributions of Layer $i$ of the PCH \\ \hline
    $\bbZ$ & Set of interventional distributions $P(\*V_{\*z_k})_{k=1}^{\ell}$ (e.g., available data) \\ \hline
    $\cL_i(\cM)$, $\bbZ(\cM)$ & Set of distributions induced by $\cM$ \\ \hline
    $\bbC$, $\bbD$ & Intervariable and intravariable clusters (Def.~\ref{def:var-clusterings}) \\ \hline
    Subscript with $L$ ($\cM_L$, $\*V_L$) & Low-level model/variables \\ \hline
    Subscript with $H$ ($\cM_H$, $\*V_H$) & High-level model/variables \\ \hline
    $\tau$ & Mapping from $\cD_{\*V_L}$ to $\cD_{\*V_H}$ \\ \hline
    $\delta$ & Partial SCM projection mapping (Prop.~\ref{prop:partial-scm-projection}) \\ \hline
    Superscript with $o$ ($\*W^o$) & Observed portion of projection (Prop.~\ref{prop:partial-scm-projection}) \\ \hline
    Superscript with $u$ ($\*W^u$) & Unobserved portion of projection (Prop.~\ref{prop:partial-scm-projection}) \\ \hline
    $\sigma_{\*X}$ & Soft (probabilistic) intervention applied on $\*X$ \\ \hline
    $Q$ & Query representing a probability distribution \\ \hline
    $Q^{\cM}$ & Query $Q$ evaluated in SCM $\cM$ \\ \hline
    $\tau(Q)$ & High-level corresponding query of $Q$ (Def.~\ref{def:q-tau-consistency}) \\ \hline
    $\*V_H^{\dagger}$ & AIC violation variables (Def.~\ref{def:aic-violators}) \\ \hline
    $\cG_{\bbC}$ & Cluster causal diagram (CDAG, Def.~\ref{def:cdag}) \\ \hline
    $\cG_{\bbC}^{\dagger}$ & Partially projected CDAG (Def.~\ref{def:proj-cdag}) \\ \hline
    $\Omega$ & Set of SCMs \\ \hline
    Hat ($\widehat{M}$, $\widehat{\cF}$, $\widehat{\tau}$) & Neural-parameterized variant (Def.~\ref{def:gncm}) \\ \hline
    $\bm{\theta}$ & Parameters of a neural-parameterized object \\ \hline
    \end{tabular}
    \caption{Table of notation.}
    \label{tab:notation}
\end{table}

\subsection{Extended Preliminaries}
\label{app:extended-prelim}

Given the brevity of Sec.~\ref{sec:prelims}, the definitions are explained further in this section.

\begin{customdef}{\ref{def:scm}}[Structural Causal Model (SCM)]
    A structural causal model $\cM$ is a 4-tuple $\langle \*U, \*V, \cF, P(\*U) \rangle$, where
    \begin{itemize}
        \item $\*U$ is a set of background (exogenous) variables that are determined by factors outside the model;
        \item $\*V$ is a set $\{V_1, V_2, \dots, V_n\}$ of variables, called endogenous, that are determined by other variables in the model -- that is, variables in $\*U \cup \*V$;
        \item $\cF$ is a set of functions $\{f_{V_1}, f_{V_2}, \dots, f_{V_n}\}$ such that each $f_{V_i}$ is a mapping from exogenous parents $\Ui{V_i} \subseteq \*U$ and endogenous parents $\Pai{V_i} \subseteq \*V \setminus V_i$ to $V_i$;
        \item $P(\*U)$ is a probability function defined over $\cD_{\*U}$.
        \hfill $\blacksquare$
    \end{itemize}
\end{customdef}
An SCM $\cM$ consists of endogenous variables $\*V$, exogenous variables $\*U$ with distribution $P(\*U)$, and mechanisms $\cF$. $\cF$ contains functions $f_{V_i}$ (for all $V_i \in \*V$) that map endogenous parents $\Pai{V_i}$ and exogenous parents $\Ui{V_i}$ to $V_i$. Given a fixed instance of $\*U \gets \*u$, each function $f_{V_i} \in \cF$ (when evaluated in topological order) deterministically outputs a value for $V_i$, which is passed into downstream functions. Collectively, the entire set of $\cF$ maps $\*U$ to $\*V$.

This provides the semantical framework for defining the three layers of the Pearl Causal Hierarchy (PCH).

\begin{definition}[Layer 1 Valuation {\citep[Def.~2]{bareinboim:etal20}}]
    \label{def:l1}
    An SCM $\cM = \langle \*U, \*V, \cF, P(\*U) \rangle$ defines a joint probability distribution $P^{\cM}(\*V)$ such that for each $\*Y \subseteq \*V$:
    \begin{equation}
        P^{\cM}(\*y) = \int_{\cD_{\*U}} \mathbbm{1}\{\*Y(\*u) = \*y\} dP(\*u)
    \end{equation}
    where $\*Y(\*u)$ is the solution for $\*Y$ after evaluating $\cF$ with $\*U = \*u$.
    \hfill $\blacksquare$
\end{definition}

In words, as discussed above, $\cL_1$ quantities are evaluated by fixing $\*U$ and passing the values topologically through $\cF$ to obtain a value for $\*Y \subseteq \*V$.

\begin{definition}[Layer 2 Valuation {\citep[Def.~5]{bareinboim:etal20}}]
    \label{def:l2}
    An SCM $\cM = \langle \*U, \*V, \cF, P(\*U) \rangle$ induces a family of joint distributions over $\*V$, one for each intervention $\*x$. For each $\*Y \subseteq \*V$:
    \begin{equation}
        P^{\cM}(\*y_{\*x}) = \int_{\cD_{\*U}} \mathbbm{1}\{\*Y_{\*x}(\*u) = \*y\} dP(\*u)
    \end{equation}
    where $\*Y_{\*x}(\*u)$ is the solution for $\*Y$ in the submodel $\cM_{\*x} = \langle \*U, \*V, \cF_{\*x}, P(\*U) \rangle$, where $\cF_{\*x} := \{f_{V}: V \in \*V \setminus \*X\} \cup \{f_X \gets x : X \in \*X\}$.
    \hfill $\blacksquare$
\end{definition}

In words, $\cL_2$ quantities are evaluated by performing the same topological evaluation, but this time through $\cF_{\*x}$, the set of functions obtained from applying the mutilation procedure, setting the functions of $\*X$ directly to the values $\*x$.

\begin{definition}[Layer 3 Valuation {\citep[Def.~7]{bareinboim:etal20}}]
    An SCM $\cM = \langle \*U, \*V, \cF, P(\*U) \rangle$ induces a family of joint distributions over counterfactual events $\*Y_{1[\*x_1]}, \*Y_{2[\*x_2]}, \dots$ for any $\*Y_i, \*X_i \subseteq \*V$:
    \begin{equation}
        \label{eq:l3}
        P^{\cM}(\*y_{1[\*x_1]}, \*y_{2[\*x_2]}, \dots) = \int_{\cD_{\*U}} \mathbbm{1}\{\*Y_{1[\*x_1]}(\*u) = \*y_1, \*Y_{2[\*x_2]}(\*u) = \*y_2, \dots\} dP(\*u).
    \end{equation}
    \hfill $\blacksquare$
\end{definition}

In words, $\cL_3$ quantities are evaluated by passing the same values of $\*U$ through multiple different submodels. Each intervention $\*x_i$ induces a different submodel $\cM_{\*x_i}$ through the mutilation procedure, with a new set of intervened functions $\cF_{\*x_i}$. Each $\*Y_i$, under the intervention $\*x_i$, must evaluate to $\*y_i$ under the same value of $\*u$ to contribute to Eq.~\ref{eq:l3}.

Note that the results of this work are general to all quantities of $\cL_3$, which includes $\cL_2$ and $\cL_1$. Specifically $\cL_2$ is the subset of $\cL_3$ such that there is only a single intervention $\*X_1$, and $\cL_1$ is the subset such that the single intervention $\*X_1 = \emptyset$.

Each SCM induces a graphical structure called a causal diagram, defined as follows.

\begin{definition}[Causal Diagram {\citep[Def.~13]{bareinboim:etal20}}]
    \label{def:cg}
    Each SCM $\cM$ induces a causal diagram $\cG$, constructed as follows:
    \begin{enumerate}
        \item add a vertex for each $V_i \in \*V$;
        \item add a directed arrow $(V_j \rightarrow V_i)$ for every $V_i \in \*V$ and $V_j \in \Pai{V_i}$; and
        \item add a dashed-bidirected arrow $(V_j  \dashleftarrow \dashrightarrow V_i)$ for every pair $V_i, V_j \in \*V$ such that $\Ui{V_i}$ and $\Ui{V_j}$ are not independent (Markovianity is not assumed).
        \hfill $\blacksquare$
    \end{enumerate}
\end{definition}

In words, a causal diagram visually represents variables and their functional dependencies.

In the context of causal abstraction theory, low-level variables can be constructively mapped to higher-level variables through clustering, defined below.

\interintraclusters*

In words, intervariable clusters group low-level variables together to form high-level variables while intravariable clusters group low-level values together to form high-level values. As an example, suppose $V_1, V_2, \dots, V_{100} \in \*V_L$ are binary variables that represent votes for approving a new law, but instead of storing individual voter information, it is preferred to work with an abstraction that uses a variable $V_T$ representing the total votes for the new law. Then, there may be a cluster $\*C \in \bbC$ such that $\*C = \{V_1, \dots, V_{100}\}$ that corresponds to $V_T$. Moreover, it is not necessary that $V_T$ takes a value that is a tuple of $(V_1, V_2, \dots, V_{100})$. For example, if 60 people voted for the new law, it does not matter if the votes are coming from $V_1$ to $V_{60}$ or if they were some other combination. Then, while $\cD_{\*C}$ contains all possible tuples of $(V_1, \dots, V_{100})$, some of the values can be clustered together. For example $\cD_{\*C}^{(60)}$ can be the subset of $\cD_{\*C}$ such that $\sum_{i} V_i = 60$. Each $\cD_{\*C}^{j}$ corresponds to a different value of $V_T$ (e.g., $\cD_{\*C}^{(60)}$ is the set of values that corresponds with $V_T = 60$). The definition of inter/intravariable clusters outlines a natural abstraction function $\tau$, defined below.

\consabsfunc*

Points 1 and 2 simply refer to the idea that intervariable clusters correspond to high-level variables while intravariable clusters (e.g., $\*C$ corresponds with $V_T$) correspond to high-level values (e.g., $\cD_{\*C}^{(60)}$ corresponds to $V_T = 60$). Point 3 simply states that $\tau$ separately maps each individual low-level intervariable and intravariable cluster to their corresponding high-level variable and value.

Arbitrarily choosing intravariable clusters can be problematic because clustering too many values together can result in a loss of information. This is characterized by the abstract invariance condition (AIC), defined below.

\aicdef*

In words, the AIC is violated if two values that are clustered together (and therefore map to the same high-level value) have different effects on downstream variables. For example, suppose $Y \in \*V_L$ is a variable that represents whether the law in the above example passes. However, it turns out that $V_1$ is the only vote that matters for determining $Y$. Then, clustering a value with $V_1 = 0$ together with another value with $V_1 = 1$ would violate the AIC since it is ambiguous which value should be used for determining $Y$. See Example \ref{ex:noaic-issue} for more details.

\subsection{Additional Definitions}
\label{app:additional-defs}

The following definitions about $\tau$-abstractions from \citet{beckers2019abstracting} set the groundwork for many discussions on abstraction theory.

\begin{restatable}[$\tau$-Abstraction {\citep[Def.~3.13]{beckers2019abstracting}}]{definition}{beckerstauabs}
    \label{def:tau-abs}
    Let $\cM_L = \langle \*U_L, \*V_L, \cF_L, P(\*U_L) \rangle$ and $\cM_H = \langle \*U_H, \*V_H, \cF_H, P(\*U_H)\rangle$ be two SCMs. Let $\cI_L$ and $\cI_H$ be the sets of allowed interventions respectively. Given $\tau: \cD_{\*V_L} \rightarrow \cD_{\*V_H}$, we say that $(\cM_H, \cI_H)$ is a $\tau$-abstraction of $(\cM_L, \cI_L)$ if:
    \begin{enumerate}
        \item $\tau$ is surjective;
        \item There exists surjective $\tau_{\*U}: \cD_{\*U_L} \rightarrow \cD_{\*U_H}$ that is compatible with $\tau$, i.e.
        \begin{equation}
            \label{eq:tau-u-compatibility}
            \tau(\cM_{L[\*X_L \gets \*x_L]}(\*u_L)) = \cM_{H[\omega_{\tau}(\*X_L \gets \*x_L)]}(\tau_{\*U}(\*u_L)),
        \end{equation}
        for all $\*u_L \in \cD_{\*U_L}$ and all $(\*X_L \gets \*x_L) \in \cI_L$;
        \item $\cI_H = \omega_{\tau}(\cI_L)$.
    \end{enumerate}
    \hfill $\blacksquare$
\end{restatable}

Further, we will assume that if $(\cM_H, \cI_H)$ is a $\tau$-abstraction of $(\cM_L, \cI_L)$, then $P(\*U_H) = \tau_{\*U}(P(\*U_L)) = P(\tau_{\*U}(\*U_L))$, that is, the distribution of $P(\*U_H)$ can be obtained from $P(\*U_L)$ via the push-forward measure through $\tau_{\*U}$. While it is not explicitly stated in the definition, this property aligns with the intention of linking the spaces of $\*U_L$ and $\*U_H$ through $\tau_{\*U}$.

\begin{restatable}[Strong $\tau$-Abstraction {\citep[Def.~3.15]{beckers2019abstracting}}]{definition}{beckersstrongtauabs}
    \label{def:strong-tau-abs}
    We say that $\cM_H$ is a strong $\tau$-abstraction of $\cM_L$ if $(\cM_H, \cI_H)$ is a $\tau$-abstraction of $(\cM_L, \cI_L)$ and $\cI_H = \cI_H^*$.
    \hfill $\blacksquare$
\end{restatable}

\begin{restatable}[Constructive $\tau$-Abstraction {\citep[Def.~3.19]{beckers2019abstracting}}]{definition}{beckersconstauabs}
    \label{def:cons-tau-abs}
    $\cM_H$ is a constructive $\tau$-abstraction of $\cM_L$ if $\cM_H$ is a strong $\tau$-abstraction of $\cM_L$, and there exists a partition of $\*V_L$, $\bbC = \{\*C_1, \*C_2, \dots, \*C_{n+1}\}$ (where $n = |\*V_H|$) with nonempty $\*C_1$ to $\*C_n$, such that $\tau$ can be decomposed as $\tau = (\tau_{\*C_1}, \tau_{\*C_2}, \dots, \tau_{\*C_n})$, where each $\tau_{\*C_i} : \cD_{\*C_i} \rightarrow \cD_{V_{H, i}}$ maps the $i$th partition to the $i$th variable of $\*V_H$.
    \hfill $\blacksquare$
\end{restatable}

For constructive abstraction functions, there is a notion of $Q$-$\tau$ consistency that connects low and high-level quantities. The formal definition is below.

\begin{definition}[$Q$-$\tau$ Consistency {\citep[Def.~7]{xia:bareinboim24}}]
    \label{def:q-tau-consistency}
    Let $\cM_L$ and $\cM_H$ be SCMs defined over variables $\*V_L$ and $\*V_H$, respectively. Let $\tau: \cD_{\*V_L} \rightarrow \cD_{\*V_H}$ be a constructive abstraction function w.r.t.~clusters $\bbC$ and $\bbD$. Let
    \begin{equation}
        \label{eq:q-valid}
        Q = \sum_{\*y_{L, *} \in \cD_{\*Y_{L, *}}(\*y_{H, *})} P(\*Y_{L, *} = \*y_{L, *})
    \end{equation}
    be a low-level Layer 3 quantity of interest (for some $\*y_{H, *} \in \cD_{\*Y_{H, *}}$), as expressed in Eq.~\ref{eq:valid-low-ctf}, and let
    \begin{equation}
        \label{eq:tauq-valid}
        \tau(Q) = P(\*Y_{H, *} = \*y_{H, *})
    \end{equation}
    be its high level counterpart.
    We say that $\cM_H$ is $Q$-$\tau$ consistent with $\cM_L$ if
    \begin{equation}
        \label{eq:q-tau-consistency}
        \begin{split}
            & \sum_{\*y_{L, *} \in \cD_{\*Y_{L, *}}(\*y_{H, *})} P^{\cM_L}(\*Y_{L, *} = \*y_{L, *}) \\
            &= P^{\cM_H}(\*Y_{H, *} = \*y_{H, *}),
        \end{split}
    \end{equation}
    that is, the value of $Q$ induced by $\cM_L$ is equal to the value of $\tau(Q)$ induced by $\cM_H$\footnote{Note that the equality in Eq.~\ref{eq:q-tau-consistency} is consistent with the push-forward measure through $\tau$.}. Furthermore, if $\cM_H$ is $Q$-$\tau$ consistent with $\cM_L$ for all $Q \in \cL_i(\cM_L)$ of the form of Eq.~\ref{eq:q-valid}, then $\cM_H$ is said to be $\cL_i$-$\tau$ consistent with $\cM_L$.
    \hfill $\blacksquare$
\end{definition}

The following result relates constructive abstraction functions and the concept of $Q$-$\tau$ consistency with $\tau$-abstractions.

\begin{proposition}[Abstraction Connection {\citep[Prop.~1]{xia:bareinboim24}}]
    \label{prop:abs-connect}
    Let $\tau: \cD_{\*V_L} \rightarrow \cD_{\*V_H}$  be a constructive abstraction function (Def.~\ref{def:tau}). $\cM_H$ is $\cL_3$-$\tau$ consistent (Def.~\ref{def:q-tau-consistency}) with $\cM_L$ if and only if there exists SCMs $\cM_L'$ and $\cM_H'$ s.t.~$\cL_3(\cM_L') = \cL_3(\cM_L)$, $\cL_3(\cM_H') = \cL_3(\cM_H)$, and $\cM_H'$ is a constructive $\tau$-abstraction of $\cM'_L$.
    
    \hfill $\blacksquare$
\end{proposition}


For abstraction inference, C-DAGs can often be leveraged in place of causal diagrams, defined below.

\begin{definition}[Cluster Causal Diagram (C-DAG) {\citep[Def.~1]{anand:etal23}}]
    \label{def:cdag}
    Given a causal diagram $\cG = \langle \*V, \*E \rangle$ and an admissible clustering $\bbC = \{\*C_1, \dots, \*C_k\}$ of $\*V$, construct a graph $\cG_{\bbC} = \langle \bbC, \*E_{\bbC}\rangle$ over $\bbC$ with a set of edges $\*E_{\bbC}$ defined as follows:
    \begin{enumerate}
        \item A directed edge $\*C_i \rightarrow \*C_j$ is in $\*E_{\bbC}$ if there exists some $V_i \in \*C_i$ and $V_j \in \*C_j$ such that $V_i \rightarrow V_j$ is an edge in $\*E$.

        \item A dashed bidirected edge $\*C_i \leftrightarrow \*C_j$ is in $\*E_{\bbC}$ if there exists some $V_i \in \*C_i$ and $V_j \in \*C_j$ such that $V_i \leftrightarrow V_j$ is an edge in $\*E$.
        \hfill $\blacksquare$
    \end{enumerate}
\end{definition}

This paper shows that they are insufficient for inferences when the AIC does not hold, but they are used as the base graph for constructing projected C-DAGs.

\subsection{Proofs of Sec.~\ref{sec:soft-abs}}
In this section, we prove the theoretical results stated in Sec.~\ref{sec:soft-abs}.

The first observation is that although the AIC is a property of the entire abstraction, one can clearly distinguish individual high-level variables that violate the AIC, as shown in the following definition.

\begin{definition}[AIC Violation Set]
    \label{def:aic-violators}
    Let $\cM_L$ be an SCM defined over $\*V_L$ and $\tau: \cD_{\*V_L} \rightarrow \cD_{\*V_H}$ be a constructive abstraction function w.r.t.~ clusters $\bbC$ and $\bbD$. Let $\*V_H^\dagger \subseteq \*V_H$ be the set of high-level variables such that $V_{H, i} \in \*V_H^\dagger$ iff there exists $V_{H, j} \in \*V_H$ with $V_{H, i} \in \Pai{V_{H, j}}$ such that Eq.~\ref{eq:aic} is violated for $\*C_j$, some $\*u \in \cD_{\*U}$, and some $\*v_1, \*v_2 \in \cD_{\*V_L}$ where $\*v_1$ and $\*v_2$ only differ in the values associated with $\*C_i$ ($\*C_i$ and $\*C_j$ are the corresponding clusters of $\bbC$ respectively). $\*V_H^\dagger$ is called the AIC violation set of $\tau$ w.r.t.~$\cM_L$.
    \hfill $\blacksquare$
\end{definition}

In words, a high-level variable is an AIC violator if two of its values that have different effects on its children are clustered together (e.g., $X$ is an AIC violator in Ex.~\ref{ex:noaic-issue}). Now recall the definition of an SCM projection.

\begin{restatable}[SCM Projection]{proposition}{scmproj}
    \label{def:scm-proj}
    Given an SCM $\cM = \langle \*U, \*V, \cF, P(\*U) \rangle$, there exists an SCM $\cM' = \langle \*U, \*W, \cF', P(\*U) \rangle$ such that, for all $\*u \in \cD_{\*U}$, $\*X \subseteq \*W$, and $\*x \in \cD_{\*X}$,
    \begin{equation}
        \cM_{\*x}(\*u)[\*W] = \cM'_{\*x}(\*u)
    \end{equation}
    $\cM'$ is called an SCM projection of $\cM$ over $\*W$.
    \hfill $\blacksquare$
\end{restatable}

\begin{proof}
    We show how to construct $\cM'$. For each $Y \in \*V$ in topological order according to the inputs of the functions of $\cF$, choose $f'_Y \gets f_Y(\Ui{Y}, \Pai{Y})$, where for each $X \in \Pai{Y}$,
    \begin{enumerate}
        \item If $X \in \*W$, then keep $X$ as an input of $f'_Y$;
        \item Otherwise if $X \notin \*W$, then replace $X$ with $f'_X(\Ui{X}, \Pai{X})$. Denote $\Ui{Y}'$ and $\Pai{Y}'$ as the new exogenous variables and parents of $Y$ after recursively applying this rule until all endogenous inputs are in $\*W$.
    \end{enumerate}
    Then, construct $\cF' = \{f'_Y; Y \in \*W\}$ and $\cM' = \langle \*U, \*W, \cF', P(\*U)\rangle$. Note that for all $\*u \in \cD_{\*U}$, $\*X \subseteq \*W$, and $\*x \in \cD_{\*X}$,
    \begin{align}
        & \cM_{\*x}(\*u)[\*W] \\
        &= \*W_{\*x}(\*u) \\
        &= \{f_Y(\ui{Y}, \Pai{Y[\*x]}(\*u)): Y \in \*W\} \\
        &= \{f'_Y(\ui{Y}', \Pai{Y[\*x]}'(\*u)): Y \in \*W\} \\
        &= \cM'_{\*x}(\*u).
    \end{align}
\end{proof}

Note that a version of this proposition was proven in \citet{lee:bar19a}, specifically for all $\cL_2$ queries of $\cM'$. Our proof uses a similar argument, but we show that the implied result is stronger: $\cM'$ matches $\cM$ on the SCM level for all exogenous settings $\*u$ and interventions $\*x$. This implies not only matching in $\cL_2$ query but also $\cL_3$ queries.

\partialscmproj*

\begin{proof}
    $\cM'$ can be created through Alg.~\ref{alg:map-abstraction}, and Thm.~\ref{thm:alg-construct-mh} proves that Alg.~\ref{alg:map-abstraction} is sound. See the proof of Thm.~\ref{thm:alg-construct-mh} for details on this construction.
\end{proof}

\algconstructmh*

\begin{proof}
    Let $\cM_H$ be the output from Alg.~\ref{alg:map-abstraction} given $\cM_L$ and $\tau$ constructed from clusters $\bbC$ and $\bbD$. Let $\tau(Q) = P(Y_{H, *} = \*y_{H, *})$ be any arbitrary high-level query from $\cL_3(\cM_H)$, and let $Q = \sum_{\*y_{L, *} \in \cD_{Y_{L, *}}(\*y_{H, *})}P(\*Y_{L, *} = \*y_{L, *})$ be its low-level counterpart. Without loss of generality, we can assume that the set of AIC violators $\*V_H^{\dagger} = \*V_H$, since any variable $V \notin \*V_H^{\dagger}$ can be mapped by a trivial $\delta$ that ignores $V^u$.

    We first show that $\cM_H$ is a projected abstraction of $\cM_L$. First, consider the SCM $\cM'_H$ defined over the variables $\*V'_H = \tau'(\*V_L)$, where $\tau'$ is the constructive abstraction function constructed from the same intervariable clusters $\bbC$ and the trivial intravariable clusters $\bbD' = \cD_{\*V_L}$ (i.e., each value of $\*V_L$ is its own cluster, and $\bbD$ is ignored). Note that each $V_j \in \*V_H'$ corresponds to a cluster $\*C_j \in \bbC$. Suppose $\cM'_H$ is constructed such that it is $\cL_3$-$\tau$ consistent with $\cM_L$, which is possible through Alg.~1 of \citet{xia:bareinboim24}. Then, Prop.~\ref{prop:abs-connect} states that $\cM'_H$ must be $\cL_3$-consistent with a constructive $\tau$-abstraction of $\cM_L$. Without loss of generality, suppose $\cM_H'$ is this constructive $\tau$-abstraction.

    For any $\*C_j \in \bbC$, one can construct variables $X_H^o$ and $X_H^u$ such that $X_H^o = \tau(\*C_j)$ and there exists a function $\delta$ such that $\delta(X_H^o, X_H^u) = \*C_j$, as done so in line 4 of the algorithm. This can be done by simply giving $X_H^u$ an arbitrarily large domain and using $X_H^u$ to disambiguate any information lost in the transformation from $\*C_j$ to $X_H^o$ when constructing $\delta$. Note that in the construction of $\cM_H$, the variables of $X_H^u$ are placed into $\*U_H$.
    
    To show that $\cM_H$ is a projected abstraction, we must show that it is a partial SCM projection of $\cM'_H$. Looking at Prop.~\ref{prop:partial-scm-projection}, simply choose $\*W = \*V'_H$. In Alg.~\ref{alg:map-abstraction}, each $\*C_j \in \bbC$ is split into $\*C_j^o, \*C_j^u$ such that $\delta(\*C_j^o, \*C_j^u) = \*C_j$. Denote $\*V_H^o$ and $\*V_H^u$ as the corresponding sets of variables in $\*V'_H$. By construction, indeed $\*U_H = \*U_L \cup \*V_H^u$, and $\*V_H = \*V_H^o$. Fix $\*u_L \in \*U_L$, $\*X \subseteq \*V'_H$, and $\*x \in \cD_{V'_H}$. Let $\delta(\*v_H^o, \*v_H^u) = \*V'_{H[\*x]}(\*u_L)$. Let $\*Z^u = \*V_H^u \setminus \*X^u$ and $\*z^u$ be the corresponding values from $\*w_{\*x}^u$. Then, observe that
    \begin{align}
        & \*v_H^o \\
        &= \tau(\*V'_{H[\*x]}(\*u_L)) \\
        &= \tau(\{f'_Y(\ui{Y}, \Pai{Y[\*x]}'(\*u)): Y \in \*V_H'\}) \\
        &= \tau(\{f'_Y(\ui{Y}, \delta(\Pai{Y[\*x]}^o(\*u), \Pai{Y[\*x]}^u(\*u))): Y \in \*V_H'\}) \\
        &= \tau(\{f_{H, Y}(\ui{Y}, \Pai{Y[\*x]}^o(\*u), \Pai{Y[\*x]}^u(\*u)): Y \in \*V_H\}) \\
        &= \cM_{H[\*x_o]}(\*u_L, \*x^u, \*z^u),
    \end{align}
    matching Eq.~\ref{eq:partial-scm-projection}.

    Now we show that $\cM_H$ is $\cL_3$-$\tau$ consistent with $\cM_L$. Denote $\*x_{L, *}$ and $\*x_{H, *}$ as the corresponding sets of interventions of $Q$ and $\tau(Q)$ respectively, and denote $\*y_{L,[\*x_{L, *}]}$ as the the values of $\*y_{L, *}$ specifically under the hard intervention $\*x_{L, *}$ (as opposed to the soft interventions under $\sigma_{\*X_{L, i}}$). Denote $\cD_{\*U_L}(\*y_{L,[\*x_{L, *}]}) \subseteq \cD_{\*U_L}$ as the values of $\*U$ such that $\*Y_{L,[\*x_{L, *}]}(\*u_L) = \*y_{L, *}$ (similar notation applies to $\cD_{\*U_H}$).

    Now observe that
    {
    \allowdisplaybreaks
    \begin{align}
        & Q^{\cM_L} \\
        &= \sum_{\*y_{L, *} \in \cD_{Y_{L, *}}(\*y_{H, *})}P^{\cM_L}(\*Y_{L, *} = \*y_{L, *}) \label{eq:construct-proof-1} \\
        &= \sum_{\*y_{L, *} \in \cD_{Y_{L, *}}(\*y_{H, *})}\sum_{\*x_{L, *} \in \cD_{X_{L, *}}(\*x_{H, *})} P^{\cM_L}(\*Y_{L, *} = \*y_{L, *} \mid \sigma_{\*X_{L, *}} = \*x_{L, *})P(\sigma_{\*X_{L, *}} = \*x_{L, *}) \label{eq:construct-proof-2} \\
        &= \sum_{\*y_{L, *}, \*x_{L, *}} P(\*U_L \in \cD_{\*U_L}(\*y_{L,[\*x_{L, *}]}))P(\sigma_{\*X_{L, *}} = \*x_{L, *}) \label{eq:construct-proof-3} \\
        &= \sum_{\*y_{L, *}, \*x_{L, *}} P(\*U_L \in \cD_{\*U_L}(\*y_{L,[\*x_{L, *}]})) \prod_{\*c_j \in \*x_{L, *}}P(\sigma_{\*C_j} = \*c_j) \label{eq:construct-proof-4} \\
        &= \sum_{\*y_{L, *}, \*x_{L, *}} P(\*U_L \in \cD_{\*U_L}(\*y_{L,[\*x_{L, *}]})) \prod_{\*c_j \in \*x_{L, *}}P(\*c_j \mid \tau(\*c_j) = v_{H, j}, \Pai{V_{H, j}}(\Ui{\*C_j}), \*R_{\*V_H^c(V_{H, j})}(\Ui{\*C_j})) \label{eq:construct-proof-5} \\
        &= \sum_{\*y_{L, *}, \*x_{L, *}} P(\*U_L \in \cD_{\*U_L}(\*y_{L,[\*x_{L, *}]})) \prod_{\*c_j \in \*x_{L, *}}P(\delta(\*c_j^o, \*C_j^u) = \*c_j \mid \*U_L) \label{eq:construct-proof-6} \\
        &= \sum_{\*y_{L, *}, \*x_{L, *}} P(\*U_L \in \cD_{\*U_L}(\*y_{L,[\*x_{L, *}]}))P(\delta(\*x_{H, *}, \*X_{H, *}^u) = \*x_{L, *} \mid \*U_L) \label{eq:construct-proof-7} \\
        &= \sum_{\*y_{L, *}, \*x_{L, *}} P(\*U_L \in \cD_{\*U_L}(\*y_{L,[\*x_{L, *}]}), \delta(\*x_{H, *}, \*X_{H, *}^u) = \*x_{L, *}) \label{eq:construct-proof-8} \\
        &= P\left(\bigvee_{\*y_{L, *} \in \cD_{Y_{L, *}}(\*y_{H, *}), \*x_{L, *} \in \cD_{X_{L, *}}(\*x_{H, *})}\*U_L \in \cD_{\*U_L}(\*y_{L,[\*x_{L, *}]}), \delta(\*x_{H, *}, \*X_{H, *}^u) = \*x_{L, *} \right) \label{eq:construct-proof-9} \\
        &= P((\*U_L, \*V_H^u) \in \cD_{\*U_H}(\*y_{H, *})) \label{eq:construct-proof-10} \\
        &= P(\*U_H \in \cD_{\*U_H}(\*y_{H, *})) \label{eq:construct-proof-11} \\
        &= \tau(Q)^{\cM_H}. \label{eq:construct-proof-12}
    \end{align}
    }

    Explaining each line, line \ref{eq:construct-proof-1} starts by applying the definition of $Q$. Since the interventions $\*x_{L, *}$ are determined through soft interventions $\sigma$, we can expand the soft intervention through all possible values of $\*x_{L, *}$ (via marginalization), as done so in line \ref{eq:construct-proof-2}. The first term is simply computed as the probability of all values of $\*U_L$ where $\*Y_{L,[\*x_{L, *}]}(\*u_L) = \*y_{L, *}$ (Def.~\ref{eq:def:l3-semantics}), resulting in line \ref{eq:construct-proof-3}. The second term can be broken down into the soft interventions of each intravariable cluster (line \ref{eq:construct-proof-4}), whose probability is computed through Eq.~\ref{eq:low-intervention-general}, resulting in line \ref{eq:construct-proof-5}. Line \ref{eq:construct-proof-6} is true by construction of $\cM_H$, through line 8 of the algorithm. Finally, we consolidate all terms back into $\*x_{L, *}$ and merge back into the joint distribution in lines \ref{eq:construct-proof-7} and \ref{eq:construct-proof-8}. Line \ref{eq:construct-proof-9} holds because the probabilities of each individual value of $\*y_{L,[\*x_{L, *}]}$ are disjoint since $\*X_{L, *}$ and $\*Y_{L, *}$ can both only be equal to one value at a time. Line \ref{eq:construct-proof-10} holds since $\*U_H = \*U_L \cup \*V_H^u$ and by construction of line 9 in the algorithm, we have exhausted every possible value that $\cM_H((\*u_L, \*v_H^{u})) = \*y_{H, *}$. This allows us to finish the comparison with $\tau(Q)$ on lines \ref{eq:construct-proof-11} and \ref{eq:construct-proof-12}.

    Therefore, $Q^{\cM_L} = \tau(Q)^{\cM_H}$ for all $\tau(Q) \in \cL_3(\cM_H)$, concluding the proof.
\end{proof}


\subsection{Proofs of Sec.~\ref{sec:inference}}
The proofs in this section are concerned with the properties of the partially projected C-DAG in Def.~\ref{def:proj-cdag}.

First, we must define what it means for a causal graph to be ``sufficient''. In general, the role of causal graphs in causal inference tasks is typically to encode the constraints of the model, which are useful for allowing one to make inferences of higher layers using lower layer data. These constraints, on layers 1, 2, and 3 of the PCH, can be described by the Counterfactual Bayesian Network, defined below.
\begin{definition}[Counterfactual Bayesian Network (CTF-BN) {\citep[Def.~D.1, D.2]{correa2024ctfcalc}}]
    Let $\*P_{**}$ be the collection of all distributions of the form $P(W_{1[x_1]}, W_{2[x_2]}, \dots)$, where $W_i \in \*V$, $\*X_i \subseteq \*V$, $\*x_i \in \cD_{\*X_i}$. A directed acyclic graph (possibly with bidirected edges) $\cG$ is a Counterfactual Bayesian Network for $\*P_{**}$ if:
    \begin{enumerate}[label=(\roman*)]
        \item (Independence Restrictions) Let $\*W_*$ be a set of counterfactuals of the form $W_{\pai{w}}$, $\*Z_1, \dots, \*Z_{l}$ the c-components of $\cG[\*V(\*W_*)]$ (two variables are in the same c-component if there is a bidirected path between them in $\cG$ within the variables $\*V(\*W_*)$), and $\*Z_{1*}, \dots, \*Z_{l*}$ the corresponding partition over $\*W_*$. Then $P(\*W_*)$ factorizes as
        \begin{equation}
            \label{eq:ctf-factorization}
            P \left( \bigwedge_{W_{\pai{w} \in \*W_*}} W_{\pai{w}} \right) = \prod_{j=1}^{l} P \left( \bigwedge_{W_{\pai{w}} \in \*Z_{j*}} W_{\pai{w}} \right).
        \end{equation}

        \item (Local Exclusion Restrictions) For every variable $Y \in \*V$ with parents $\Pai{y}$ for every set $\*Z \subseteq \*V \setminus (\Pai{y} \cup \{Y\})$, and any counterfactual set $\*W_{*}$, we have
        \begin{equation}
            \label{eq:exclusion}
            P(Y_{\pai{y}, \*z}, \*W_*) = P(Y_{\pai{y}}, \*W_*).
        \end{equation}

        \item (Local Consistency) For every variable $Y$ with parents $\Pai{y}$, let $\*X \subseteq \Pai{y}$, then for every set $\*Z \subseteq \*V \setminus (\*X \cup \{Y\})$, and any set of counterfactuals $\*W_*$, we have
        \begin{equation}
            P(Y_{\*z} = y, \*X_{\*z} = \*x, \*W_*) = P(Y_{\*z \*x} = y, \*X_{\*z} = \*x, \*W_*).
        \end{equation}
    \end{enumerate}
    \hfill $\blacksquare$
\end{definition}

Although a full discussion of these constraints is out of the scope of this paper, the two that are particularly of insight in regards to the difference between C-DAGs and projected C-DAGs are points (i) and (ii) in the definition. In words, point (ii) is stating that a lack of a directed edge implies a lack of an interventional effect, and point (i) is says that a lack of a bidirected edge (or lack of unobserved confounding) implies independence of functions. One particularly useful result is that causal diagrams are guaranteed to satisfy the CTF-BN constraints of the distributions induced by the SCM that generated the graph, shown below.

\begin{lemma}[\citep{correa2024ctfcalc}]
    \label{lemma:cg-is-ctfbn}
    For any SCM $\cM$ inducing causal diagram $\cG$, $\cG$ is a CTF-BN for $\cL_3(\cM)$.
    \hfill $\blacksquare$
\end{lemma}

Using the constraints of the CTF-BN, we can state the results of Sec.~\ref{sec:inference} more formally.

\begin{customprop}{\ref{prop:cdag-insufficiency}}[C-DAG Insufficiency (Formal)]
    There exists SCM $\cM_L$ and constructive abstraction function $\tau$ defined over clusters $\bbC$ and $\bbD$ with C-DAG $\cG_{\bbC}$ such that, for $\cM_H$ that is $\cL_3$-$\tau$ consistent with $\cM_L$, $\cG_{\bbC}$ is not a CTF-BN for $\cL_3(\cM_H)$.
    \hfill $\blacksquare$
    
    \begin{proof}
        Since there exist projected C-DAGs that have a superset of the edges of the corresponding C-DAG, this result is a consequence of the necessity of projected C-DAGs, stated in Thm.~\ref{thm:proj-cdag-completeness}.
    \end{proof}
\end{customprop}

The above result states that C-DAGs are insufficient for the general case abstraction problem, where the AIC may be violated. The below result shows that projected C-DAGs have precisely the correct constraints.

\begin{lemma}
    \label{lemma:alg-is-proj-cdag}
    Let $\cM_H$ be the SCM generated from running Alg.~\ref{thm:alg-construct-mh} on $\cM_L$ and $\tau$. Then, the causal diagram of $\cM_H$ is the projected C-DAG of $\cM_L$ over $\*V_H$.
    \hfill $\blacksquare$

    \begin{proof}
        This proof considers a slight modification of Alg.~\ref{thm:alg-construct-mh} that incorporates AIC violators $\*V_H^{\dagger}$. Specifically, in line 9, $\Pai{V}$ can be split into $\Pai{V}^0$ and $\Pai{V}^{\dagger}$, where $\tau(\Pai{V}^0) \cap \*V_H^\dagger = \emptyset$, and $\delta (\pai{V}^o, \pai{V}^u)$ is only applied to $\Pai{V}^\dagger$, while for parents in $\Pai{V}^0$, $\delta$ is replaced by an arbitrary $\pai{V}^0$ such that $\tau(\pai{V}^0) = \pai{V_H}$, since they all map to the same value due to the lack of AIC violation.

        The causal diagram of $\cM_H$ is a graph $\cG_H = \langle \*V_H, \*E \rangle$. $\*E$ must at least contain the edges of the C-DAG $\cG_{\bbC}$, since every function of $\cM_L$ is incorporated into $\cM_H$. Extra edges are only added through line 9 of the algorithm, where $\delta$ may introduce new dependencies through $\pai{V}^u$. If, for some $W \in \Pai{V}$, $W \notin \*V_H^\dagger$, then new edges are not added w.r.t.~$W$. Otherwise, the existence of $W^u$ may confound other functions that also take $W^u$ as an input, implying rule 3 of Def.~\ref{def:proj-cdag}. For the other rules, as stated in line 8, $W^u$ depends on its own parents $\Pai{W}$ (or the grandparents of $V$), which implies rule 1 of Def.~\ref{def:proj-cdag}. Additionally, $W^u$ also depends on $\ui{\tau(W)}^c$, which implies a dependence on any unobserved confounder that influences the parents of $W$, implying rule 2 of Def.~\ref{def:proj-cdag}. No other dependences are introduced through lines 8 and 9 of the algorthm, meaning that $\*E$ contains precisely the edges of $\cG^{\dagger}$ plus those introduced by the rules of Def.~\ref{def:proj-cdag}. Hence, $\cG_H = \cG_{\bbC}^\dagger$.
    \end{proof}
\end{lemma}

\begin{customthm}{\ref{thm:proj-cdag-completeness}}[Projected C-DAG Sufficiency and Necessity (Formal)]
    Let $\cM_L = \langle \*U_L, \*V_L, \cF_L, P(\*U_L) \rangle$ be a low-level model with causal diagram $\cG$, and let $\tau: \cD_{\*V_L} \rightarrow \cD_{\*V_H}$ be a constructive abstraction function defined over clusters $\bbC$ and $\bbD$. Let $\*V_H^{\dagger}$ be the set of AIC violators of $\tau$. Let $\cG_{\bbC}^{\dagger} = \langle \*V_H, \*E \rangle$ be the partially projected C-DAG of $\cG$ w.r.t.\ $\bbC$ and $\*V_H^{\dagger}$. Let $\cM_H = \langle \*U_H, \*V_H, \cF_H, P(\*U_H) \rangle$ be a high-level model that is $\cL_3$-$\tau$ consistent with $\cM_L$. Then
    \begin{enumerate}
        \item (Sufficiency) $\cG_{\bbC}^{\dagger}$ is a CTF-BN for $\cL_3(\cM_H)$
        \item (Necessity) For any other graph $\cG' = \langle \*V_H, \*E' \rangle$ such that $\cG' \neq \cG$ and $\cG'$ is a CTF-BN for $\cL_3(\cM_H)$, it must be the case that $\*E \subset \*E'$.
    \end{enumerate}
    \hfill $\blacksquare$

    \begin{proof}
        The proof for sufficiency is straightforward. Alg.~\ref{thm:alg-construct-mh} generates $\cM_H$ that is $\cL_3$-$\tau$ consistent with $\cM_L$ by Thm~\ref{thm:alg-construct-mh}. By Lemma \ref{lemma:alg-is-proj-cdag}, the causal diagram of $\cM_H$ is $\cG_{\bbC}^{\dagger}$. Then, by Lemma \ref{lemma:cg-is-ctfbn}, $\cG_{\bbC}^{\dagger}$ must be a CTF-BN for $\cL_3(\cM_H)$.

        The proof for necessity is more involved. We argue that every single edge in $\*E$ must be included for $\cG_{\bbC}^{\dagger}$ to maintain the correct CTF-BN constraints.

        First, at least every edge in the C-DAG $\cG_{\bbC}$ is necessary. This is because every edge in the C-DAG corresponds to an edge in the original graph $\cG$. This edge cannot be removed without adding new constraints to the original graph generated by $\cM_L$.

        Next, we step through each of the three rules of Def.~\ref{def:proj-cdag} and argue that they must hold. For each of the rules, consider the basic case with $Z, X, Y \in \*V_H$ (and their corresponding clusters $\*C_Z, \*C_X, \*C_Y \in \bbC$).
        \begin{enumerate}
            \item If $Z \rightarrow X \rightarrow Y$ in $\cG_{\bbC}$ and $X \in \*V_H^{\dagger}$, consider the query $P(y_{\pai{y}, z})$. According to Eq.~\ref{eq:exclusion}, $P(y_{\pai{y}, z}) = P(y_{\pai{y}})$ if there is no edge from $Z \rightarrow Y$. However, we see that
            \begin{align}
                & P(y_{\pai{y}, z}) \\
                &= P(\*c_{Y[\sigma_{\Pai{\*C_Y}}, \sigma_{\*C_Z}]}) \\
                &= P(\*c_{Y[\sigma_{\Pai{\*C_Y} \setminus \*C_X}, \sigma_{\*C_X}, \sigma_{\*C_Z}]}) \\
                &= \sum_{\*c_X \in \cD_{\*C_X}} P(\*c_{Y[\sigma_{\Pai{\*C_Y} \setminus \*C_X}, \*c_X, \sigma_{\*C_Z}]})P(\sigma_{\*C_X} = \*c_X) \\
                &= \sum_{\*c_X \in \cD_{\*C_X}} P(\*c_{Y[\sigma_{\Pai{\*C_Y} \setminus \*C_X}, \*c_X, \sigma_{\*C_Z}]})P(\*c_X \mid \tau(\*c_X) = x, \pai{\*C_X}, \ui{\*C_X}^c). \label{eq:cdag-nec-step1}
            \end{align}
            Clearly, if $Z \notin \Pai{Y}$, then including $\sigma_{\*C_Z}$ into the low-level query can impact the value of the query, since $Z \in \Pai{X}$, so the right term in Eq.~\ref{eq:cdag-nec-step1} depends on $\sigma_{\*C_Z}$. This would break Eq.~\ref{eq:exclusion}.

            \item If $Z \xdasharrow[<->]{}  X \rightarrow Y$ in $\cG_{\bbC}$ and $X \in \*V_H^{\dagger}$, two types of edges must be considered.
            \begin{enumerate}
                \item If there is no bidirected edge between $Z$ and $Y$, then according to Eq.~\ref{eq:ctf-factorization}, $P(y_{\pai{y}}, z_{\pai{z}}) = P(y_{\pai{y}})P(z_{\pai{z}})$. However, we see that
                \begin{align}
                    & P(y_{\pai{y}}, z_{\pai{z}}) \\
                    &= P(\*c_{Y[\sigma_{\Pai{\*C_Y}}]}, \*c_{Z[\sigma_{\Pai{\*C_Z}}]}) \\
                    &= P(\*c_{Y[\sigma_{\Pai{\*C_Y} \setminus \*C_X}, \sigma_{\*C_X}]}, \*c_{Z[\sigma_{\Pai{\*C_Z}}]}) \\
                    &= \sum_{\*c_X \in \cD_{\*C_X}, \pai{\*C_Z} \in \cD_{\Pai{\*C_Z}}} P(\*c_{Y[\sigma_{\Pai{\*C_Y} \setminus \*C_X}, \*c_X]}, \*c_{Z[\pai{\*C_Z}]})P(\sigma_{\*C_X} = \*c_X)P(\sigma_{\Pai{\*C_Z}} = \pai{\*C_Z}).
                \end{align}
                From here,
                \begin{equation}
                    P(\sigma_{\*C_X} = \*c_X) = P(\*c_X \mid \tau(\*c_X) = x, \pai{\*C_X}, \ui{\*C_X}^c)
                \end{equation}
                and
                \begin{equation}
                    P(\*c_{Y[\sigma_{\Pai{\*C_Y} \setminus \*C_X}, \*c_X]}, \*c_{Z[\pai{\*C_Z}]})
                    = P\left( \*c_{Y[\sigma_{\Pai{\*C_Y} \setminus \*C_X}, \*c_X]}, \bigwedge_{V \in \*C_Z} f_V(\pai{V}, \ui{V}) = v \right)
                \end{equation}
                for $v$ consistent with $\*c_Z$. However, since there is a bidirected edge between $Z$ and $X$, there may be a dependence between $\Ui{V}$ for some $V \in \*C_Z$ and $\ui{\*C_X}^c$. This would make the independence between the two terms $y_{\pai{y}}$, $z_{\pai{z}}$ impossible, violating Eq.~\ref{eq:ctf-factorization}.

                \item If there is no bidirected edge between $X$ and $Y$, then according to Eq.~\ref{eq:ctf-factorization}, $P(y_{\pai{y}}, x_{\pai{x}}) = P(y_{\pai{y}})P(x_{\pai{x}})$. However, following the same argument as above, this cannot hold either because
                \begin{equation}
                    P(\*c_{Y[\sigma_{\Pai{\*C_Y} \setminus \*C_X}, \*c_X]}, \*c'_{X[\pai{\*C_X}]})
                    = P\left( \*c_{Y[\sigma_{\Pai{\*C_Y} \setminus \*C_X}, \*c_X]}, \bigwedge_{V \in \*C_X} f_V(\pai{V}, \ui{V}) = v \right)
                \end{equation}
                for $v$ consistent with $\*c'_X$. Certainly, there could be a dependence between $\*U_V$ for some $V \in \*C_X$ and $\*u_{\*C_X}^c$, since both terms influence the functionality of $\*C_X$.
                
            \end{enumerate}

            \item If $Z \leftarrow X \rightarrow Y$ in $\cG_{\bbC}$ and $X \in \*V_H^{\dagger}$, consider the query $P(y_{\pai{y}}, z_{\pai{z}})$. According to Eq.~\ref{eq:ctf-factorization}, $P(y_{\pai{y}}, z_{\pai{z}}) = P(y_{\pai{y}})P(z_{\pai{z}})$ if there is no bidirected edge between $Z$ and $Y$. However, we see that
            \begin{align}
                & P(y_{\pai{y}}, z_{\pai{z}}) \\
                &= P(\*c_{Y[\sigma_{\Pai{\*C_Y}}]}, \*c_{Z[\sigma_{\Pai{\*C_Z}}]}) \\
                &= P(\*c_{Y[\sigma_{\Pai{\*C_Y} \setminus \*C_X}, \sigma_{\*c_X}]}, \*C_{Z[\sigma_{\Pai{\*C_Z} \setminus \*C_X}, \sigma_{\*C_X}]}).
            \end{align}
            Note that $\sigma_{\*C_X}$ is computed once for both terms, so clearly the two terms cannot be independent as they both depend on $\sigma_{\*C_X}$. This would break Eq.~\ref{eq:ctf-factorization}.
        \end{enumerate}
        With all rules covered, no edge can be removed without breaking the CTF-BN condition, ensuring that the edge set of $\cG_{\bbC}^{\dagger}$ is minimal.
    \end{proof}
\end{customthm}

Finally, we prove how the projected C-DAG can be used for cross-layer inferences by solving the abstraction identification problem. First consider the classical identification problem.

For the following proofs, consider the classical definition of identifiability.
\begin{definition}
    \label{def:classic-id}
    Let $\Omega^*$ be the space containing all SCMs defined over endogenous variables $\*V$. We say that a causal query $Q$ is identifiable (ID) from the available data $\bbZ$ and the causal diagram $\cG$ if $Q(\cM_1) = Q(\cM_2)$ for every pair of models $\cM_1, \cM_2 \in \Omega^*$ such that $\cM_1$ and $\cM_2$ both induce $\cG$ and $\bbZ(\cM_1) = \bbZ(\cM_2)$.
    \hfill $\blacksquare$
\end{definition}

Now we show how abstract identification is equivalent.

\dualabsid*

\begin{proof}
    Let $\Omega_L$ and $\Omega_H$ be the space of SCMs defined over $\*V_L$ and $\*V_H$ respectively, and let $\Omega_L(\cG_{\bbC}^\dagger)$ and $\Omega_H(\cG_{\bbC}^\dagger)$ be their corresponding subsets that induce graph $\cG_{\bbC}^\dagger$. For clarity, $\cM_L \in \Omega_L(\cG_{\bbC}^\dagger)$ if $\cG_{\bbC}^\dagger$ is a partially projected C-DAG of its causal diagram $\cG$ w.r.t.~$\bbC$ and AIC violation set $\*V_H^{\dagger}$. $\cM_H \in \Omega_H(\cG_{\bbC}^\dagger)$ if $\cM_H$ induces $\cG_{\bbC}^\dagger$ as its causal diagram.
    
    If $Q$ is $\tau$-ID from $\cG_{\bbC}^\dagger$ and $\bbZ$, then every pair of $\cM_L \in \Omega_L(\cG_{\bbC}^\dagger), \cM_H \in \Omega_H(\cG_{\bbC}^\dagger)$ such that $\cM_H$ is $\bbZ$-$\tau$ consistent with $\cM_L$ must have $\cM_H$ be $Q$-$\tau$ consistent with $\cM_L$. For all such $\cM_H$, $\bbZ$-$\tau$ consistency and $Q$-$\tau$ consistency with $\cM_L$ implies that $\cM_H$ is $\tau(\bbZ)$-consistent and $\tau(Q)$-consistent by Def.~\ref{def:q-tau-consistency}. For any pair $\cM_1, \cM_2 \in \Omega_H$ that induce $\cG_{\bbC}^\dagger$, $\tau(\bbZ)(\cM_1) = \tau(\bbZ)(\cM_2)$ therefore implies that both $\cM_1$ and $\cM_2$ must be $\bbZ$-$\tau$ consistent with $\cM_L$ and must therefore both be $Q$-$\tau$ consistent, so $\tau(Q)(\cM_1) = \tau(Q)(\cM_2)$. Hence, $\tau(Q)$ is ID from $\cG_{\bbC}$ and $\tau(\bbZ)$ by Def.~\ref{def:classic-id}.

    Conversely, if $\tau(Q)$ is ID from $\cG_{\bbC}^\dagger$ and $\tau(\bbZ)$, then for any $\cM_1, \cM_2 \in \Omega_H$ that induces $\cG_{\bbC}^{\dagger}$ such that $\tau(\bbZ)(\cM_1) = \tau(\bbZ)(\cM_2)$, it must be the case that $\tau(Q)(\cM_1) = \tau(Q)(\cM_2)$. For every $\cM_L \in \Omega_L(\cG_{\bbC}^{\dagger})$, Thm.~\ref{thm:alg-construct-mh} and Lemma \ref{lemma:alg-is-proj-cdag} state that there exists some $\cM_H \in \Omega_H(\cG_{\bbC}^{\dagger})$ that is $\cL_3$-$\tau$ consistent with $\cM_L$, implying that $\cM_H$ is both $\bbZ$-$\tau$ consistent and $Q$-$\tau$ consistent with $\cM_L$. Since all $\cM_H \in \Omega_H(\cG_{\bbC}^\dagger)$ that match in $\tau(\bbZ)$ must also match in $\tau(Q)$, it must be the case that all such $\cM_H$ that are $\bbZ$-$\tau$ consistent with $\cM_L$ must also be $Q$-$\tau$ consistent with $\cM_L$. Hence, by definition, $Q$ is $\tau$-ID from $\cG_{\bbC}^\dagger$ and $\bbZ$.
\end{proof}

%% file: section/B_additional_results.tex
\section{Additional Results}
\label{app:add-results}

In this section, we add additional technical results and expand on the ideas presented in the main body.

\subsection{Choosing an Intervention Mapping Definition}
\label{app:choose-soft-intervention}

When the AIC is violated, a high-level intervention may be ambiguous because it is not clear which corresponding low-level intervention is being applied. This section discusses the derivation and reasoning of Eq.~\ref{eq:low-intervention-short}, used in this paper, which assigns a specific soft intervention over the corresponding low-level interventions, as well as explaining why some natural alternatives are undesirable. This disambiguation process is similar to the idea of disjunctive interventions from \citet{pearl:17-r359}, which aim to define interventions that include disjunctions (e.g., $do(X = x_1 \text{ or } X = x_2)$). Although the results of this section are specific to causal abstractions, some of the derivations provide similar or more general results.

Consider a basic setting with only two variables $\*V_L = \{X_L, Y\}$, where $X_L$ is ternary ($\cD_{X_L} = \{x_0, x_1, x_2\}$), and $Y$ is binary ($\cD_{Y} = \{y_0, y_1\}$). Let $\*V_H = \{X_H, Y\}$, where $\cD_{X_H} = \{x_A, x_B\}$, such that
\begin{equation}
    \tau(x_L, y) = 
    \begin{cases}
        (x_A, y) & x_L = x_0 \\
        (x_B, y) & x_L = x_1, x_2,
    \end{cases}
\end{equation}
that is, $x_1$ and $x_2$ are both mapped to the same high-level value $x_B$. Naturally, one may be interested in causal queries on the high-level model such as $P(Y_{X_H = x_B} = y_1)$. However, making no assumptions about the AIC or the structural equations and probability distributions of the low-level model, how would such a quantity be defined on the low-level?

When the AIC holds, the answer is simple, since the AIC would imply that $P(Y_{X_L = x_1} = y_1) = P(Y_{X_L = x_2} = y_1)$. Since both of these values are equal, it must be the case that $P(Y_{X_H = x_B} = y_1) = P(Y_{X_L = x_1} = y_1) = P(Y_{X_L = x_2} = y_1)$. When the AIC does not hold, however, the answer is ambiguous. It is possible that $P(Y_{X_L = x_1} = y_1) \neq P(Y_{X_L = x_2} = y_1)$, so the choice of $P(Y_{X_H = x_B} = y_1)$ is not clear.

To illustrate the full range of possible options of $P(Y_{X_H = x_B} = y_1)$, consider a perspective of the problem akin to the canonical model formulation used for causal partial identification \citep{balke:pea97, zhang:bareinboim21b}. Note that there are eight possible functions from $X_L$ to $Y$, since there are three possible values of $X_L$ and two possible values for $Y$. Define $R_X = f_X(\mathbf{U})$, $R_Y^0 = f_Y(X_L = x_0, \mathbf{U})$, $R_Y^1 = f_Y(X_L = x_1, \mathbf{U}), R_Y^2 = f_Y(X_L = x_2, \mathbf{U})$, all of which are random variables that depend on $\mathbf{U}$. Now define

\begin{equation}
    p_{ijk\ell} = P(R_X = x_i, R_Y^0 = y_j, R_Y^1 = y_k, R_Y^2 = y_{\ell}).
\end{equation}

Note that $Y_{X_L = x_1} = y_1$ holds as long as $R_Y^1 = y_1$ and $Y_{X_L = x_2} = y_1$ holds as long as $R_Y^2 = y_1$. Expanding this result, we get

\begin{align}
    & P(Y_{X_L = x_1} = y_1) \label{eq:soft-q-choice-x1} \\
    &= {\color{purple}p_{0010}} + p_{0011} + {\color{purple}p_{0110}} + p_{0111} \nonumber \\
    &+ {\color{blue}p_{1010}} + p_{1011} + {\color{blue}p_{1110}} + p_{1111} \nonumber \\
    &+ {\color{red}p_{2010}} + p_{2011} + {\color{red}p_{2110}} + p_{2111}, \nonumber \\
    & P(Y_{X_L = x_2} = y_1) \label{eq:soft-q-choice-x2} \\
    &= {\color{purple}p_{0001}} + p_{0011} + {\color{purple}p_{0101}} + p_{0111} \nonumber \\
    &+ {\color{red}p_{1001}} + p_{1011} + {\color{red}p_{1101}} + p_{1111} \nonumber \\
    &+ {\color{blue}p_{2001}} + p_{2011} + {\color{blue}p_{2101}} + p_{2111}. \nonumber
\end{align}

The terms that are colored black are terms that are contained in both equations. This implies that $P(Y_{X_H = x_B} = y_1)$ must at least contain all of the black terms and may potentially contain any of the colored terms to any proportion. In other words,
\begin{align}
&P(Y_{X_H = x_B} = y_1) \label{eq:soft-q-choice-min} \\
&\geq p_{0011} + p_{0111} \nonumber \\
&+ p_{1011} + p_{1111} \nonumber \\
&+ p_{2011} + p_{2111}, \nonumber
\end{align}
and
\begin{align}
&P(Y_{X_H = x_B} = y_1) \label{eq:soft-q-choice-max} \\
&\leq {\color{purple}p_{0010}} + {\color{purple}p_{0110}} + {\color{purple}p_{0001}} + {\color{purple}p_{0101}} + p_{0011} + p_{0111} \nonumber \\
&+ {\color{blue}p_{1010}} + {\color{blue}p_{1110}} + {\color{red}p_{1001}} + {\color{red}p_{1101}} + p_{1011} + p_{1111} \nonumber \\
&+ {\color{red}p_{2010}} + {\color{red}p_{2110}} + {\color{blue}p_{2001}} + {\color{blue}p_{2101}} + p_{2011} + p_{2111}. \nonumber
\end{align}

The question is then how to choose which of these colored terms to include in the definition of $P(Y_{X_H = x_B} = y_1)$. It is entirely possible to define $P(Y_{X_H = x_B} = y_1)$ as simply being equal to $P(Y_{X_L = x_1} = y_1)$ or $P(Y_{X_L = x_2} = y_1)$ (i.e., choosing Eq.~\ref{eq:soft-q-choice-x1} or \ref{eq:soft-q-choice-x2}). It could also be defined as Eq.~\ref{eq:soft-q-choice-min} or Eq.~\ref{eq:soft-q-choice-max}, which can be interpreted as the minimum or maximum possible value of the query. However, these choices are somewhat arbitrary and extreme---it is unlikely that a practitioner would intuitively mean one of these definitions when studying the high-level query $P(Y_{X_H = x_B} = y_1)$.

More specifically, the reason that the above definitions are undesirable is because they do not take into account the nuance of when $X_H = x_B$ should be interpreted as $X_L = x_1$ or as $X_L = x_2$. Indeed, all of the colored terms in the above equations show a disconnect between $Y_{X_L = x_1}$ and $Y_{X_L = x_2}$. For example, $\color{blue}p_{1010}$ represents a case where $R_Y^1 = y_1$ and $R_Y^2 = y_0$, which means that $Y$ will take the value of $y_1$ if $X_L = x_1$ and $y_0$ if $X_L = x_2$. In such cases, it is important to distinguish whether $X_L = x_1$ or $X_L = x_2$. In contrast, both $R_Y^1 = y_1$ and $R_Y^2 = y_1$ for the black terms. By interpreting $X_H = x_B$ as the disjunctive intervention $X_H = x_1 \vee x_2$, it becomes clear that the ambiguity largely has to do with which particular value is used as $X_H$. From the unit-level perspective, how should $x_B$ be interpreted for any particular individual datapoint?

One answer to making this decision is to look at the natural value of the intervened variable. In this case, if the intervention $X_H = x_B$ is applied, one can check if $X_L$ was originally going to be $x_1$ or $x_2$. In such cases, the blue terms would be included, while the red terms would be excluded. For example, in $\color{blue}p_{1010}$, $R_Y^1 = y_1$ and $R_Y^2 = y_0$. However, since $R_X = x_1$, we know that the natural value of $X_L = x_1$, so we would apply the value of $R_Y^1$ instead of $R_Y^2$, implying that $Y = y_1$. In contrast, in $\color{red}p_{2010}$, $R_Y^1$ and $R_Y^2$ are identical, but this time $R_X = x_2$, so we would apply the value of $R_Y^2$ instead of $R_Y^1$, implying that $Y = y_0$. Such a definition would look like
\begin{align}
&P(Y_{X_H = x_B} = y_1) \label{eq:soft-q-choice-natural} \\
&= p_{0011} + p_{0111} + p_{1011} + p_{1111} + p_{2011} + p_{2111} \nonumber \\
&+ {\color{blue}p_{1010}} + {\color{blue}p_{1110}} + {\color{blue}p_{2001}} + {\color{blue}p_{2101}} \nonumber \\
&+ \beta_1{\color{purple}p_{0010}} + \beta_2{\color{purple}p_{0110}} + \beta_3{\color{purple}p_{0001}} + \beta_4{\color{purple}p_{0101}}, \nonumber
\end{align}
for $\beta_1, \beta_2, \beta_3, \beta_4 \in [0, 1]$.

\begin{figure}
\centering
\includegraphics[width=0.5\linewidth]{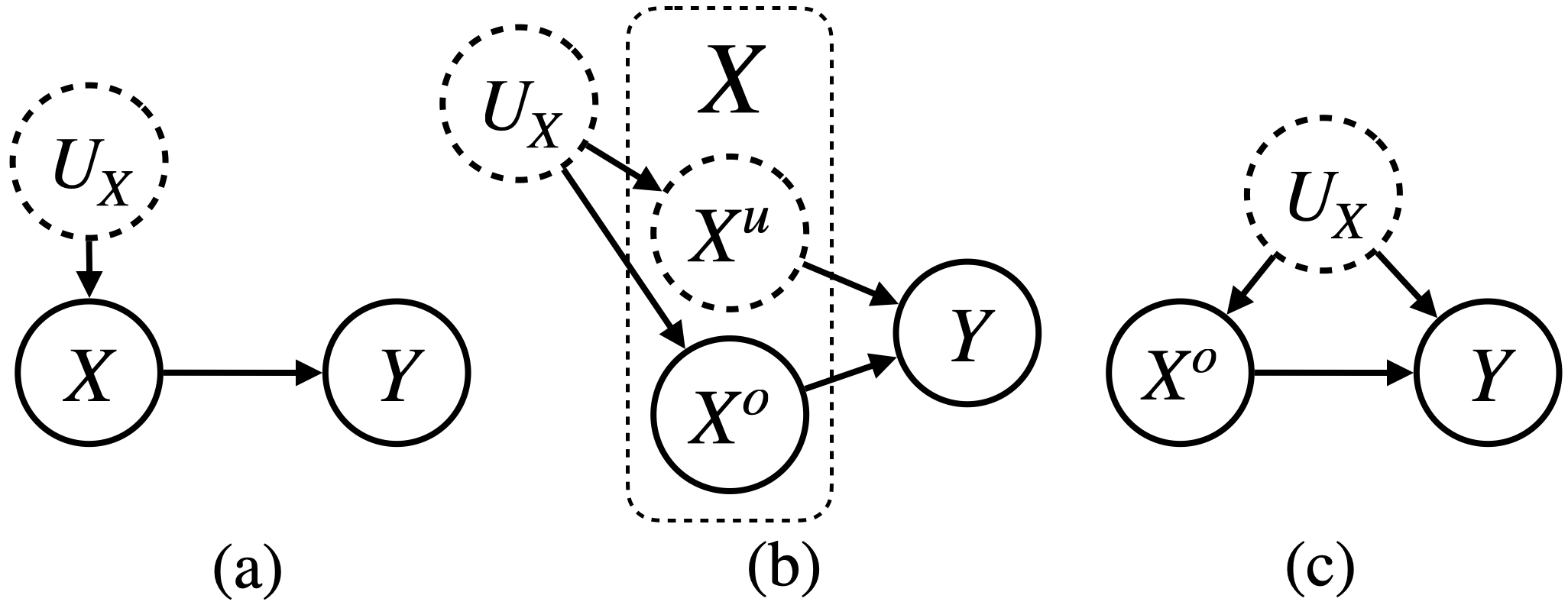}
\caption{The confounding issue of using the natural value in partial projections. (a) In the original system, there is no confounding between $X$ and $Y$. (b) When $X$ is partially projected into $X^o$ and $X^u$ , there is now a path from $U_X$ to $Y$ both through $X^o$ and $X^u$ which is used for calculating the natural value of $X$. (c) The path through $X^u$ results in confounding between the remaining $X^o$ and $Y$.}
\label{fig:partial-proj-confounding}
\end{figure}

This seems like an appealing method of deciding between $x_1$ and $x_2$, but it has many issues. For one, it is unclear what should happen with the purple terms, in which $R_Y^1 \neq R_Y^2$, but $R_X = x_0$. In other words, the natural value takes a value that does not map to the high-level value $X_H = x_B$, so it is unhelpful for deciding between whether $x_B$ implies $x_1$ or $x_2$. Another issue is that the mechanisms of deciding the natural value of $X_L$ are usually unobserved, so requiring this information adds a layer of unobserved confounding. Specifically, all variables will now be confounded with their parents since they will need the exogenous information used to generate the parents to find their natural value (see Fig.~\ref{fig:partial-proj-confounding}). This extra confounding adds difficulty in practical applications that require identifying the high-level quantities from minimal data.


In general, it is preferable to avoid the extra level of complication added by considering the natural value of variables given that the natural value is typically not observable in practice if an intervention was performed. For that reason, the interpretation of whether $X_H = x_B$ should be disambiguated as $x_1$ or $x_2$ should allow for either possibility without depending on the mechanism of $X_L$. This leads to the following formulation.

\begin{align}
&P(Y_{X_H = x_B} = y_1) \label{eq:soft-q-choice-sample-expanded} \\
&= \alpha{\color{purple}p_{0010}} + \alpha{\color{purple}p_{0110}} + (1-\alpha){\color{purple}p_{0001}} + (1-\alpha){\color{purple}p_{0101}} + p_{0011} + p_{0111} \nonumber \\
&+ \alpha{\color{blue}p_{1010}} + \alpha{\color{blue}p_{1110}} + (1-\alpha){\color{red}p_{1001}} + (1-\alpha){\color{red}p_{1101}} + p_{1011} + p_{1111} \nonumber \\
&+ \alpha{\color{red}p_{2010}} + \alpha{\color{red}p_{2110}} + (1-\alpha){\color{blue}p_{2001}} + (1-\alpha){\color{blue}p_{2101}} + p_{2011} + p_{2111}, \nonumber
\end{align}

where $\alpha \in [0, 1]$. In other words,

\begin{equation}
    \label{eq:soft-q-choice-sample}
    P(Y_{X_H = x_B} = y_1) = \alpha P(Y_{X_L = x_1} = y_1) + (1-\alpha) P(Y_{X_L = x_2} = y_1).
\end{equation}

This formulates the high-level intervention as a soft intervention over the low-level interventions, where $\alpha$ determines the probability of each possible low-level value. From the canonical model perspective, every possible value from Eq.~\ref{eq:soft-q-choice-max} is considered, but its weight is determined by $\alpha$. This resolves issues arising from using the natural value, since Eq.~\ref{eq:soft-q-choice-sample} can be computed for a fixed $\alpha$ as long as $P(Y_{X_L = x_1} = y_1)$ and $P(Y_{X_L = x_2} = y_1)$ can be computed. However, the question of choosing $\alpha$ remains. Indeed, an arbitrary value such as $\alpha = 0.5$ may not make the most sense. For example, if $P(X_L = x_1) >>> P(X_L = x_2)$, it may make more sense to pick a choice of $\alpha$ that is biased towards $x_1$. By this line of reasoning, the ideal choice of $\alpha$ should be

\begin{equation}
\label{eq:soft-q-alpha-choice}
    \alpha = P(X_L = x_1 \mid X_L \in \{x_1, x_2\}) = \frac{P(X_L = x_1, X_L \in \{x_1, x_2\})}{P(X_L \in \{x_1, x_2\})} = \frac{P(X_L = x_1)}{P(X_L = x_1) + P(X_L = x_2)}.
\end{equation}

More generally, this choice of $\alpha$ implies that a high-level intervention should be a soft intervention over the corresponding low-level interventions with probabilities based on their proportions. Formally, for a high-level intervention $\*X_H \gets \*x_H$, there is a corresponding soft intervention $\sigma_{\*X_L}$ that is a distribution over all low level interventions $\*X_L \gets \*x_L$ where $\tau(\*x_L) = \*x_H$. In general, $\*X_L$ must be a union of clusters for the abstraction mapping to be well defined, that is, $\*X_L = \bigcup_{\*C \in \bbC'} \*C$ for some $\bbC' \subseteq \bbC$. $\sigma_{\*X_L}$ must be decomposed into $\{\sigma_{\*C} : \*C \in \bbC'\}$, which must be sampled independently, otherwise having multiple interventions may introduce unintentional confounding. Following the above example, $\sigma_{\*C_i}$ for some $\*C_i \in \bbC$ (and corresponding $V_{H, i} \in \*V_H$) would be defined as

\begin{equation}
    \label{eq:low-intervention-agnostic}
    P(\sigma_{\*C_i} = \*c_i) = P(\*c_i \mid \tau(\*c_i) = v_{H, i}).
\end{equation}

Still, Eq.~\ref{eq:low-intervention-agnostic} may not be expressive enough for many applications. While this choice of $\alpha$ in Eq.~\ref{eq:soft-q-alpha-choice} works for the two-variable study shown here, it may fail to hold in general cases with more variables. Consider the following example.

\begin{example}
    \label{ex:soft-inter-connection-markov}
    Recall the setting discussed in Ex.~\ref{ex:noaic-issue}. For convenience, the setting is described again here. Different insurance companies ($Z$) offer various insurance plans ($X$), which affect whether an insurance claim is approved ($Y$). For simplicity, suppose there are two insurance companies ($z_1$ and $z_2$) which offer three different insurance plans ($x_1$, $x_2$, and $x_3$), and the claim is either approved ($Y = 1$) or not approved ($Y = 0$). Suppose the true model $\cM^* = \cM_L = \langle \*U_L, \*V_L, \cF_L, P(\*U_L)\rangle$ is described as follows.

    \begin{equation}
        \label{eq:ex-insurance-scm-query-form}
        \begin{split}
            \*U_L &= \{U_Z, U_X^{z_1}, U_X^{z_2}, U_Y^{x_1}, U_Y^{x_2}, U_Y^{x_3}\} \\
            \*V_L &= \{Z, X, Y\} \\
            \cF_L &=
            \begin{cases}
                f^L_Z(u_Z) &= u_Z \\
                f^L_X(z, u_X^{z_1}, u_X^{z_2}) &= u_X^{z} \\
                f^L_Y(x, u_Y^{x_1}, u_Y^{x_2}, u_Y^{x_3}) &= u_Y^x
            \end{cases} \\
            P(\*U_L) &= \begin{cases}
                P(U_Z = z_1) = 0.7 \\
                P(U_X^{z_1} = x_1) = 0.4, P(U_X^{z_1} = x_2) = 0.1, P(U_X^{z_1} = x_3) = 0.5 \\
                P(U_X^{z_2} = x_1) = 0.1, P(U_X^{z_2} = x_2) = 0.4, P(U_X^{z_2} = x_3) = 0.5 \\
                P(U_Y^{x_1} = 1) = 0.9, P(U_Y^{x_2} = 1) = 0.1, P(U_Y^{x_3} = 1) = 0.9
            \end{cases}
        \end{split}
    \end{equation}

    The interpretation of the model is as follows: Insurance plans $x_1$ and $x_3$ are very effective, with $0.9$ probability of claim acceptance, while $x_2$ is very ineffective at only $0.1$ probability. Insurance company $z_1$ is more reputable than $z_2$ and is more likely to offer plan $x_1$ over $x_2$, while company $z_2$ prefers to offer plan $x_2$ over $x_1$.

    Moreover, an important factor of consideration not shown in the model is that $x_1$ and $x_2$ are cheaper plans, while $x_3$ is more expensive. A data scientist who is studying this model may choose to abstract the different plans away, categorizing them simply as ``cheap'' and ``expensive'' plans. Formally, they would study a set of higher-level variables $\*V_H = \{Z_H, X_H, Y_H\}$, where $Z_H = Z$, $Y_H = Y$, and $X_H$ has a domain $\cD_{X_H} = \{x_C, x_E\}$ corresponding to cheap and expensive plans respectively. There exists an abstraction function $\tau: \cD_{\*V_L} \rightarrow \cD_{\*V_H}$ such that $\tau$ maps $x_1$ and $x_2$ to $x_C$ and maps $x_3$ to $x_E$. We will use the notation $Z$ and $Y$ instead of $Z_H$ and $Y_H$ since the variables are the same. One question that the data scientist may have is ``What is the causal effect of choosing a cheap plan on claim acceptance rate?'', denoted as $P(Y_{X_H = x_C} = 1)$.

    First, we note that $\tau$ is a constructive abstraction function with clusters $\bbC$ and $\bbD$, where $\bbC$ and $\bbD$ trivially leaves the original variables and values in their own clusters, except $x_1$ and $x_2$ are clustered together. Under this choice of $\tau$, observe that the AIC does not hold, notably that
    \begin{equation}
        0.9 = P(Y_{X = x_1} = 1) \neq P(Y_{X = x_2} = 1) = 0.1.
    \end{equation}

    Then, how could one compute $P(Y_{X_H = x_C} = 1)$ given that both $x_1$ and $x_2$ map to $x_C$? The answer is that the intervention $X_H = x_C$ should correspond to a soft intervention on $X$, denoted as $\sigma_X$ and assigning a different probability to $x_1$ and $x_2$ (but not to $x_3$, since $\tau(x_3)$ does not map to $x_C$). If $P(\sigma_X = x_1) = \alpha$ and $P(\sigma_X = x_2) = 1 - \alpha$, then it is clear that
    \begin{equation}
        P(Y_{X_H = x_C} = 1) = \alpha P(Y_{X = x_1} = 1) + (1 - \alpha) P(Y_{X = x_2} = 1),
    \end{equation}
    for some choice of $\alpha \in [0, 1]$.

    Still, this leaves the question of how to choose $\alpha$. As a first attempt, it may be appealing to choose
    \begin{equation}
        \label{eq:ex-soft-choice1}
        \begin{split}
            \alpha &= P(X = x_1 \mid X = x_1 \vee X = x_2) = \frac{P(X = x_1)}{P(X = x_1) + P(X = x_2)} = 0.5 \\
            (1 - \alpha) &= P(X = x_2 \mid X = x_1 \vee X = x_2) = \frac{P(X = x_2)}{P(X = x_1) + P(X = x_2)} = 0.5,
        \end{split}
    \end{equation}
    implying that $P(Y_{X_H = x_C} = 1) = 0.5$. Indeed, this choice has some appealing properties, notably that
    \begin{equation}
        \label{eq:ex-soft-choice1-markovian}
        P(Y_{X_H = x_C} = 1) = P(Y = 1 \mid X_H = x_C) = P(Y = 1 \mid X = x_1 \vee X = x_2).
    \end{equation}
    
    One immediate observation that arises from this choice of $\sigma_X$ is that the insurance company, $Z$, is not taken into account. Indeed, consider another query $P(Y_{X_H = x_C} = 1 \mid Z = z_1)$, which answers the question ``What is the causal effect of choosing a cheap plan on claim acceptance rate given that the plan was provided by company $z_1$?'' We would expect that, while $P(Y_{X_H = x_C} = 1) = 0.5$, it is very obvious that conditioning on $Z = z_1$ should change the result given that company $z_1$ is much more likely to recommend plan $x_1$ over $x_2$. However, with the choice of $\sigma_X$ from Eq.~\ref{eq:ex-soft-choice1}, we would evaluate $P(Y_{X_H = x_C} = 1 \mid Z = z_1)$ to be equal to $P(Y_{X_H = x_C} = 1)$, since neither $\sigma_X$ nor $f_Y$ takes $Z$ into account. To resolve this issue, it seems that a better choice of $\alpha$ may be deduced as follows
    \begin{equation}
        \label{eq:ex-soft-choice2}
        \begin{split}
            \alpha &= P(X = x_1 \mid X = x_1 \vee X = x_2, Z) = \frac{P(x_1 \mid Z)}{P(x_1 \mid Z) + P(x_2 \mid Z)} \\
            (1 - \alpha) &= P(X = x_2 \mid X = x_1 \vee X = x_2, Z) = \frac{P(x_2 \mid Z)}{P(x_1 \mid Z) + P(x_2 \mid Z)},
        \end{split}
    \end{equation}
    which evaluates as $\alpha = 0.8$ when $Z = z_1$ and $\alpha = 0.2$ when $Z = z_2$. This translates to $P(Y_{X_H = x_C} = 1 \mid Z = z_1) = 0.74$, whereas $P(Y_{X_H = x_C} = 1 \mid Z = z_2) = 0.26$.
    \hfill $\blacksquare$
\end{example}

As illustrated in the example, the soft intervention applied to the low-level should not be agnostic of the parents of the intervened variables. This is shown visually in Fig.~\ref{fig:full-vs-partial-proj}. Given that important information about $X$ is lost through the abstraction, information from $Z$ may be required in downstream functions to supplement the lost information. This brings us to the more general definition:

\begin{equation}
    \label{eq:low-intervention-markov}
    P(\sigma_{\*C_i} = \*c_i) = P(\*c_i \mid \tau(\*c_i) = v_{H, i}, \pai{V_{H, i}}).
\end{equation}

Note that this definition of $\sigma_{\*X_L}$ does not take into account any exogenous variables of $\*U_L$. This is sufficient in cases where the high-level model $\cM_H$ is Markovian (i.e., no unobserved confounders), but this is not a reasonable assumption in most settings, as even Markovian low-level models can translate to non-Markovian high-level models after the abstraction function is applied (see row (c) of Fig.~\ref{fig:proj-cdag-examples} for an example). The next section discusses why this definition is insufficient for non-Markovian cases and presents ideas on possible generalizations.

\subsection{Non-Markovian Considerations}
\label{app:discussion-nonmarkov}

Eq.~\ref{eq:low-intervention-markov} is reasonable for Markovian cases, where there is no unobserved confounding (i.e., all $U_H \in \*U_H$ are independent and are only parents for at most one $V_H \in \*V_H$). The interpretation of Eq.~\ref{eq:low-intervention-markov} is that the disambiguation of low-level interventions for any given high-level intervention should depend on the endogenous parents of the intervened variables, but the exogenous parents should be ignored and resampled. Having a dependence on the exogenous variables would result in identifiability issues (e.g., from newly generated confounding as visualized in Fig.~\ref{fig:partial-proj-confounding}). However, if the exogenous parents are causing unobserved confounding, it does not make sense to simply ignore them. Consider the following example.

\begin{example}
    Consider the same setting as Ex.~\ref{ex:soft-inter-connection-markov}, except in the data collection process, data is collected on hospitals instead of insurance companies. That is, for each person, $Z$ is recorded as their registered hospital instead of their insurance company, which is now left unobserved. For the sake of simplicity, $Z$ will stay as a binary variable, representing two possible hospitals $z_1$ and $z_2$. It turns out that people will choose their hospital based on which hospitals are covered by their insurance company, which now serves as an unobserved confounder between hospital choice and insurance plan. The full SCM $\cM^* = \cM_L = \langle \*U_L, \*V_L, \cF_L, P(\*U_L)\rangle$ is described as follows.

    \begin{equation}
        \label{eq:ex-insurance-scm-query-form-hospital}
        \begin{split}
            \*U_L &= \{U_Z, U_X^{z_1}, U_X^{z_2}, U_Y^{x_1}, U_Y^{x_2}, U_Y^{x_3}\} \\
            \*V_L &= \{Z, X, Y\} \\
            \cF_L &=
            \begin{cases}
                f^L_Z(u_Z) &= u_Z \\
                f^L_X(u_Z, u_X^{z_1}, u_X^{z_2}) &= u_X^{u_Z} \\
                f^L_Y(x, u_Y^{x_1}, u_Y^{x_2}, u_Y^{x_3}) &= u_Y^x
            \end{cases} \\
            P(\*U_L) &= \begin{cases}
                P(U_Z = z_1) = 0.7 \\
                P(U_X^{z_1} = x_1) = 0.4, P(U_X^{z_1} = x_2) = 0.1, P(U_X^{z_1} = x_3) = 0.5 \\
                P(U_X^{z_2} = x_1) = 0.1, P(U_X^{z_2} = x_2) = 0.4, P(U_X^{z_2} = x_3) = 0.5 \\
                P(U_Y^{x_1} = 1) = 0.9, P(U_Y^{x_2} = 1) = 0.1, P(U_Y^{x_3} = 1) = 0.9
            \end{cases}
        \end{split}
    \end{equation}

    Note that the only different between this SCM and the one from Eq.~\ref{eq:ex-insurance-scm-query-form} is that instead of $Z$, $f_X$ now takes $U_Z$ as input. However, the behavior of the two SCMs are identical on the observational level, and moreover, if $Z$ is projected away, the rest of the SCM is completely the same. Therefore, $P(Y_{X_H = x_C} = 1 \mid z)$ should be the same as the result computed in Eq.~\ref{eq:ex-soft-choice2}. However, this is obviously not the case when applying Eq.~\ref{eq:low-intervention-markov}, since $Z$ is no longer a parent of $X$.

    Indeed, the computation of $P(Y_{X_H = x_C} = 1 \mid z)$ according to Eq.~\ref{eq:low-intervention-markov} can now be shown as follows.
    \begin{align}
        &P(Y_{X_H = x_C} = 1 \mid z) \\
        &= P(Y_{\sigma_{X_L}(x_C) = 1} \mid z) \\
        &= P(X_L = x_1 \mid \tau(X_L) = x_C)P(Y_{X_L = x_1} \mid z) + P(X_L = x_2 \mid \tau(X_L) = x_C)P(Y_{X_L = x_2} \mid z) \label{eq:soft-map-step} \\
        &= P(X_L = x_1 \mid \tau(X_L) = x_C)P(Y_{X_L = x_1}) + P(X_L = x_2 \mid \tau(X_L) = x_C)P(Y_{X_L = x_2}) \label{eq:soft-r1-step} \\
        &= 0.5.
    \end{align}
    Notably, line \ref{eq:soft-map-step} applies Eq.~\ref{eq:low-intervention-markov}, which no longer includes $z$ in the probability of choosing the low-level intervention on $X_L$, and line \ref{eq:soft-r1-step} follows since $Z$ and $Y$ are independent when intervening on $X_L$.
    \hfill $\blacksquare$
\end{example}

The discrepancy in the above example follows from the issue that Eq.~\ref{eq:low-intervention-markov} makes a distinction between whether a variable's parent is endogenous or exogenous. In this particular example, the issue could be solved by modifying Eq.~\ref{eq:low-intervention-markov} to include $U_Z$ instead of $Z$. However, it is unclear why $U_Z$ should be included but not $U_X^{z_1}$ or $U_X^{z_2}$. Even in this example, SCM $\cM_L$ could be designed in a way that behaves identically, but the exogenous space is chosen differently. For example $U_Z$, $U_X^{z_1}$, and $U_X^{z_2}$ could be subsumed into a Gaussian distribution, and their behavior can be mimicked using the inverse integral transform. In such a case, one could not pick and choose individual variables from $\*U_L$ to include in the low-level soft intervention.

The key insight for solving this problem in the non-Markovian setting is to find a way to disentangle the confounded parts of the exogenous variables from the parts that are only influencing individual variables. For example, perhaps $\*U_X$ could be split into $\*U_X^c$ and $\*U_X^u$, where $\*U_X^c$ are all the exogenous variables that affect $X$ and also some other variable, while $\*U_X^u$ only affects $X$. Moreover, $\*U_X^u$ needs to be chosen in a way that is ``maximal'', so as to not allow arbitrary flexibility between whether a variable belongs in $\*U_X^u$ or $\*U_X^c$.

Once again, to solve this problem, one can leverage the principles of canonical models \citep{balke:pea97, zhang:bareinboim21b}. For any high-level variable $X_H$, define $R_{X_H}$ as a random variable, where $\cD_{R_{X_H}}$ consists of all possible functions of $f_{X_H}$ w.r.t.~$\Pai{X_H}$. Note that for a fixed choice of $\Ui{X_H}$, $f_{X_H}$ is a deterministic function w.r.t.~$\Pai{X_H}$. Hence,
\begin{equation}
    P(R_{X_H} = r_{X_H}) = \sum_{\*u \in \cD_{\Ui{X_H}} : f_{X_H}(\cdot, \*u) = r_{X_H}} P(\*u).
\end{equation}

For any high-level variable $X_H$, denote $\*V_H^c(X_H) \subseteq \*V_H \setminus \{X_H\}$ as the set of variables of $\*V_H$ that share an confounding exogenous variable with $X_H$. Finally, denote $R_{X_H}(\*u')$ as the random variable $R_{X_H}$ over the distribution $P(R_{X_H} = r_{X_H} \mid \*U' = \*u')$ for some $\*U' \subseteq \*U$.

Now redefine $\sigma_{\*C_i}$ as
\begin{equation}
    \label{eq:low-intervention-general}
    P(\sigma_{\*C_i} = \*c_i) = P(\*c_i \mid \tau(\*c_i) = v_{H, i}, \pai{V_{H, i}}, \*R_{\*V_H^c(V_{H, i})}(\ui{\*C_i})).
\end{equation}

This now matches Eq.~\ref{eq:low-intervention-short} in Sec.~\ref{sec:soft-abs}, with $\ui{V_{H, i}}^c$ being used as a shorthand for $\*R_{\*V_H^c(V_{H, i})}(\ui{\*C_i})$. Intuitively, the soft intervention over $\*C_i$ now also depends on $\Ui{\*C_i}$ but only in the way that it affects the functions of the confounded neighbors of $V_{H, i}$. Notably $\*V_H^c(V_{H, i})$ does not contain $V_{H, i}$ itself, so $\Ui{\*C_i}$ is still free to vary in ways that affect $f_{V_{H, i}}$ but not any other function.

Two important points must be clarified to avoid ambiguity when considering queries that contain interventions over multiple variables. First, $\sigma_{\*C_i}(V_{H, i}, \pai{V_{H, i}}, \ui{\*C_i})$ is applied at most once for each value of $v_{H, i}$ in $\tau(Q)$, so if there are multiple terms $\*Y_{H, i[\*x_{H, i}]}$ that share the same intervention (e.g., $V_{H, i} = v_{H, i}$ in both $\*x_{H, 1}$ and $\*x_{H, 2}$), then $\sigma_{\*C_i}$ is only sampled once and is used for both terms. However, if $V_{H, i} = v_{H, i}$ in $\*x_{H, 1}$ but $V_{H, i} = v'_{H, i}$ in $\*x_{H, 2}$, then it is sampled separately even though $V_{H, i}$ is in both terms. Second, if two high-level variables in the same intervention are confounded, interventions on both variables are performed according to $\sigma_{\*C_i}$ with $\*V_H^c(V_{H, i})$ remaining the same, ignoring the fact that the confounded neighbor is being intervened. These conditions are set to allow low-level queries to match corresponding high-level queries without generating semantic differences between identical high-level queries that are written in different forms (e.g., $P(Y_{x}, Z_x) = P(\{Y, Z\}_x)$).


\subsection{Projected Sampling}
\label{app:projected-sampling}

In the original implementation of the representational NCM (Def.~\ref{def:rncm}), the representation is learned using an autoencoder structure. That is, a neural network is used to implement $\widehat{\tau}$ which maps the original data $\*V_L$ to its high level representation $\*V_H$, and another neural network $\widehat{\tau}^{-1}$ is trained to invert $\*V_H$ back to $\*V_L$. To avoid AIC violations, the dimensionality of $\*V_H$ must be sufficiently large to allow $\widehat{\tau}^{-1}$ to accurately reconstruct $\*V_L$ with no loss of information. This poses a dilemma in practice of choosing the proper representation size, since a small representation will lose information and violate the AIC while a large representation is too difficult to learn in an NCM and poses little benefit compared to directly outputting the image.

The results of Sec.~\ref{sec:soft-abs} that generalize abstraction theory to cases with AIC-violations provides insight on how to work around this issue. Suppose a variable $X_H \in \*V_H^{\dagger}$ is an AIC violator, and the goal is to perform a high-level intervention on $X_H$. The corresponding low-level counterpart of the intervention $do(X_H = x_H)$ can be computed with $\sigma_{\*X_L}$ for $\tau(\*X_L) = X_H$, from Eq.~\ref{eq:low-intervention-short}. We can leverage this equation for the different purpose of sampling a value from $\*X_L$. If $X_H$ violates the AIC, then it is possible that knowing that $X_H = x_H$ does not provide enough information to pinpoint a precise corresponding value of $\*X_L$ (i.e., there could be many choices of $\*x_L \in \cD_{\*X_L}$ such that $\tau(\*x_L) = x_H$. Hence, instead of directly taking $\widehat{\tau}^{-1}(x_H)$ as the corresponding value of $\*x_L$, which is calculated through a deterministic function, one can sample from $\sigma_{\*X_L}$ to obtain a faithful value of $\*x_L$ that corresponds to $x_H$. We call this procedure \emph{projected sampling}.

To explain why projected sampling can provide quality samples that do not depend on the dimensionality of the representation, consider a concrete example of generating image $\*X_L$ from its representation $\*X_H$. Suppose 100 bits of information are required to fully reconstruct $\*X_L$, but $\*X_H$ only contains 10 bits. Using the autoencoding strategy of the original RNCM, it is clear that the reconstruction of $\*X_L$ based on $\*X_H$ will be imperfect since we have lost 90 bits of information. Moreover, it may be expected that as we add bits to $\*X_H$, the quality of the reconstruction will gradually improve as we approach 100 bits. In contrast, using the projected sampling approach, we take the 10 bits of information of $\*X_L$ and then use them to sample the other 90 bits. This forms a complete 100 bits that can be used to construct a full sample of $\*X_L$. This procedure can be used with any representation size of $\*X_H$ without sacrificing in the quality of the reconstruction. This comparison is shown visually in the experiment depicted in Fig.~\ref{fig:scaling-repr-results}.

%% file: section/C_examples.tex
\section{Additional Examples}
\label{app:examples}

In this section, we provide additional examples to illustrate the key points of the paper.

The main limitation that this paper aims to address is the requirement of the abstract invariance condition (AIC) in Def.~\ref{def:invariance-condition}. A commonly cited example of this issue is about the abstraction of the two types of cholesterol, HDL and LDL, as shown below.

\begin{example}
    \label{ex:noaic-heart-disease}
    Consider a study on the effects of diet on heart disease. Having an unhealthy diet ($X$) can raise the risk of heart disease ($Y$) depending on its cholesterol content. Cholesterol comes in two forms, called high-density and low-density lipoproteins (HDL and LDL, respectively), where HDL is believed to lower heart disease risk while LDL increases it \citep{steinberg2007, truswell2010}. Suppose the study is simplified to binary variables, and the true model $\cM_L$ is:
    \begin{align}
        \*U_L &= \{U_X, U_{C1}, U_{C2}, U_Y\} \\
        \*V_L &= \{X, HDL, LDL, Y\} \\
        \cF_L &=
        \begin{cases}
            X \gets f^L_X(u_X) = u_X \\
            HDL \gets f^L_{HDL}(x, u_{C1}) = x \oplus u_{C1} \\
            LDL \gets f^L_{LDL}(x, u_{C2}) = x \oplus u_{C2} \\
            Y \gets f^L_Y(hdl, ldl, u_Y) = (ldl \wedge \neg hdl) \oplus u_Y 
        \end{cases} \label{eq:ex-cholesterol-F} \\
        P(\*U_L) &=
        \begin{cases}
            P(U_X = 1) = 0.5 \\
            P(U_{C1} = 1) = 0.1 \\
            P(U_{C2} = 1) = 0.1 \\
            P(U_Y = 1) = 0.1
        \end{cases}
    \end{align}
    It can be computed from $\cM_L$ that a person is more likely to get heart disease if their diet consists of higher LDL levels and lower HDL levels, notably
    \begin{align}
        P^{\cM_L}(Y_{HDL = 0, LDL = 1} = 1) &= 0.9, \label{eq:nonaic_q1} \\
        P^{\cM_L}(Y_{HDL = 1, LDL = 0} = 1) &= 0.1. \label{eq:nonaic_q2}
    \end{align}
    
    Now, suppose a data scientist decides to abstract HDL and LDL together into a variable called ``total cholesterol'' (TC), defined as
    \begin{equation}
        TC = HDL + LDL.
    \end{equation}

    This naturally leads to the choice of intervariable clusters
    \begin{equation}
        \bbC = \{\*C_1 = \{X\}, \*C_2 = \{HDL, LDL\}, \*C_3 = \{Y\}\},
    \end{equation}
    and intravariable clusters
    \begin{equation}
    \bbD_{\*C_2} = 
    \begin{cases}
        tc_0 &= \{(HDL = 0, LDL = 0)\} \\
        tc_1 &= \{(HDL = 0, LDL = 1), \\
        & (HDL = 1, LDL = 0)\} \\
        tc_2 &= \{(HDL = 1, LDL = 1)\}.
    \end{cases}
    \end{equation}
    For the other clusters, the variables remain the same. Let $\tau$ be the constructive abstraction function defined with this choice of $\bbC$ and $\bbD$ (i.e.~$\tau_{\*C_2}(hdl, ldl) = hdl + ldl$).

    A violation of the AIC arises due to the grouping of values $(HDL = 0, LDL = 1)$ and $(HDL = 1, LDL = 0)$ into the same intravariable cluster. To witness, note that $\tau_{\*C_1}(HDL = 0, LDL = 1) = \tau_{\*C_2}(HDL = 1, LDL = 0) = (TC = 1)$. Consider two queries $Q_1 = P(Y_{HDL = 0, LDL = 1} = 1)$ and $Q_2 = P(Y_{HDL = 1, LDL = 0} = 1)$, and recall from Eqs.~\ref{eq:nonaic_q1} and \ref{eq:nonaic_q2} that $Q_1^{\cM_L} = 0.9$ and $Q_2^{\cM_L} = 0.1$. However, since $\tau_{\*C_1}(HDL = 0, LDL = 1) = \tau_{\*C_2}(HDL = 1, LDL = 0) = (TC = 1)$, both $Q_1$ and $Q_2$ have the same high-level counterpart (i.e., $\tau(Q_1) = \tau(Q_2) = P(Y_{TC = 1} = 1)$). No choice of $\cM_H$ over $\*V_H$ can be both $Q_1$-$\tau$ consistent and $Q_2$-$\tau$ consistent with $\cM_L$ because $P^{\cM_H}(Y_{TC = 1} = 1)$ cannot both be equal to $0.9$ and $0.1$.

    This holds true fundamentally on the SCM-level as well. Note that a $\tau$-abstraction with this choice of $\tau$ cannot exist for $\cM_L$ for similar reasons. Specifically, note that
    \begin{align}
        Y_{L[HDL = 0, LDL = 1]}(U_Y = 0) = 1, \label{eq:nonaic_y1} \\
        Y_{L[HDL = 1, LDL = 0]}(U_Y = 0) = 0, \label{eq:nonaic_y2}
    \end{align}
    but $Y_{H[TC = 1]}(\tau_{\mathbf{U}}(U_Y = 0))$ cannot both be equal to $0$ and $1$. This violates Eq.~\ref{eq:tau-u-compatibility}, implying that no such $\tau$-abstraction can exist.
    \hfill $\blacksquare$
\end{example}

Below, we give an example of an SCM projection followed by a partial SCM projection for comparison.

\begin{example}
    \label{ex:scm-projection}
    For concreteness, consider a setting in which different insurance companies ($Z$) offer various insurance plans ($X$), which affect whether an insurance claim is approved ($Y$). For simplicity, suppose there are two insurance companies ($z_1$ and $z_2$) which offer three different insurance plans ($x_1$, $x_2$, and $x_3$), and the claim is either approved ($Y = 1$) or not approved ($Y = 0$). Suppose the true model $\cM^* = \cM_L$ is described as follows.

    \begin{equation}
        \label{eq:ex-insurance-scm-2}
        \cM_L = \begin{cases}
            \*U_L &= \{U_Z, U_X^1, U_X^2, U_Y^1, U_Y^2, U_Y^3\} \\
            \*V_L &= \{Z, X, Y\} \\
            \cF_L &=
            \begin{cases}
                f^L_Z(u_Z) &= u_Z \\
                f^L_X(z, u_X^1, u_X^2) &= \begin{cases}
                    u_X^1 & z = z_1 \\
                    u_X^2 & z = z_2
                \end{cases} \\
                f^L_Y(x, u_Y^1, u_Y^2, u_Y^3) &= \begin{cases}
                    u_Y^1 & x = x_1 \\
                    u_Y^2 & x = x_2 \\
                    u_Y^3 & x = x_3
                \end{cases} \\
            \end{cases} \\
            P(\*U_L) &= \begin{cases}
                P(U_Z = z_1) = P(U_Z = z_2) = 0.5 \\
                P(U_X^1 = x_1) = 0.4, P(U_X^1 = x_2) = 0.1, P(U_X^1 = x_3) = 0.5 \\
                P(U_X^2 = x_1) = 0.1, P(U_X^2 = x_2) = 0.4, P(U_X^2 = x_3) = 0.5 \\
                P(U_Y^1 = 1) = 0.9, P(U_Y^2 = 1) = 0.1, P(U_Y^3 = 1) = 0.9
            \end{cases}
        \end{cases}
    \end{equation}

    The interpretation of the model is as follows: Insurance plans $x_1$ and $x_3$ are very effective, with $0.9$ probability of claim acceptance, while $x_2$ is very ineffective at only $0.1$ probability. Insurance company $z_1$ is more reputable than $z_2$ and is more likely to offer plan $x_1$ over $x_2$, while company $z_2$ prefers to offer plan $x_2$ over $x_1$.

    A data scientist may be interested in studying which insurance company ($z_1$ or $z_2$) is the better company for getting claims approved. In this case, the specific plan $X$ that is being used may not be relevant. One may wish to instead study only the set of variables $\{Z, Y\}$, excluding $X$ from the set. In other words, the SCM of interest is the \emph{SCM projection} of $\cM_L$ to the variable set $\*V_H = \{Z, Y\}$. The SCM projection of $\cM_L$ over $\*V_H$ is quite straightforward to specify.

    \begin{equation}
        \label{eq:ex-insurance-projection}
        \cM_H = \begin{cases}
            \*U_H &= \*U_L \\
            \*V_H &= \{Z, Y\} \\
            \cF_H &=
            \begin{cases}
                f^H_Z(u_Z) &= u_Z \\
                f^H_Y(z, u_X^1, u_X^2, u_Y^1, u_Y^2, u_Y^3) &= \begin{cases}
                    u_Y^1 & f^L_X(z, u_X^1, u_X^2) = x_1 \\
                    u_Y^2 & f^L_X(z, u_X^1, u_X^2) = x_2 \\
                    u_Y^3 & f^L_X(z, u_X^1, u_X^2) = x_3
                \end{cases} \\
            \end{cases} \\
            P(\*U_L) &= P(\*U_H)
        \end{cases}
    \end{equation}

    With $X$ excluded from the model, the functionality of $X$ is projected into the function of its child, $Y$. Hence, the natural construction of the SCM projection $\cM_H$ is simply the same as the construction of $\cM_L$, but with $f_Y$ computing $X$ internally using $f_X$ (comparing Eq.~\ref{eq:ex-insurance-scm} with Eq.~\ref{eq:ex-insurance-projection}). It is not difficult to verify that computations of values of $Z$ and $Y$ under any choice of $\*U_L$ remains the same in both models. Consequently, the induced PCH distributions are also the same, and $\cM_H$ can be viewed simply as $\cM_L$ but ignoring $X$.
    
    \hfill $\blacksquare$
\end{example}

\begin{example}
    \label{ex:scm-partial-projection}
    Continuing the insurance example in Ex.~\ref{ex:scm-projection}, suppose an important factor of consideration not shown in the model is that $x_1$ and $x_2$ are cheaper insurance plans, while $x_3$ is more expensive. A data scientist who is studying this model may choose to abstract the different plans away, categorizing them simply as ``cheap'' and ``expensive'' plans. Formally, they would study a set of higher-level variables $\*V_H = \{Z_H, X_H, Y_H\}$, where $Z_H = Z$, $Y_H = Y$, and $X_H$ has a domain $\cD_{X_H} = \{x_C, x_E\}$ corresponding to cheap and expensive plans respectively. There exists an abstraction function $\tau: \cD_{\*V_L} \rightarrow \cD_{\*V_H}$ such that $\tau$ maps $x_1$ and $x_2$ to $x_C$ and maps $x_3$ to $x_E$. We will use the notation $Z$ and $Y$ instead of $Z_H$ and $Y_H$ since the variables are the same. Note that in the new abstraction model $\cM_H$, $X$ is not removed entirely, but it is reduced down to only two possible values instead of three.

    One possible method of accounting for this is as follows. First, redefine $X$ into two parts, $X^o$ and $X^u$, where $X^o$ represents the observed portion of $X$ and $X^u$ represents the unobserved portion. $X^o$ can simply be defined as $\tau(X)$. However, when $X^o = x_C$, it is ambiguous whether $X = x_1$ or $x_2$. Define $X^u$ as a binary variable, where, whenever $X^o = x_C$, $X^u = 0$ represents $X = x_1$ while $X^u = 1$ represents $X = x_2$. $X^u$ can be thought of as an indicator variable disambiguating any loss of information of $X^o$. Putting everything together, one can construct $\cM_H$ as follows.

    \begin{equation}
        \label{eq:ex-insurance-partial-projection}
        \cM_H = \begin{cases}
            \*U_H &= \*U_L \cup \{X^u\} \\
            \*V_H &= \{Z, X_H, Y\} \\
            \cF_H &=
            \begin{cases}
                f^H_Z(u_Z) &= u_Z \\
                f^H_{X}(z, u_X^1, u_X^2) &= \tau(f^L_X(z, u_X^1, u_X^2)) \\
                f^H_Y(z, x^o, x^u, u_Y^1, u_Y^2, u_Y^3) &= \begin{cases}
                    u_Y^1 & x^o = x_C, x^u = 0 \\
                    u_Y^2 & x^o = x_C, x^u = 1 \\
                    u_Y^3 & x^o = x_E
                \end{cases} \\
            \end{cases} \\
            P(\*U_H) &= P(U_L)P(X^u \mid \*U_L) \\
            & P(X^u = 0 \mid \*U_L) = P(X = x_1 \mid X \in \{x_1, x_2\}) \\
            & P(X^u = 1 \mid \*U_L) = P(X = x_2 \mid X \in \{x_1, x_2\})
        \end{cases}
    \end{equation}

    Note that in this model, $f^H_Y$ is trying to retain the same functionality as $f^L_Y$, but it is only given $X_H$ as input instead of $X$. To disambiguate between $X = x_1$ and $X = x_2$, which both map to $X_H = x_C$, it utilizes the new exogenous variable $X^u$, whose probability is based on the probability of whether $X$ is $x_1$ or $x_2$. In doing so, $f^H_Y$ can mimic the functionality of $f^L_Y$ in the sense that the lost information for $X$ is partially projected into the exogenous space.
    \hfill $\blacksquare$
\end{example}




%% file: section/D_experimental_details.tex
\section{Experimental Details}
\label{app:experimental-details}

In this section, we add further details to the experiments.

\subsection{Neural Causal Models}

Most of the experiments in this paper leverage the $\cG$-constrained neural causal model for practical implementations, defined below.

\begin{definition}[$\cG$-Constrained Neural Causal Model ($\cG$-NCM) {\citep[Def.~7]{xia:etal21}}]
    \label{def:gncm}
    Given a causal diagram $\cG$, a $\cG$-constrained Neural Causal Model (for short, $\cG$-NCM) $\widehat{M}(\bm{\theta})$ over variables $\*V$ with parameters $\bm{\theta} = \{\theta_{V_i} : V_i \in \*V\}$ is an SCM $\langle \widehat{\*U}, \*V, \widehat{\cF}, P(\widehat{\*U}) \rangle$ such that
    \begin{itemize}
        \item $\widehat{\*U} = \{\widehat{U}_{\*C} : \*C \in \bbC(\cG)\}$, where $\bbC(\cG)$ is the set of all maximal cliques over bidirected edges of $\cG$;
        
        \item $\widehat{\cF} = \{\hat{f}_{V_i} : V_i \in \*V\}$, where each $\hat{f}_{V_i}$ is a feedforward neural network parameterized by $\theta_{V_i} \in \bm{\theta}$ mapping values of $\Ui{V_i} \cup \Pai{V_i}$ to values of $V_i$ for $\Ui{V_i} = \{\widehat{U}_{\*C} : \widehat{U}_{\*C} \in \widehat{\*U} \text{ s.t. } V_i \in \*C\}$ and $\Pai{V_i} = \Parents_{\cG}(V_i)$;
        
        \item $P(\widehat{\*U})$ is defined s.t.\ $\widehat{U} \sim \unif(0, 1)$ for each $\widehat{U} \in \widehat{\*U}$.
        \hfill $\blacksquare$
    \end{itemize}
\end{definition}

The representational form of the NCM (RNCM) is incorporated for the sake of leveraging abstractions as representation learning tools, defined below.

\begin{definition}[Representational NCM (RNCM) {\citep[Def.~11]{xia:bareinboim24}}]
    \label{def:rncm}
    A representational NCM (RNCM) is a tuple $\langle \widehat{\tau}, \widehat{M} \rangle$, where $\widehat{\tau}(\*v_L; \bm{\theta}_{\tau})$ is a function parameterized by $\bm{\theta}_{\tau}$ mapping from $\*V_L$ to $\*V_H$, and $\widehat{M}$ is an NCM defined over $\*V_H$. A $\cG_{\bbC}$-constrained RNCM ($\cG_{\bbC}$-RNCM) is an RNCM $\langle \widehat{\tau}, \widehat{M} \rangle$ such that $\widehat{\tau}$ is composed of subfunctions $\widehat{\tau}_{\*C_i}$ for each $\*C_i \in \bbC$ (each with its own parameters $\bm{\theta}_{\tau_{\*C_i}}$), and $\widehat{M}$ is a $\cG_{\bbC}$-NCM (Def.~\ref{def:gncm}).
    \hfill $\blacksquare$
\end{definition}

\subsection{Projected C-DAG Experiment}

The first experiment tests the necessity of the projected C-DAGs in an estimation task where the AIC does not hold. The setting is described by three variables $\*V_L = \{Z, X, Y\}$, and the low level model is described as
\begin{itemize}
    \item $Z$ is a 10-dimensional one-hot encoding ($\cD_Z = \{0, 1\}^{10}$) of a digit from 0-9, and it samples one uniformly at random.
    \item $X$ is an MNIST image ($\bbR^{3 \times 32 \times 32}$) of the digit of $Z$. It is colored either red or blue and is shaded either light or dark. If the digit is odd, there is a 0.9 probability that the color will be red and 0.1 that it will be blue. The odds are flipped if the digit is even. Blue digits have a 0.7 probability of being light and 0.3 of being dark, and the odds are flipped for red digits.
    \item $Y$ is a label  ($\cD_Y = \{0, 1\}$) that predicts whether $X$ is red ($Y = 1$) or blue ($Y = 0$), but it is incorrect with 0.1 probability.
\end{itemize}

On the high level, $Z$ and $Y$ remain the same, but $\tau(X) = X_H$, where $X_H$ is a binary variable ($\cD_{X_H} = \{0, 1\}$) that represents whether $X$ is light or dark.

The corresponding causal diagram $\cG$ is shown in the l.h.s.\ of Fig.~\ref{fig:proj-cdag-examples}(a), which is also the C-DAG $\cG_{\bbC}$. The r.h.s.\ shows the projected C-DAG $\cG_{\bbC}^\dagger$, which is a result of $X$ being an AIC violator.

The query being estimated is $P(Y_{X_H = 1} = 1 \mid Z = 0)$, or the probability that $Y$ predicts red under the intervention of forcing the image to be a light image, and conditioning on the digit being 0. The results are shown in Fig.~\ref{fig:mnist-est-results}. Three different GAN-NCMs \citep{xia:etal23} are trained. The first (red line) is a $\cG$-NCM that is trained directly on the low-level data and attempts to estimate the low-level query without abstractions. The second (yellow line) is a $\cG_{\bbC}$-NCM trained on the high-level data and is constrained by the C-DAG. The third (blue line) is similar to the second except it is a $\cG_{\bbC}^\dagger$-NCM, constrained by the projected C-DAG. 95\% confidence intervals of the errors across 10 trials are plotted in the figure.

\subsection{Colored MNIST Sampling Experiment}

The second experiment shows the ability of causal generative models to generate samples from causal queries involving high-dimensional images. The setting is described by three variables $\*V_L = \{D, C, I\}$, and the low level model is described as
\begin{itemize}
    \item $D$ and $C$ are 10-dimensional one-hot encodings ($\cD_{D} = \cD_{C} = \{0, 1\}^{10}$ representing digits from 0-9 and colors from a spectrum respectively. Each digit is correlated with a color, a consequence of confounding. The correlated colors are shown on the right side of Fig.~\ref{fig:colored-mnist-legend}. A digit has a 0.9 probability of being its assigned color with a 0.1 probability of deviating.
    
    \item $I$ is a corresponding MNIST digit ($\cD_{I} = \bbR^{3 \times 64 \times 64}$) with color $C$ and digit $D$.
\end{itemize}

The corresponding causal diagram is shown on the left side of Fig.~\ref{fig:colored-mnist-legend}. The results are shown in Fig.~\ref{fig:mnist-bd-samp-results}, demonstrating the ability for each of the methods on the left to sample images from the queries on the top. The non-causal approach simply trains a conditional GAN to sample image given digit. The RNCM \citep{xia:bareinboim24} maps images to a learned representation (i.e., $\tau$ is learned), which serves as the high-level space. However, due to AIC limitations, the dimensionality of $X_H$ must remain high. When $\cD_{X_H} \in \bbR^{16}$, the RNCM is able to sample the digits properly. However, when $\cD_{X_H} \in \{0, 1\}$, the RNCM is unable to get enough expressivity from the representation to perform the sampling. In contrast, the projected sampling approach, which trains a sampling model on top of the high-level model to sample from Eq.~\ref{eq:low-intervention-short}, is still able to reproduce the images despite the low-dimensional representation.

\subsection{Additional Results (Scaling Representation)}

\begin{figure*}
    \begin{center}
    \includegraphics[width=\textwidth,keepaspectratio]{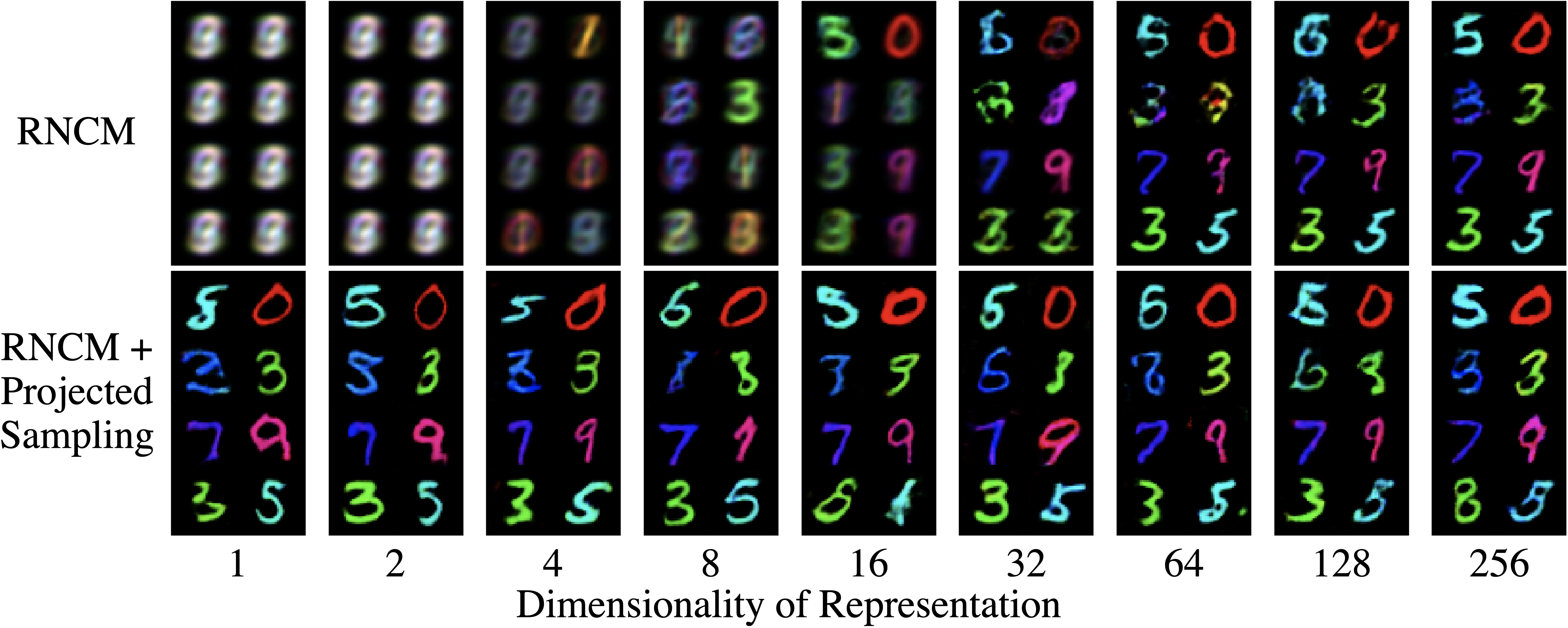}
    \caption{Scaling representation size for the Colored MNIST experiment. Results are shown for generating Colored MNIST digits from an RNCM using the standard strategy (top) and our projected sampling strategy (bottom), with increasing sizes of representation from left to right (shown in number of bits).
    }
    \label{fig:scaling-repr-results}
    \end{center}
\end{figure*}

To demonstrate the drop-off in performance in the original autoencoding representation learning procedure of the RNCM, we generate images from the RNCM trained with a fixed representation size and compare them to images generated from an RNCM trained on the same representation with the projected sampling procedure (described in App.~\ref{app:projected-sampling}). The representations are discrete binary representations ranging from 1 to 256 dimensions, and they are mapped by an untrained CNN with randomly initialized weights.

The results are shown in Fig.~\ref{fig:scaling-repr-results}. The original RNCM sampling procedure is shown in the top row while the projected sampling approach is shown in the bottom row. The size of the representation is increased from 1 to 256 dimensions, from left to right. Note that the quality of the image is very low using the original sampling procedure with a low-dimensional representation, but it gradually improves as the dimensionality increases. This aligns with the logic that a low-dimensional representation is more likely to lose important details of the original image, resulting in violations of the AIC. As expected, higher-dimensional representations capture more information and are less likely to violate the AIC, and violations are likely to be less severe, eventually resulting in perfect reconstruction with sufficiently high representation capacity. In contrast, the quality of images from the projected sampling approach is unaffected by the size of the representation. Any information that cannot be extracted from the representation is supplemented through additional noise in the projected sampling approach. The AIC violation of the low-dimensional representation is handled through the projected abstraction approach described in this paper.

%% file: section/E_discussion.tex
\section{Extended Discussion}

In this section, we include additional discussion related to the contents of this paper.

\subsection{Related Works}

The contents of this work are built on a foundation of prior work, notably on causal inference theory regarding SCMs and the PCH, causal abstraction theory, and causal generative models for the empirical portion. Important definitions are introduced in Sec.~\ref{sec:prelims} and App.~\ref{app:additional-defs}, and further discussion of these prior works are in App.~\ref{app:extended-prelim}.

Many works have achieved strong theoretical and empirical results in the field of mechanistic interpretability by leveraging causal abstraction theory \citep{geiger2023causal2, geiger2023causal, pmlr-v213-massidda23a, DBLP:conf/clear2/ZennaroDAWD23, felekis:etal24}. These works typically operate by treating the black-box neural model as low-level causal model and providing an interpretable high-level SCM as a hypothesis of the functionality of the neural network. An abstraction function $\tau$ is then learned to map the outputs of the low-level neural network with the high-level hypothesis, and the validity of the hypothesis can be evaluated based on whether $\tau$ can achieve the theoretical properties expected of abstractions. This paper focuses on causal abstraction inference, which is a different task that aims to infer causal quantities given lower layer data from the PCH by first mapping the low-level data to a high-level representation. Causal abstraction inference tasks typically leverage abstractions in the form of constructive abstractions, which is not the most general form of abstractions. Nonetheless, challenges surrounding the AIC remain in any causal abstraction task, and leveraging the projected generalizations in this paper to aid in mechanistic interpretability tasks may be a promising avenue of future work.

The strategy of this paper for performing inferences across abstractions under the constraints of the AIC is to generalize the framework to allow for AIC violations. Such an approach allows for one to make an abstraction of any granularity while still having well-defined high-level queries at the cost of additional edges in the C-DAG. Still, this is not the only strategy to avoid the strict constraints placed by the AIC. An alternative approach is to use a weaker version of the AIC that is verifiable by data. This is explored by \citet{Chalupka2015MultiLevelCS, 10.5555/3020847.3020867}, where intravariable clusters of a high-dimensional image data setting are learned. Instead of assuming the full AIC, they compare observational and interventional partitions, which ensure that the outputs of the conditional and interventional distributions given two inputs of the same partition are identical rather than the full SCM function. Notably, they prove the Causal Coarsening Theorem \citep[Thm.~5]{10.5555/3020847.3020867} which shows that the interventional partition is almost always (measure-theoretically speaking) coarser than the observational partition, implying that learning a partition using observational data can guarantee an intervention-level AIC. A more in-depth discussion of this comparison can be found in \citet[App.~D.2]{xia:bareinboim24}.

The process of learning intravariable clusters in causal abstraction inference can be considered a form of representation learning \citep{10.1109/TPAMI.2013.50}. Causal representation learning is an emerging field \citep{Scholkopfetal21}, where the goal is to discover high-level variables from available data. The term ``causal representation learning'' is fairly broad and can refer to several tasks and techniques. One of the most commonly researched tasks is the task of disentanglement causal representation learning \citep{10.5555/3540261.3541519, 10.5555/3586589.3586830, 10.5555/3600270.3603046, varici2023scorebasedcausalrepresentationlearning, 10.5555/3618408.3618426, 10.5555/3618408.3619756, 10.5555/3666122.3667533, 10.5555/3722577.3722852, 10.5555/3692070.3694554, li:etal24}. The goal of disentanglement causal representation learning is to learn high-level variables and their causal relationships from available data and perform causal inferences within those variables. The mapping between the data and the high-level variables is not immediately clear since the variables are often entangled in the data (e.g., a set of several high-level variables may be responsible for explaining a single image). In contrast, this work specifically handles cases where the high-level is studied as a constructive abstraction of the low-level, and there are no claims made on how to disentangle low-level variables. The process of disentanglement is challenging, and resulting high-level inferences are often non-identifiable. To compensate, all of these works incorporate assumptions of some form, such as assuming the availability of certain high-level labels, working in parametric spaces, or having interventional capabilities.

\subsection{Projected Abstraction Limitations and Tradeoffs}

The approaches introduced in this paper are limited by the validity of the assumptions. When the AIC is violated, it is necessary to disambiguate between low-level quantities that are ambiguous on the high-level, since it is otherwise impossible to define high-level queries in terms of their low-level counterparts. This paper uses Eq.~\ref{eq:low-intervention-short} as the method of grounding high-level queries to a specific distribution of its low-level counterparts, and the justification and derivation arises from the discussion in App.~\ref{app:choose-soft-intervention}. Nonetheless, it is possible that some use cases may find benefit from defining the connection differently, and this may still allow for causal inferences provided that the AIC violations are still disambiguated.

The projected abstraction framework allows one to ignore AIC violations and still perform inferences, but notably, this comes at the cost of additional dependencies between high-level variables, indicated by the additional edges of the projected C-DAG in Def.~\ref{def:proj-cdag} and Fig.~\ref{fig:proj-cdag-examples}. It is possible that these newly introduced dependencies may result in the non-identifiability of certain causal queries or present additional computational challenges. If this is undesirable, then it may be preferred to choose a solution that does not violate the AIC at all. If one is trying to learn intravariable clusters but does not wish to violate the AIC, additional assumptions about the variable domains or a weaker definition of the AIC may be required.

Another assumption that may receive criticism is the assumption of the availability of the projected C-DAG for causal inferences in Sec.~\ref{sec:inference}. Note that by the Causal Hierarchy Theorem \citep[Thm.~1]{bareinboim:etal20}, one cannot make inferences about higher layers of the PCH (e.g., interventions from $\cL_2$ or counterfactuals from $\cL_3$) using only data from lower layers of the PCH (e.g., observations from $\cL_1$). This also extends to abstraction inference \citep[Prop.~4]{xia:bareinboim24} and causal inferences with neural networks \citep[Corol.~1]{xia:etal21}. Therefore, assumptions are required in some form to make progress in this problem setting. Many works assume the availability of a causal diagram (Def.~\ref{def:cg}), which provide a qualitative description of the causal relationships of the variables that can enable certain inferences across PCH layers. The qualitative nature of causal diagrams makes this assumption a much weaker one than assuming the availability of the generating SCM. That said, in the context of causal abstraction inference, even this assumption can be relaxed by using the cluster causal diagram or C-DAG (Def.~\ref{def:cdag}) instead. A C-DAG can be interpreted as the high-level counterpart of the low-level causal diagram, which abstracts away any of the detailed causal connections between low-level variables. In this work, it is assumed that the C-DAG $\cG_{\bbC}$ and the AIC violating variables $\*V_H^{\dagger}$ are provided, allowing for the construction of the projected C-DAG $\cG_{\bbC}^{\dagger}$ through Def.~\ref{def:proj-cdag}. Still, it may be possible to perform causal inferences using other types of assumptions. For example, causal discovery is a task where the goal is to learn causal graphical models from observational data using weaker assumptions. However, it is generally the case that weaker assumptions tend to imply weaker inferences.

Finally, the empirical results of this work rely on the assumptions that models are sufficiently large and optimization is perfect. This is never ideal in practice, and it is an open question how to analyze the error bounds of the approach given imperfect optimization. Stronger base architectures may result in better performance that can achieve interesting results in much higher-dimensional problems than what is shown in this paper, and it is also an open question of how to best optimize these models.

\subsection{Duality of SCM Projections and Causal Abstractions}

The concept of SCM projections was first formalized in \citet{lee:bar19a} (see Prop.~\ref{def:scm-proj}), which was leveraged to generalize solutions of causal bandit problems to settings with non-manipulable variables. Notably, non-manipulable variables could be projected away, and the model comprising of the remaining manipulable variables had the same marginal causal distributions as the original model. The definition was compatible with the graphical criterion of projections shown in Fig.~\ref{fig:full-vs-partial-proj}(a).

This paper operates in the problem space of causal abstraction inference, which is orthogonal to the topic of causal bandits, but the concept of SCM projections is integral to both settings. Indeed, SCM projections can be interpreted as a primitive form of an abstraction, where the only operation involves including or excluding specific variables in the transition from low to high-level. This paper generalizes the concepts of SCM projections to the partial case (Def.~\ref{eq:partial-scm-projection}), which can intuitively be interpreted as performing the SCM projection while retaining the originally projected variables. Remarkably, this concept can be tied with the concept of constructive abstractions leveraging inter/intravariable clusters, which is shown throughout the remainder of Sec.~\ref{sec:soft-abs}. Sec.~\ref{sec:inference} further shows how projected C-DAGs are the graphical generalization of partially projections.

\subsection{Soft Intervention Generalization}

Throughout this paper, it is assumed that all interventions applied on the high-level model $\cM_H$ are atomic interventions (i.e., interventions that strictly set variables to one specific value). While some atomic high-level interventions translate to soft (probabilistic) interventions on the low-level through Eq.~\ref{eq:low-intervention-short}, the variation in the system is due to the ambiguity of AIC violations from the high-level intervention.

Still, the setting of the paper can be easily generalized to incorporate soft interventions on the high-level, where interventions can take an arbitrary probabilistic form \citep{correa:bar19}. That is, instead of simply intervening $\*X_H \gets \*x_H$, one can set $\*X_H \gets \sigma_{\*X_H}$, where $\sigma_{\*X_H}$ can be any fixed distribution $P^*(\*X_H \mid \*Z_H)$ for arbitrary $\*Z_H \subseteq \*V_H \setminus \*X_H$. The straightforward way to translate such an intervention to the low-level involves applying Eq.~\ref{eq:low-intervention-short} for every term in the support of $\sigma_{\*X_H}$, multiplied by the corresponding probability. That is,

\begin{align*}
    &P(\*y_{H[\sigma_{\*X_H}]}) \\
    &= \sum_{\*x_H \in \cD_{\*X_H}}P(\sigma_{\*X_H} = \*x_H) P(\*y_{H[\*x_H]}) \\
    &= \sum_{\*x_H \in \cD_{\*X_H}}P(\sigma_{\*X_H} = \*x_H) \sum_{\*y_L \in \cD_{\*Y_L}(\*y_{H[\*x_H]})} P(\*y_{L[\sigma_{\*X_L}]}),
\end{align*}
as from the definition of $Q$-$\tau$ consistency (Def.~\ref{def:q-tau-consistency}). The generalization to counterfactual queries follows similarly.

%% file: main.bbl
\begin{thebibliography}{39}
\providecommand{\natexlab}[1]{#1}
\providecommand{\url}[1]{\texttt{#1}}
\expandafter\ifx\csname urlstyle\endcsname\relax
  \providecommand{\doi}[1]{doi: #1}\else
  \providecommand{\doi}{doi: \begingroup \urlstyle{rm}\Url}\fi

\bibitem[Ahuja et~al.(2023)Ahuja, Mahajan, Wang, and Bengio]{10.5555/3618408.3618426}
Ahuja, K., Mahajan, D., Wang, Y., and Bengio, Y.
\newblock Interventional causal representation learning.
\newblock In \emph{Proceedings of the 40th International Conference on Machine Learning}, ICML'23. JMLR.org, 2023.

\bibitem[Anand et~al.(2023)Anand, Ribeiro, Tian, and Bareinboim]{anand:etal23}
Anand, T.~V., Ribeiro, A.~H., Tian, J., and Bareinboim, E.
\newblock Causal effect identification in cluster dags.
\newblock In \emph{Proceedings of the 37th AAAI Conference on Artificial Intelligence}. AAAI Press, 2023.

\bibitem[Balke \& Pearl(1997)Balke and Pearl]{balke:pea97}
Balke, A. and Pearl, J.
\newblock {Bounds on treatment effects from studies with imperfect compliance}.
\newblock \emph{Journal of the American Statistical Association}, 92\penalty0 (439):\penalty0 1172--1176, 9 1997.

\bibitem[Bareinboim et~al.(2022)Bareinboim, Correa, Ibeling, and Icard]{bareinboim:etal20}
Bareinboim, E., Correa, J.~D., Ibeling, D., and Icard, T.
\newblock On pearl’s hierarchy and the foundations of causal inference.
\newblock In \emph{Probabilistic and Causal Inference: The Works of Judea Pearl}, pp.\  507–556. Association for Computing Machinery, New York, NY, USA, 1st edition, 2022.

\bibitem[Beckers \& Halpern(2019)Beckers and Halpern]{beckers2019abstracting}
Beckers, S. and Halpern, J.~Y.
\newblock Abstracting causal models.
\newblock In \emph{Proceedings of the Thirty-Third AAAI Conference on Artificial Intelligence and Thirty-First Innovative Applications of Artificial Intelligence Conference and Ninth AAAI Symposium on Educational Advances in Artificial Intelligence}, AAAI'19/IAAI'19/EAAI'19. AAAI Press, 2019.
\newblock ISBN 978-1-57735-809-1.
\newblock \doi{10.1609/aaai.v33i01.33012678}.
\newblock URL \url{https://doi.org/10.1609/aaai.v33i01.33012678}.

\bibitem[Beckers et~al.(2019)Beckers, Eberhardt, and Halpern]{Beckers2019-BECACA-8}
Beckers, S., Eberhardt, F., and Halpern, J.~Y.
\newblock Approximate causal abstraction.
\newblock In \emph{Proceedings of the 35th Conference on Uncertainty in Artificial Intelligence}. 2019.

\bibitem[Bengio et~al.(2013)Bengio, Courville, and Vincent]{10.1109/TPAMI.2013.50}
Bengio, Y., Courville, A., and Vincent, P.
\newblock Representation learning: A review and new perspectives.
\newblock \emph{IEEE Trans. Pattern Anal. Mach. Intell.}, 35\penalty0 (8):\penalty0 1798–1828, aug 2013.
\newblock ISSN 0162-8828.
\newblock \doi{10.1109/TPAMI.2013.50}.
\newblock URL \url{https://doi.org/10.1109/TPAMI.2013.50}.

\bibitem[Brehmer et~al.(2022)Brehmer, de~Haan', Lippe, and Cohen]{10.5555/3600270.3603046}
Brehmer, J., de~Haan', P., Lippe, P., and Cohen, T.
\newblock Weakly supervised causal representation learning.
\newblock In \emph{Proceedings of the 36th International Conference on Neural Information Processing Systems}, NeurIPS '22, Red Hook, NY, USA, 2022. Curran Associates Inc.
\newblock ISBN 9781713871088.

\bibitem[Chalupka et~al.(2015{\natexlab{a}})Chalupka, Eberhardt, and Perona]{Chalupka2015MultiLevelCS}
Chalupka, K., Eberhardt, F., and Perona, P.
\newblock Multi-level cause-effect systems.
\newblock In \emph{International Conference on Artificial Intelligence and Statistics}, 2015{\natexlab{a}}.

\bibitem[Chalupka et~al.(2015{\natexlab{b}})Chalupka, Perona, and Eberhardt]{10.5555/3020847.3020867}
Chalupka, K., Perona, P., and Eberhardt, F.
\newblock Visual causal feature learning.
\newblock In \emph{Proceedings of the Thirty-First Conference on Uncertainty in Artificial Intelligence}, UAI'15, pp.\  181–190, Arlington, Virginia, USA, 2015{\natexlab{b}}. AUAI Press.
\newblock ISBN 9780996643108.

\bibitem[Correa \& Bareinboim(2024)Correa and Bareinboim]{correa2024ctfcalc}
Correa, J. and Bareinboim, E.
\newblock Counterfactual graphical models: Constraints and inference.
\newblock Technical Report R-115, Causal Artificial Intelligence Lab, Columbia University, August 2024.

\bibitem[Correa \& Bareinboim(2019)Correa and Bareinboim]{correa:bar19}
Correa, J.~D. and Bareinboim, E.
\newblock {From Statistical Transportability to Estimating the Effect of Stochastic Interventions}.
\newblock In \emph{Proceedings of the Twenty-Eighth International Joint Conference on Artificial Intelligence}, pp.\  1661--1667. IJCAI Organization, 2019.

\bibitem[Felekis et~al.(2024)Felekis, Zennaro, Branchini, and Damoulas]{felekis:etal24}
Felekis, Y., Zennaro, F.~M., Branchini, N., and Damoulas, T.
\newblock Causal optimal transport of abstractions.
\newblock In \emph{Conference on Causal Learning and Reasoning, CLeaR 2024}, 2024.

\bibitem[Geiger et~al.(2023{\natexlab{a}})Geiger, Ibeling, Zur, Chaudhary, Chauhan, Huang, Arora, Wu, Goodman, Potts, et~al.]{geiger2023causal2}
Geiger, A., Ibeling, D., Zur, A., Chaudhary, M., Chauhan, S., Huang, J., Arora, A., Wu, Z., Goodman, N., Potts, C., et~al.
\newblock Causal abstraction: A theoretical foundation for mechanistic interpretability.
\newblock \emph{arXiv preprint arXiv:2301.04709}, 2023{\natexlab{a}}.

\bibitem[Geiger et~al.(2023{\natexlab{b}})Geiger, Potts, and Icard]{geiger2023causal}
Geiger, A., Potts, C., and Icard, T.
\newblock Causal abstraction for faithful model interpretation, 2023{\natexlab{b}}.

\bibitem[Lee \& Bareinboim(2019)Lee and Bareinboim]{lee:bar19a}
Lee, S. and Bareinboim, E.
\newblock {Structural Causal Bandits with Non-manipulable Variables}.
\newblock In \emph{Proceedings of the 33rd AAAI Conference on Artificial Intelligence}, 2019.

\bibitem[Li et~al.(2024)Li, Pan, and Bareinboim]{li:etal24}
Li, A., Pan, Y., and Bareinboim, E.
\newblock Disentangled representation learning in non-markovian causal systems.
\newblock In \emph{Advances in Neural Information Processing Systems}, 2024.

\bibitem[Massidda et~al.(2023)Massidda, Geiger, Icard, and Bacciu]{pmlr-v213-massidda23a}
Massidda, R., Geiger, A., Icard, T., and Bacciu, D.
\newblock Causal abstraction with soft interventions.
\newblock In van~der Schaar, M., Zhang, C., and Janzing, D. (eds.), \emph{Proceedings of the Second Conference on Causal Learning and Reasoning}, volume 213 of \emph{Proceedings of Machine Learning Research}, pp.\  68--87. PMLR, 11--14 Apr 2023.
\newblock URL \url{https://proceedings.mlr.press/v213/massidda23a.html}.

\bibitem[Pearl(1995)]{pearl:95a}
Pearl, J.
\newblock {Causal diagrams for empirical research}.
\newblock \emph{Biometrika}, 82\penalty0 (4):\penalty0 669--688, 1995.

\bibitem[Pearl(2000)]{pearl:2k}
Pearl, J.
\newblock \emph{{Causality: Models, Reasoning, and Inference}}.
\newblock Cambridge University Press, New York, NY, USA, 2nd edition, 2000.

\bibitem[Pearl(2017)]{pearl:17-r359}
Pearl, J.
\newblock Physical and metaphysical counterfactuals: Evaluating disjunctive actions.
\newblock \emph{Journal of Causal Inference}, 5\penalty0 (2), 2017.

\bibitem[Pearl \& Mackenzie(2018)Pearl and Mackenzie]{pearl:mackenzie2018}
Pearl, J. and Mackenzie, D.
\newblock \emph{{The Book of Why}}.
\newblock Basic Books, New York, 2018.

\bibitem[Rubenstein et~al.(2017)Rubenstein, Weichwald, Bongers, Mooij, Janzing, Grosse-Wentrup, and Sch{\"{o}}lkopf]{rubenstein:etal17-causalsem}
Rubenstein, P.~K., Weichwald, S., Bongers, S., Mooij, J., Janzing, D., Grosse-Wentrup, M., and Sch{\"{o}}lkopf, B.
\newblock {Causal Consistency of Structural Equation Models}.
\newblock In \emph{Proceedings of the Thirty-Third Conference on Uncertainty in Artificial Intelligence}, 2017.

\bibitem[Sch{\"o}lkopf* et~al.(2021)Sch{\"o}lkopf*, Locatello*, Bauer, Ke, Kalchbrenner, Goyal, and Bengio]{Scholkopfetal21}
Sch{\"o}lkopf*, B., Locatello*, F., Bauer, S., Ke, N.~R., Kalchbrenner, N., Goyal, A., and Bengio, Y.
\newblock Toward causal representation learning.
\newblock \emph{Proceedings of the IEEE}, 109\penalty0 (5):\penalty0 612--634, 2021.
\newblock URL \url{https://ieeexplore.ieee.org/stamp/stamp.jsp?arnumber=9363924}.
\newblock *equal contribution.

\bibitem[Shen et~al.(2022)Shen, Liu, Dong, Lian, Chen, and Zhang]{10.5555/3586589.3586830}
Shen, X., Liu, F., Dong, H., Lian, Q., Chen, Z., and Zhang, T.
\newblock Weakly supervised disentangled generative causal representation learning.
\newblock \emph{J. Mach. Learn. Res.}, 23\penalty0 (1), January 2022.
\newblock ISSN 1532-4435.

\bibitem[Squires et~al.(2023)Squires, Seigal, Bhate, and Uhler]{10.5555/3618408.3619756}
Squires, C., Seigal, A., Bhate, S., and Uhler, C.
\newblock Linear causal disentanglement via interventions.
\newblock In \emph{Proceedings of the 40th International Conference on Machine Learning}, ICML'23. JMLR.org, 2023.

\bibitem[Steinberg(2007)]{steinberg2007}
Steinberg, D.
\newblock Copyright.
\newblock In \emph{The Cholesterol Wars}. Academic Press, Oxford, 2007.
\newblock ISBN 978-0-12-373979-7.
\newblock \doi{https://doi.org/10.1016/B978-0-12-373979-7.50003-0}.
\newblock URL \url{https://www.sciencedirect.com/science/article/pii/B9780123739797500030}.

\bibitem[Truswell(2010)]{truswell2010}
Truswell, A.
\newblock \emph{Cholesterol and Beyond: The Research on Diet and Coronary Heart Disease 1900-2000}.
\newblock 01 2010.
\newblock ISBN 978-90-481-8874-1.
\newblock \doi{10.1007/978-90-481-8875-8}.

\bibitem[Varici et~al.(2023)Varici, Acarturk, Shanmugam, Kumar, and Tajer]{varici2023scorebasedcausalrepresentationlearning}
Varici, B., Acarturk, E., Shanmugam, K., Kumar, A., and Tajer, A.
\newblock Score-based causal representation learning with interventions, 2023.
\newblock URL \url{https://arxiv.org/abs/2301.08230}.

\bibitem[von K\"{u}gelgen et~al.(2021)von K\"{u}gelgen, Sharma, Gresele, Brendel, Sch\"{o}lkopf, Besserve, and Locatello]{10.5555/3540261.3541519}
von K\"{u}gelgen, J., Sharma, Y., Gresele, L., Brendel, W., Sch\"{o}lkopf, B., Besserve, M., and Locatello, F.
\newblock Self-supervised learning with data augmentations provably isolates content from style.
\newblock In \emph{Proceedings of the 35th International Conference on Neural Information Processing Systems}, NeurIPS '21, Red Hook, NY, USA, 2021. Curran Associates Inc.
\newblock ISBN 9781713845393.

\bibitem[Wang \& Jordan(2024)Wang and Jordan]{10.5555/3722577.3722852}
Wang, Y. and Jordan, M.~I.
\newblock Desiderata for representation learning: a causal perspective.
\newblock \emph{J. Mach. Learn. Res.}, 25\penalty0 (1), January 2024.
\newblock ISSN 1532-4435.

\bibitem[Wendong et~al.(2023)Wendong, Keki\'{c}, von K\"{u}gelgen, Buchholz, Besserve, Gresele, and Sch\"{o}lkopf]{10.5555/3666122.3667533}
Wendong, L., Keki\'{c}, A., von K\"{u}gelgen, J., Buchholz, S., Besserve, M., Gresele, L., and Sch\"{o}lkopf, B.
\newblock Causal component analysis.
\newblock In \emph{Proceedings of the 37th International Conference on Neural Information Processing Systems}, NeurIPS '23, Red Hook, NY, USA, 2023. Curran Associates Inc.

\bibitem[Xia \& Bareinboim(2024)Xia and Bareinboim]{xia:bareinboim24}
Xia, K. and Bareinboim, E.
\newblock Neural causal abstractions.
\newblock In \emph{Proceedings of the 38th AAAI Conference on Artificial Intelligence}. AAAI Press, 2024.

\bibitem[Xia \& Bareinboim(2025)Xia and Bareinboim]{xialossy:tr}
Xia, K. and Bareinboim, E.
\newblock {Causal Abstraction Inference under Lossy Representations}.
\newblock Technical Report Technical Report R-124, Columbia University, Department of Computer Science, New York, 2025.

\bibitem[Xia et~al.(2021)Xia, Lee, Bengio, and Bareinboim]{xia:etal21}
Xia, K., Lee, K.-Z., Bengio, Y., and Bareinboim, E.
\newblock The causal-neural connection: Expressiveness, learnability, and inference.
\newblock In Ranzato, M., Beygelzimer, A., Dauphin, Y., Liang, P., and Vaughan, J.~W. (eds.), \emph{Advances in Neural Information Processing Systems}, volume~34, pp.\  10823--10836. Curran Associates, Inc., 2021.

\bibitem[Xia et~al.(2023)Xia, Pan, and Bareinboim]{xia:etal23}
Xia, K., Pan, Y., and Bareinboim, E.
\newblock Neural causal models for counterfactual identification and estimation.
\newblock In \emph{Proceedings of the 11th International Conference on Learning Representations (ICLR-23)}, 2023.

\bibitem[Zennaro et~al.(2023)Zennaro, Dr{\'{a}}vucz, Apachitei, Widanage, and Damoulas]{DBLP:conf/clear2/ZennaroDAWD23}
Zennaro, F.~M., Dr{\'{a}}vucz, M., Apachitei, G., Widanage, W.~D., and Damoulas, T.
\newblock Jointly learning consistent causal abstractions over multiple interventional distributions.
\newblock In van~der Schaar, M., Zhang, C., and Janzing, D. (eds.), \emph{Conference on Causal Learning and Reasoning, CLeaR 2023, 11-14 April 2023, Amazon Development Center, T{\"{u}}bingen, Germany, April 11-14, 2023}, volume 213 of \emph{Proceedings of Machine Learning Research}, pp.\  88--121. {PMLR}, 2023.
\newblock URL \url{https://proceedings.mlr.press/v213/zennaro23a.html}.

\bibitem[Zhang et~al.(2022)Zhang, Jin, and Bareinboim]{zhang:bareinboim21b}
Zhang, J., Jin, T., and Bareinboim, E.
\newblock Partial counterfactual identification from observational and experimental data.
\newblock In \emph{Proceedings of the 39th International Conference on Machine Learning (ICML-22)}, 2022.

\bibitem[Zhang et~al.(2024)Zhang, Xie, Ng, and Zheng]{10.5555/3692070.3694554}
Zhang, K., Xie, S., Ng, I., and Zheng, Y.
\newblock Causal representation learning from multiple distributions: a general setting.
\newblock In \emph{Proceedings of the 41st International Conference on Machine Learning}, ICML'24. JMLR.org, 2024.

\end{thebibliography}
